\documentclass{article}
\pdfoutput=1

\usepackage{microtype}
\usepackage{graphicx}
\usepackage{subfigure}
\usepackage{booktabs} % for professional tables
\usepackage{amssymb}
\usepackage{bm}
\usepackage{bbm}
\usepackage{amsmath}  % Define \boldsymbol (in amsbsy too) and align
\usepackage{amsthm}
\usepackage{xcolor}
\usepackage{textcomp}
\usepackage{multirow}
\usepackage{mathtools}

\usepackage{tikz}
\usetikzlibrary{arrows}
\usetikzlibrary{positioning}

\tikzset{
  treenode/.style = {align=center, inner sep=0pt, text centered,
    font=\sffamily},
  arn_n/.style = {treenode, circle, black, font=\sffamily\bfseries, draw=black,
    fill=white, text width=1.5em},% arbre rouge noir, noeud noir
  arn_r/.style = {treenode, circle, black, font=\sffamily\bfseries, draw=black,
    fill=white, text width=1.0em},% arbre rouge noir, noeud rouge
  arn_x/.style = {treenode, rectangle, draw=black,
    minimum width=0.5em, minimum height=0.5em}% arbre rouge noir, nil
}

\usepackage{hyperref}

\usepackage[arxiv]{icml2020}

\icmltitlerunning{Fast and Three-rious: Speeding Up Weak Supervision with Triplet Methods}

\usepackage{bm}
\usepackage{enumitem}

\newtheorem{theorem}{Theorem}
\newtheorem{lemma}{Lemma}

\newtheorem{proposition}{Proposition}
\newtheorem{remark}{Remark}

\newcommand{\argmin}[2]{\textrm{argmin}_{#1}~#2}
\newcommand{\argmax}[2]{\textrm{argmax}_{#1}~#2}
\newcommand{\E}[2]{\mathbb{E}_{#1}\left[#2\right]}
\newcommand{\Ehat}[1]{\hat{\mathbb{E}}\left[#1\right]}
\newcommand{\Cov}[2]{\textrm{\textbf{Cov}}_{#1}\left[#2\right]}
\newcommand{\ind}[1]{\mathbbm{1}\left\{#1\right\}}

\newcommand\independent{\protect\mathpalette{\protect\independenT}{\perp}}
\def\independenT#1#2{\mathrel{\rlap{$#1#2$}\mkern2mu{#1#2}}}

\newcommand{\lf}[0]{\lambda}

\ifdefined\ShowNotes
  \newcommand{\colornote}[3]{{\color{#1}\bf{#2 #3}\normalfont}}
\else
  \newcommand{\colornote}[3]{}
\fi

\definecolor{darkred}{rgb}{0.7,0.1,0.1}
\definecolor{darkgreen}{rgb}{0.1,0.5,0.1}
\definecolor{cyan}{rgb}{0.7,0.0,0.7}
\definecolor{dblue}{rgb}{0.2,0.2,0.8}
\definecolor{maroon}{rgb}{0.76,.13,.28}
\definecolor{burntorange}{rgb}{0.81,.33,0}
\definecolor{royalpurple}{rgb}{0.47,.31,0.66}

\ifdefined\ShowNotes
  \newcommand{\num}[1]{{\color{red}\bf{#1}\normalfont}}
\else
  \newcommand{\num}[1]{#1}
\fi

\newcommand{\sysx}{\textsc{FlyingSquid}}

\newcommand{\maxspeedupdugong}{\num{4,000}}

\newcommand{\avgspeedupdpbench}{\num{170}}
\newcommand{\maxspeedupdpbench}{\num{440}}
\newcommand{\avgspeedupdpbenchround}{\num{170}}

\newcommand{\maxliftdpbench}{\num{4.9}}

\newcommand{\maxlifttsonline}{\num{36.5}}

\newcommand{\maxliftmvonline}{\num{15.7}}
\newcommand{\avgliftrandomabstains}{\num{25.6}}

\newcommand{\avgliftsingletriplet}{\num{23.8}}

\newcommand{\spam}{\textbf{Spam}}
\newcommand{\spouse}{\textbf{Spouse}}
\newcommand{\weather}{\textbf{Weather}}

\newcommand{\interview}{\textbf{Interview}}
\newcommand{\commercial}{\textbf{Commercial}}
\newcommand{\tennis}{\textbf{Tennis Rally}}
\newcommand{\basketball}{\textbf{Basketball}}

\begin{document}

\newif\ifsinglecolumn
\newif\ifarxiv
\arxivtrue

\twocolumn[
\icmltitle{Fast and Three-rious: Speeding Up Weak Supervision with Triplet Methods}

\icmlsetsymbol{equal}{*}

\begin{icmlauthorlist}
\icmlauthor{Daniel~Y.~Fu}{equal,to}
\icmlauthor{Mayee~F.~Chen}{equal,to}
\icmlauthor{Frederic~Sala}{to}
\icmlauthor{Sarah~M.~Hooper}{goo}
\icmlauthor{Kayvon~Fatahalian}{to}
\icmlauthor{Christopher~R{\'e}}{to}
\end{icmlauthorlist}

\icmlaffiliation{to}{Department of Computer Science, Stanford University}
\icmlaffiliation{goo}{Department of Electrical Engineering, Stanford University}

\icmlcorrespondingauthor{Daniel~Y.~Fu}{danfu@cs.stanford.edu}

\icmlkeywords{Weak Supervision, Latent Variable Models}

\vskip 0.3in
]

\printAffiliationsAndNotice{\icmlEqualContribution} % otherwise use the standard text.

\begin{abstract}
  %!TEX root = ../main.tex

Weak supervision is a popular method for building machine learning models
without relying on ground truth annotations.
Instead, it generates probabilistic training
labels by estimating the accuracies of multiple noisy labeling sources (e.g.,
heuristics, crowd workers).
Existing approaches use latent variable estimation to model the noisy sources,
but these methods can be computationally expensive, scaling superlinearly in
the data.
In this work, we show that, for a class of latent variable models highly
applicable to weak supervision, we can find
a \textit{closed-form solution} to model parameters, 
obviating the need for iterative solutions like stochastic gradient
descent (SGD).
We use this insight to build \sysx, a weak supervision framework that
runs \textit{orders of magnitude} faster than previous weak supervision
approaches and requires fewer assumptions.
In particular, we prove bounds on generalization error without assuming that
the latent variable model can exactly parameterize the underlying data
distribution.
Empirically, we validate \sysx\ on benchmark weak supervision datasets and
find that it achieves the same or higher quality compared to previous
approaches without the need to tune an SGD procedure,
recovers model parameters \avgspeedupdpbenchround\ times faster on average, and
enables new video analysis and online learning applications.

\end{abstract}

\begin{figure*}[t]
  \centering
  \includegraphics[width=6.75in]{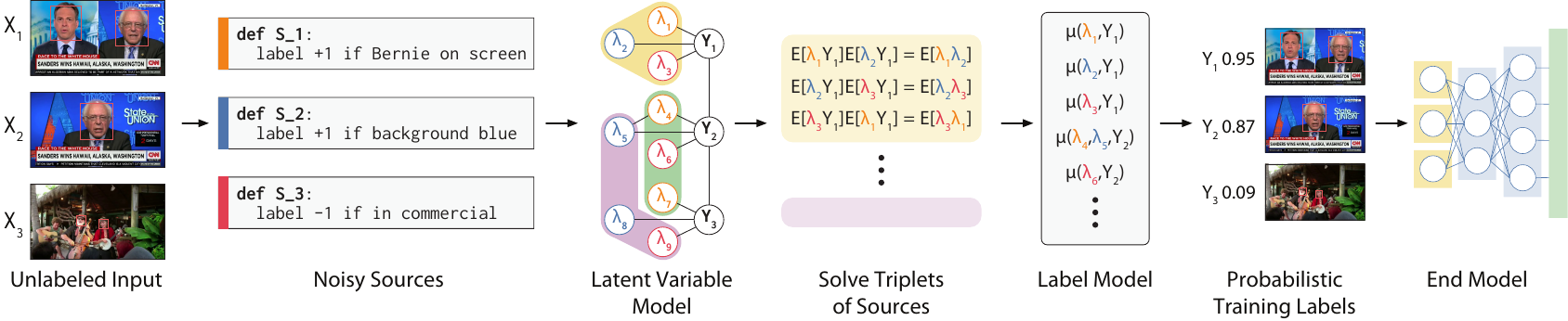}
  \ifsinglecolumn
  \else
  \vspace{-2em}
  \fi
  \caption{
  The \sysx\ pipeline.
  Users provide weak supervision sources, which generate noisy labels for a
  set of unlabeled data.
  \sysx\ uses a latent variable model and constructs 
  triplets of sources to turn model parameter estimation into a set of minimal
  subproblems with closed-form solutions.
  The label model then generates probabilistic training labels to
  train a downstream end model.
  }
  \label{fig:overview}
\end{figure*}

%!TEX root = ../main.tex

\section{Introduction}
\label{sec:intro}

Modern machine learning systems require large amounts of labeled training data
to be successful.
Weak supervision is a class of popular methods for building models without
resorting to manually labeling training data~\cite{dehghani2017neural,
dehghani2017learning, jia2017constrained, mahajan2018exploring, niu2012deepdive};
it drives applications used by
billions of people every day, ranging from Gmail~\cite{sheng2020gmail} to
AI products at Apple~\cite{re2019overton} and search
products~\cite{bach2018snorkel}.
These approaches use noisy
sources, such as heuristics, crowd workers, external
knowledge bases, and user-defined functions \cite{gupta2014improved, Ratner19,
karger2011iterative, dawid1979maximum, mintz2009distant, zhang:cacm17,
hearst1992automatic} to
generate probabilistic training labels without hand-labeling.

The major technical challenge in weak supervision is to efficiently estimate
the accuracies of---and potentially the correlations among---the noisy sources
without any labeled data~\cite{guan2018said,takamatsu2012reducing,
xiao2015learning,Ratner18}.
Standard approaches to this problem, from classical crowdsourcing to more
recent methods, use latent variable probabilistic graphical models (PGMs) to
model the primary sources of signal---the
agreements and disagreements between sources, along with known or estimated
source independencies~\cite{dawid1979maximum, karger2011iterative, Ratner16}.

However, latent variable estimation is
challenging, and the techniques
are often sample- and computationally-complex.
For example, \citet{bach2018snorkel} required multiple iterations of a
Gibbs-based algorithm, and \citet{Ratner19} required estimating
the full inverse covariance matrix among the sources, while
\citet{sala2019multiresws} and \citet{zhan2019sequentialws} required the use of
multiple iterations of stochastic gradient descent (SGD) to learn accuracy
parameters. 
These limitations make it difficult to use weak supervision in applications
that require modeling complex temporal or spatial dependencies, such as video
and image analysis, or in streaming applications that have strict latency
requirements.
In contrast, our solution is motivated by a key observation: that by breaking
the problem into minimal subproblems---solving parameters
for triplets of sources at a time, similar to~\citet{joglekar2013evaluating}
and~\citet{chaganty2014estimating}---we can reduce parameter estimation into
solving systems of equations that have simple, \emph{closed-form solutions}.

Concretely, we show that, for a class of binary Ising models,
we can reduce the problem of accuracy and
correlation estimation to solving a
set of systems of equations whose size is linear in the number of sources.
These systems admit a closed-form solution, so we can estimate the model parameters
in time linear in the data with provable bounds,
even though inference is NP-hard in general Ising models~\cite{ch2012complexity, koller2009probabilistic}.
Critically, the class of Ising models we use captures many weak supervision settings
and is larger than that used in previous efforts.
We use these insights to build \sysx, a new weak supervision framework that
learns label source accuracies with a closed-form solution.

We analyze the downstream performance of end models trained with labels
generated by \sysx, and prove the following results:
\begin{itemize}[itemsep=0.5pt,topsep=0pt]
\item 
We prove that the generalization error of a model trained with
labels generated by \sysx\ scales at
the same asymptotic rate as supervised learning.

\item
We analyze model misspecification using KL divergence, a more fine-grained result than \citet{Ratner19}.

\item 
We show that our parameter estimation approach can be sample optimal up to constant
factors via an information-theoretic lower bound on minimax risk.

\item 
We prove a first-of-its-kind result for downstream generalization of a
window-based online weak supervision algorithm,
accounting for distributional drift.
\end{itemize}

Next, we empirically validate \sysx\ on three benchmark weak supervision datasets that have been
used to evaluate previous state-of-the-art weak supervision
frameworks~\cite{Ratner18}, as well as on four video analysis
tasks, where labeling training data is particularly expensive and modeling
temporal dependencies introduces significant slowdowns in learning graphical
model parameters.
We find that \sysx\ achieves the same or higher quality as previous approaches
while learning parameters orders of magnitude faster.
Since \sysx\ runs so fast, we can learn graphical model parameters
\textit{in the training loop} of a discriminative end model.
This allows us to extend \sysx\ to the online learning setting with a window-based algorithm, where we update
model parameters at the same time as we generate labels for an end model.
In summary, we observe the following empirical results:
\begin{itemize}[itemsep=0.5pt,topsep=0pt]
\item We replicate evaluations of previous approaches and match or exceed their
accuracy (up to \maxliftdpbench\ F1 points).

\item On  tasks with relatively simple graphical model structures,
\sysx\ learns model parameters \avgspeedupdpbench\ times faster on
average; on
video analysis tasks, where there are complex temporal dependencies, \sysx\
learns up to \maxspeedupdugong\ times faster.

\item We demonstrate that our window-based online weak supervision extension can
both update model parameters and train an end model completely online,
outperforming a majority vote baseline by up to \maxliftmvonline\ F1 points.
\end{itemize}

\tikzstyle{place4}=[rectangle,draw=black!100,dashed,fill=white!100,thick, inner sep=12.5pt,fill opacity=0.2]

\definecolor{tableaured}{RGB}{225,87,89}
\definecolor{tableauyellow}{RGB}{237,201,73}
\definecolor{tableaugreen}{RGB}{90,161,80}
\definecolor{tableaupurple}{RGB}{176,123,161}
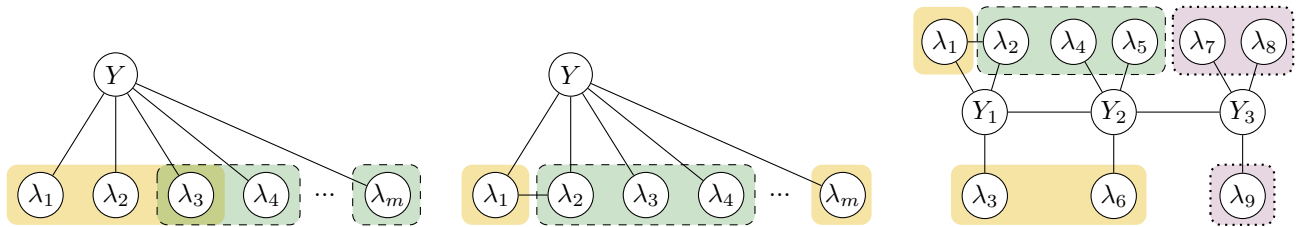
\begin{figure*}[t]
\centering
\begin{tikzpicture}[-,>=stealth',level/.style={sibling distance = 1cm/#1,
  level distance = 1.5cm}] 
\draw [fill=tableauyellow, draw opacity=0.0, fill opacity=0.5, thick, rounded corners] (-1.45, -1.2) rectangle (1.45, -2.0);
\draw [fill=tableaugreen, draw opacity=1.0, fill opacity=0.3, thin, rounded corners, dashed] (.55, -1.2) rectangle (2.45, -2.0);
\draw [fill=tableaugreen, draw opacity=1.0, fill opacity=0.3, thin, rounded corners, dashed] (4.05, -1.2) rectangle (3.15, -2.0);
\node [arn_n] (Y) {$Y$} ;
\node [arn_n] [below=1cm of Y] (L2){$\lf_2$}   ;
\node [arn_n] [left of=L2] (L1){$\lf_1$}   ;
\node [arn_n] [ right of=L2] (L3){$\lf_3$}   ;
\node [arn_n] [ right of=L3] (L4){$\lf_4$}   ;
\node [arn_n] [ right=1cm of L4] (Lm) {$\lf_m$}   ;
\node [] [ right=0.2cm of L4] {$...$}   ;
\path (Y) edge            (L1);
\path (Y) edge            (L2);
\path (Y) edge            (L3);
\path (Y) edge            (L4);
\path (Y) edge            (Lm);
%\node[place4](avg) at (0.0,-2.0) {}; %{\quad \quad\quad\quad\quad Averaging \quad\quad\quad\quad\quad};
\end{tikzpicture}
\hspace{1em}
\begin{tikzpicture}[-,>=stealth',level/.style={sibling distance = 1cm/#1,
  level distance = 1.5cm}] 
\draw [fill=tableauyellow, draw opacity=0.0, fill opacity=0.5, thick, rounded corners] (-1.45, -1.2) rectangle (-0.55, -2.0);
\draw [fill=tableaugreen, draw opacity=1.0, fill opacity=0.3, thin, rounded corners, dashed] (-0.45, -1.2) rectangle (2.45, -2.0);
\draw [fill=tableauyellow, draw opacity=0.0, fill opacity=0.5, thick, rounded corners] (4.0, -1.2) rectangle (3.2, -2.0);
\node [arn_n] (Y) {$Y$} ;
\node [arn_n] [below=1cm of Y] (L2){$\lf_2$}   ;
\node [arn_n] [left of=L2] (L1){$\lf_1$}   ;
\node [arn_n] [ right of=L2] (L3){$\lf_3$}   ;
\node [arn_n] [ right of=L3] (L4){$\lf_4$}   ;
\node [arn_n] [ right=1cm of L4] (Lm) {$\lf_m$}   ;
\node [] [ right=0.2cm of L4] {$...$}   ;
\path (Y) edge            (L1);
\path (L2) edge            (L1);
\path (Y) edge            (L2);
\path (Y) edge            (L3);
\path (Y) edge            (L4);
\path (Y) edge            (Lm);
%\node[place4](avg) at (0.0,-2.0) {}; %{\quad \quad\quad\quad\quad Averaging \quad\quad\quad\quad\quad};
\end{tikzpicture}
\hspace{1em}
\begin{tikzpicture}[-,>=stealth',level/.style={sibling distance = 1cm/#1,
  level distance = 1.0cm}] 
\draw [fill=tableauyellow, draw opacity=0.0, fill opacity=0.5, thick, rounded corners] (-0.95, 0.5) rectangle (-0.15, 1.4);
\draw [fill=tableauyellow, draw opacity=0.0, fill opacity=0.5, thick, rounded corners] (-0.45, -1.5) rectangle (2.15, -0.7);
\draw [fill=tableaugreen, draw opacity = 1.0, fill opacity=0.3, thin, rounded corners, dashed] (-0.1, 1.4) rectangle (2.4, 0.5);
\draw [fill=tableaupurple, draw opacity = 1.0, fill opacity=0.3, thick, rounded corners, dotted] (2.5, 1.4) rectangle (4.1, 0.5);
\draw [fill=tableaupurple, draw opacity=1.0, fill opacity=0.3, thick, rounded corners, dotted] (3.0, -1.5) rectangle (3.85, -0.7);
\node [arn_n] (Y1) {$Y_1$} ;
\node [arn_n] [ right=1.1cm of Y1] (Y2) {$Y_2$} ;
\node [arn_n] [ right=1.1cm of Y2] (Y3) {$Y_3$} ;
\node [arn_n] [below=0.5cm of Y1] (L3){$\lf_3$}   ;
\node [arn_n] [below=0.5cm of Y2] (L6){$\lf_6$}   ;
\node [arn_n] [below=0.5cm of Y3] (L9){$\lf_9$}   ;
\node [arn_n] [above left=0.5cm and 0.1 cm of Y1] (L1){$\lf_1$}   ;
\node [arn_n] [right=0.2cm of L1] (L2){$\lf_2$}   ;
\node [arn_n] [above left=0.5cm and 0.1 cm of Y2] (L4){$\lf_4$}   ;
\node [arn_n] [right=0.2cm of L4] (L5){$\lf_5$}   ;
\node [arn_n] [above left=0.5cm and 0.1 cm of Y3] (L7){$\lf_7$}   ;
\node [arn_n] [right=0.2cm of L7] (L8){$\lf_8$}   ;
\path (Y1) edge            (Y2);
\path (Y2) edge            (Y3);
\path (Y1) edge            (L3);
\path (Y2) edge            (L6);
\path (Y3) edge            (L9);
\path (Y1) edge            (L1);
\path (Y1) edge            (L2);
\path (Y2) edge            (L4);
\path (Y2) edge            (L5);
\path (Y3) edge            (L7);
\path (Y3) edge            (L8);
\path (L1) edge            (L2);
\end{tikzpicture}
\ifsinglecolumn
\else
\vspace{-1.5em}
\fi
\caption{Example of dependency structure graphs and triplets (rectangles). Left: Conditionally independent sources; Middle: With dependencies. Right: Multiple temporally-correlated labels $\{Y_1, Y_2, Y_3\}$ with per-label sources.}
\label{fig:graph_structures}
\end{figure*}

We release \sysx\ as a novel layer integrated into PyTorch.\footnote{https://github.com/HazyResearch/flyingsquid}
This layer allows weak supervision to be integrated off-the-shelf
into any deep learning model, learning the accuracies of noisy labeling sources
in the same training loop as the end model.
Our layer can be used in any standard training set up, enabling
new modes of training from multiple label sources.

%!TEX root = ../main.tex

\section{Weakly Supervised Machine Learning}
\label{sec:overview}

In this section, we give an overview of weak supervision and our problem setup.
In Section~\ref{sec:overview_ws}, we give an overview of the inputs to weak
supervision from the user's perspective.
In Section~\ref{sec:overview_problem}, we describe the formal problem setup.
Finally, in Section~\ref{sec:overview_label_model}, we show how the problem
reduces to estimating the parameters of a latent variable PGM.

\subsection{Background: Weak Supervision}
\label{sec:overview_ws}
We first give some background on weak supervision at a high level.
In weak supervision, practitioners programmatically generate training labels through the process shown in Figure~\ref{fig:overview}.
Users build multiple weak supervision sources that assign noisy labels to data.
For example, an analyst trying to detect interviews of Bernie Sanders in a
corpus of cable TV news may use off-the-shelf face detection and identity
classification networks to detect frames where Sanders is on screen, or she
may write a Python function to search closed captions for instances
of the text ``Bernie Sanders."
Critically, these weak supervision sources can vote or abstain on individual
data points; this lets users express high-precision signals without requiring
them to have high recall as well.
For example, while the text ``Bernie Sanders" in the transcript is a strong
signal for an interview, the absence of the text is not a strong signal for the
absence of an interview (once he is introduced, his name is not mentioned for
most of the interview).

These sources are noisy and may conflict with each other, so a latent variable
model, which we refer to as a \textit{label model}, is used to express the
accuracies of and correlations between them.
Once its parameters are learned, the model is used to aggregate source votes and generate probabilistic
training labels, which are in turn used to train a downstream discriminative
model (\emph{end model} from here on).

\subsection{Problem Setup}
\label{sec:overview_problem}
Now, we formally define our learning problem.
Let $\bm{X} = [X_1, X_2, \ldots, X_D] \in \mathcal{X}$ be a vector of $D$
related elements (e.g., contiguous frames in a video, or
neighboring pixels in an image).
Let $\bm{Y} = [Y_1, Y_2, \ldots, Y_D] \in \mathcal{Y}$ be the vector of
\textit{unobserved} true labels for each element (e.g., the per-frame label for
event detection in video, or a per-pixel label for a segmentation mask in an
image).
We refer to each $Y_i$ as a \textit{task}.
We have $ (\bm{X}, \bm{Y}) \sim \mathcal{D}$ for some distribution $\mathcal{D}$.
We simplify to binary $Y_i \in \{\pm 1\}$ for ease of exposition (we
discuss the multi-class case in
\ifarxiv
Appendix \ref{subsec:extensions}).
\else
Appendix C.2).
\fi
Let $m$ be the number of sources $S_1, \ldots, S_m$, each assigning a label $\lambda_j \in \{\pm 1 \}$ to some single element $X_i$ to vote on its respective $Y_i$,
or abstaining ($\lambda_j = 0$).

The goal is to apply the $m$ weak supervision sources to an unlabeled dataset
$\{\bm{X}^i\}_{i = 1}^n$ with $n$ data points to create an $n \times m$ label matrix $L$,
combine the source votes into element-wise
probabilistic training labels, $\{\bm{\widetilde{Y}}^i\}_{i = 1}^n$, and use them to train a discriminative
classifier $f_w : \mathcal{X} \rightarrow \mathcal{Y}$, \textit{all without
observing any ground truth labels}.

\subsection{Label Model}
\label{sec:overview_label_model}
Now, we describe how we use a probabilistic graphical model to generate
training data based on labeling function outputs.
First, we describe how we use a graph to specify the conditional dependencies
between label sources and tasks.
Next, we describe how to represent the task labels $\bm{Y}$ and source votes
$\bm{\lf}$ using a binary Ising
model from user-provided conditional dependencies between sources and tasks.
Then, we discuss how to perform inference using the junction tree formula and introduce the label model parameters our method focuses on estimating.

\paragraph{Conditional Dependencies}
Let a graph $G_{dep}$ specify conditional dependencies between sources and
tasks, using standard technical notions from the PGM
literature~\cite{koller2009probabilistic, Lauritzen, wainwright2008graphical}. 
In particular, the lack of an edge in $G_{dep}$ between a pair of
variables indicates independence conditioned on a \emph{separator set} of
variables \cite{Lauritzen}.
We assume that $G_{dep}$ is user-provided; it can also be
estimated directly from source votes~\cite{Ratner19}.
Figure~\ref{fig:graph_structures} shows three graphs, capturing different
relationships between tasks and supervision sources.
Figure~\ref{fig:graph_structures} (left) is a single-task scenario
where noisy source errors are conditionally independent; this case covers many
benchmark weak supervision datasets.
Here, $D = 1$, and there are no dependencies between
different elements in the dataset (e.g., randomly sampled comments from YouTube
for sentiment analysis).
Figure~\ref{fig:graph_structures} (middle) has dependencies between the errors of two sources ($\lambda_1$ and
$\lambda_2$).
Finally, Figure~\ref{fig:graph_structures} (right) depicts a more complex
scenario, where three tasks have dependencies between
them.
This structure is common in applications with temporal dependencies like video; for example, $Y_1, Y_2, Y_3$ might be contiguous
frames~\cite{sala2019multiresws}. 

\paragraph{Binary Ising Model}
We augment the dependency graph $G_{dep}$ to set up a binary Ising model on
$G = (V, E)$.
Let the vertices $V = \{\bm{Y}, \bm{v}\}$ contain a set of hidden variables $\bm{Y}$ (one for every
task $Y_i$) and observed variables $\bm{v}$, generated by augmenting $\bm{\lf}$.
We generate $\bm{v}$ by letting there be a pair of binary observed variables
$(v_{2i-1}, v_{2i})$ for each label source $\lambda_i$, such that
$(v_{2i-1}, v_{2i})$ is equal to $(1, -1)$ when $\lambda_i = 1$, $(-1, 1)$ when
$\lambda_i = -1$, and $(1, 1)$ or $(-1, -1)$ with equal probability when
$\lambda_i = 0$. This mapping also produces an augmented label matrix $\mathcal{L}$ from the empirical label matrix $L$, which contains $n$ samples of source labels.

Next, let the edges $E$ be constructed as follows.
Let $Y^{dep}(i)$ denote the task that $\lf_i$ labels for all $i \in [1, m]$.
Then for all $i$, there is an edge between each of $(v_{2i-1}, v_{2i})$ and $Y^{dep}(i)$ representing the accuracy of $\lf_i$ as well
as an edge between $v_{2i-1}$ and $v_{2i}$ representing the abstain rate of $\lf_i$.
If there is an edge between $\lf_i$ and $\lf_j$ in $G_{dep}$, then there are
four edges between $(v_{2i-1}, v_{2i})$ and $(v_{2j-1}, v_{2j})$.
We also define $Y(j)$ as the hidden variable that $v_j$ acts on for all $j \in [1, 2m]$; in particular, $Y(2i - 1) = Y^{dep}(i)$.

\paragraph{Inference}
The Ising model defines a joint distribution $P(\bm{Y}, \bm{\lf})$ (detailed
in
\ifarxiv
Appendix \ref{subsec:core_alg}),
\else
Appendix C.1),
\fi
which we wish to use for inference.
We can take advantage of the graphical model properties of $G_{dep}$ for
efficient inference.
In particular, suppose that $G_{dep}$ is triangulated; if not, edges can always
be added to $G_{dep}$ until it is.
Then, $G_{dep}$ admits a junction tree representation with maximal cliques
$C \in \tilde{\mathcal{C}}_{dep}$ and separator sets $S \in \mathcal{S}_{dep}$.
Inference is performed via a standard approach, using the junction tree formula
\begin{align}
P(\bm{Y}, \bm{\lambda}) = {\prod_{C \in \tilde{\mathcal{C}}_{dep}} \mu_C} / {\prod_{S \in \mathcal{S}_{dep}} \mu_S^{d(S) - 1}},
\label{eq:inf}
\end{align}
where $\mu_C$ is the marginal probability of a clique $C$, $\mu_S$ is the
marginal probability of a separator set $S$, and $d(S)$ is the number of maximal cliques $S$ is adjacent to~\cite{Lauritzen, wainwright2008graphical}.
We refer to these marginals as the \textit{label model parameters}
$\bm{\mu}$.

We assume the distribution prior $P(\bm{\bar{Y}})$ is user-provided, but it can also be
estimated directly by using source votes as in \citet{Ratner19} or by optimizing a composite likelihood function as in \citet{chaganty2014estimating}.
Some other marginals are directly observable from the votes generated by the sources
$S_1, \ldots, S_m$.
However, marginals containing elements from both $\bm{Y}$ and $\bm{\lambda}$ are not directly observable, since we do not observe $\bm{Y}$.
The challenge is thus recovering this set of marginals 
$P(Y_i, \ldots, Y_j, \lambda_k, \ldots, \lambda_l)$.

%!TEX root = ../main.tex

\section{Learning The Label Model}
\label{sec:model}

Now that we have defined our label model parameters $\bm{\mu}$, we need to recover
the parameters directly
from the label matrix $L$ without observing the true labels $\bm{Y}$.
First, we discuss how we recover the mean parameters of our Ising model using
Algorithm~\ref{alg:triplet}
(Section~\ref{subsec:model_mean_params}).
Then, we map the mean parameters to label model parameters
(Section~\ref{subsec:model_mapping})
by computing
expectations over cliques of $G$ and applying a linear transform to obtain
$\bm{\mu}$.
Finally, we discuss an extension to the online setting (Section~\ref{subsec:ext}).

\paragraph{Inputs and Outputs}
As input, we take in a label matrix $L$ that has, on average, better-than-random samples; dependency graph $G_{dep}$; and the prior $P(\bm{\bar{Y}})$.
As output, we want to compute $\bm{\mu}$, which would enable us to produce probabilistic training data via \eqref{eq:inf}.

\subsection{Learning the Mean Parameters}
\label{subsec:model_mean_params}
We explain how to compute the mean parameters $\E{}{Y_i}, \E{}{Y_i Y_j}, \E{}{v_i Y(i)}$, and $\E{}{v_i v_j}$ of the Ising model.
Note that all of these parameters can be directly estimated besides $\E{}{v_i Y(i)}$.
Although we cannot observe $Y(i)$, we can compute $\E{}{v_i Y(i)}$ using a
closed-form method by relying on notions of independence and rates of agreement
between groups of three conditionally independent observed variables for the
hidden variable $Y(i)$. Set $a_i := \E{}{v_i Y(i)}$, which can be thought of as the \textit{accuracy} of the observed variable scaled to $[-1, +1]$. The following proposition produces sufficient signal to learn from:
\begin{proposition}
If $v_i \independent v_j | Y(i)$, then $v_i Y(i) \independent v_j Y(i)$. 
\label{lemma:triplet}
\end{proposition}

\begin{algorithm}[t]
	\caption{Triplet Method (before averaging)}
	\begin{algorithmic}
		\STATE \textbf{Input:}
		Set of variables $\Omega_G$, augmented label matrix $\mathcal{L}$ %\vspace{-1em}
		\STATE Initialize $A = \emptyset$	
		\WHILE{$\exists \; v_i \in \Omega_G - A$}
			\STATE Pick $v_j, v_k: v_i \independent v_j | Y(i), v_i \independent v_k | Y(i), v_j \independent v_k | Y(i)$. %\vspace{-1em}
			\STATE Estimate $\hat{\mathbb{E}}[v_i v_j] = \frac{1}{n} \sum_{t} \mathcal{L}_{it} \mathcal{L}_{jt}$, $\hat{\mathbb{E}}[v_i v_k] = \frac{1}{n} \sum_{t} \mathcal{L}_{it} \mathcal{L}_{kt}$, and $\hat{\mathbb{E}}[v_j v_k] = \frac{1}{n} \sum_{t} \mathcal{L}_{jt} \mathcal{L}_{kt}$. 
  		         \STATE $\hat{a}_i \gets \sqrt{|\hat{\mathbb{E}}[v_i v_j] \cdot \hat{\mathbb{E}}[v_i v_k] \, / \,\hat{\mathbb{E}}[v_j v_k]|}$
			\STATE $\hat{a}_j \gets \sqrt{|\hat{\mathbb{E}}[v_i v_j] \cdot \hat{\mathbb{E}}[v_j v_k] \, / \, \hat{\mathbb{E}}[v_i v_k]|}$
			\STATE $\hat{a}_k \gets \sqrt{|\hat{\mathbb{E}}[v_i v_k] \cdot \hat{\mathbb{E}}[v_j v_k] \, / \, \hat{\mathbb{E}}[v_i v_j]|}$
			\STATE $A \gets A \cup \{v_i, v_j, v_k \}$
		\ENDWHILE
		\RETURN \textsc{ResolveSigns}$(\hat{a}_i) \; \forall \; v_i \in V$
	\end{algorithmic}
	\label{alg:triplet}
\end{algorithm}

Our proof is provided in
\ifarxiv
Appendix \ref{subsubsec:prop1proof}.
\else
Appendix C.1.1.
\fi
This follows from a symmetry argument applied to the conditional independence of two variables $v_i$ and $v_j$ given $Y(i)$. Then
\begin{align*}
a_i a_j = \E{}{v_i Y(i)}\E{}{v_j Y(i)} = \E{}{v_i v_j Y(i)^2} = \E{}{v_i v_j},
\end{align*}
where we used $Y(i)^2 = 1$. While we cannot observe $a_i$, the product of $a_i a_j$ is just $\E{}{v_i v_j}$, the observable rate at which a pair of variables act together. We can then utilize a third variable $v_k$ such that $a_i a_k$ and $a_j a_k$ are also observable, and solve a system of three equations for the accuracies up to sign, e.g., $|a_i|, |a_j|, |a_k|$.  We explain how to recover signs with the \textsc{ResolveSigns} function in
\ifarxiv
Appendix \ref{subsubsec:resolvesigns}.
\else
Appendix C.1.5.
\fi

Formally, define $\Omega_G = \{v_i \in V: \exists \; v_j, v_k \; \text{s.t.} \; v_i \independent v_j | Y(i), v_j \independent v_k | Y(i), v_i \independent v_k | Y(i)\}$ to be the set of variables that can be grouped
into triplets in this way.
For each variable $v_i \in \Omega_G$, we can compute the accuracy $a_i$ by solving the system
$a_i a_j = \E{}{v_i v_j}, a_i a_k = \E{}{v_i v_k} , a_j a_k = \E{}{v_j v_k}.$
In many practical settings, $\Omega_G = V$, so the \textit{triplet method} of recovery applies to each $v_i$, motivating Algorithm~\ref{alg:triplet} (some examples of valid triplet groupings shown in Figure~\ref{fig:graph_structures}).
Note that variables can appear in multiple triplets, and variables do not necessarily need
to vote on the same task $Y(i)$ as long as they are conditionally independent given $Y(i)$.  
Different triplets give different accuracy values, so we compute accuracy values
from all possible triplets and use the mean or median over all triplets.
In cases where $\Omega_G$ is not equal to $V$, we supplement the triplet method with other independence properties to recover accuracies on more complex graphs, detailed in
\ifarxiv
Appendix \ref{subsec:extensions}.
\else
Appendix C.2.
\fi

\begin{algorithm}[t]
	\caption{Label Model Parameter Recovery}
	\begin{algorithmic}
		\STATE \textbf{Input:}
		$G_{dep}$, distribution prior $P(\bm{\bar{Y}})$, label matrix $L$.
        \ifsinglecolumn
        \else
        \vspace{-0em} %$\bm{\lf}$
        \fi
		\STATE Augment $G_{dep}$ and $L$ to generate $G = (V, E)$ with cliqueset $\mathcal{C}$ and augmented label matrix $\mathcal{L}$.
		\STATE Obtain set of variables $\Omega_G$ with solvable accuracies. %based on specification of $\lf_i$.	
		\STATE Compute mean parameters and estimate all $\hat{a}_i = \Ehat{v_i Y(i)}$ using Algorithm \ref{alg:triplet}.
		\FOR{clique $C \in \mathcal{C}$ of observed variables}		
			\STATE Compute $\hat{a}_C = \Ehat{\prod_{k \in C} v_k Y(C)}$ by factorizing into observable averages and mean parameters.  %\vspace{-0.8em}
			\STATE Map $\hat{a}_C$ in $G$ to $\hat{a}_{C_{dep}}$ in $G_{dep}$.
			\STATE Linearly transform $\hat{a}_{C_{dep}}$ to $\hat{\mu}_{C_{dep}}$.
		\ENDFOR	
		\RETURN Label model parameters $\bm{\hat{\mu}}$
	\end{algorithmic}
	\label{alg:full}
\end{algorithm}

\subsection{Mapping to the Label Model Parameters}
\label{subsec:model_mapping}
Now we map the mean parameters of our Ising model to label model parameters.
We use the mean parameters to compute relevant expectations over the set $\mathcal{C}$ of all cliques in $G$, map them to expectations over cliques $\mathcal{C}_{dep}$ in $G_{dep}$, and linearly transform them into label model parameters. Define $Y(C)$ as the hidden variable that the entire clique $C \in \mathcal{C}$ of observed variables acts on. Each expectation over a clique of observed variables $C$ and $Y(C)$, denoted $a_C := \E{}{\prod_{k \in C} v_k Y(C)}$, can be factorized in terms of the mean parameters and directly observable expectations
\ifarxiv
(Appendix \ref{subsubsec:large_cliques}).
\else
(Appendix C.1.2).
\fi
For instance, $v_i v_j \independent Y(i, j)$ for $(v_i, v_j) \in E$, such that $\E{}{v_i v_j Y(i, j)} = \E{}{v_i v_j} \cdot \E{}{Y(i, j)}$.

Next, we convert the expectations over cliques in $G$ back into expectations over cliques in $G_{dep}$. 
\ifsinglecolumn
Denote $a_{C_{dep}} \\ := \E{}{\prod_{k \in C_{dep}} \lf_k Y^{dep}(C_{dep})}$
\else
Denote $a_{C_{dep}} := \E{}{\prod_{k \in C_{dep}} \lf_k Y^{dep}(C_{dep})}$
\fi
for each source clique $C_{dep} \in \mathcal{C}_{dep}$; then, there exists a $C \in \mathcal{C}$ over $\{v_{2k - 1}\}_{k \in C_{dep}}$ such that
$a_C = \E{}{\prod_{k \in C_{dep}} v_{2k-1} Y^{dep}(C_{dep})} = a_{C_{dep}}$
\ifarxiv
(Appendix \ref{subsubsec:augment}).
\else
(Appendix C.1.3).
\fi

Finally, the label model parameters, which are marginal distributions over maximal cliques and separator sets, can be expressed as linear combinations of $a_{C_{dep}}$ and probabilities that can be estimated directly from the data. Below is an example of how to recover $\mu_i(a, b) = P(Y^{dep}(i) = a, \lf_i = b)$ from $\E{}{\lf_i Y^{dep}(i)}$:
\begin{align}
\begin{bsmallmatrix}
1 & 1 & 1 & 1 & 1 & 1 \\ 
1 & 0 & 1 & 0 & 1 & 0 \\
1 & 1 & 0 & 0 & 0 & 0 \\
1 & 0 & 0 & 0 & 0 & 1 \\
0 & 0 & 1 & 1 & 0 & 0 \\
0 & 0 & 1 & 0 & 0 & 0
\end{bsmallmatrix} \begin{bsmallmatrix}
\mu_i(1, 1) \\ \mu_i(-1, 1) \\ \mu_i(1, 0) \\ \mu_i(-1, 0) \\ \mu_i(1, -1) \\ \mu_i(-1, -1)
\end{bsmallmatrix} = \begin{bsmallmatrix}
1 \\ P(Y^{dep}(i) = 1) \\ P(\lf_i = 1) \\ P(\lf_i Y^{dep}(i) = 1) \\ P(\lf_i = 0) \\ P(\lf_i = 0, Y^{dep}(i) = 1)
\end{bsmallmatrix}.
\label{eq:maineq}
\end{align}
$P(\lf_i Y^{dep}(i) = 1)$ can be written as $\frac{1}{2}(\E{}{\lf_i Y^{dep}(i)} - P(\lf_i = 0) + 1)$ and $P(\lf_i = 0, Y^{dep}(i) = 1)$ is factorizable due to the construction of $G$, so all values on the right of \eqref{eq:maineq} are known, and we can solve for $\mu_i$. Extending this example to larger cliques requires computing more $a_C$ values and more directly estimatable probabilities; we detail the general case in 
\ifarxiv
Appendix \ref{subsubsec:prod_to_joint}.
\else
Appendix C.1.4.
\fi

\subsection{Weak Supervision in Online Learning}
\label{subsec:ext}

Now we discuss an extension to online learning.
Online learning introduces two challenges: first, samples are introduced one by
one, so we can only see each $\bm{X}^t$ once before discarding it;
second, online learning is subject to \textit{distributional drift},
meaning that the distribution $P_t$ each $(\bm{X}^t, \bm{Y}^t)$ is sampled from changes
over time.
Our closed-form approach is fast, both in terms of sample complexity and wall-clock time, and only requires computing the averages
of observable summary statistics, so we can learn $\bm{\mu}_t$ online with a
rolling window, interleaving label model estimation and end model training.
We describe this online variant of our method and how window size can be adjusted to
optimize for sampling noise and distributional drift in
\ifarxiv
Appendix \ref{subsec:online}.
\else
Appendix C.3.
\fi

%!TEX root = ../main.tex

\section{Theoretical Analysis}
\label{sec:theory}

In this section, we analyze our method for label model parameter recovery and provide bounds on its performance. 
First, we derive a $\mathcal{O}(1/\sqrt{n})$ bound for the sampling error $\|\bm{\hat{\mu}}- \bm{\mu}\|_2$ in Algorithm~\ref{alg:full}. Next, we show that this sampling error has a tight minimax lower bound for certain graphical models, proving that our method is information-theoretically optimal. Then, we present a generalization error bound for the end model that scales in the sampling error and a \textit{model misspecification} term, which exists when the underlying data distribution $\mathcal{D}$ cannot be represented with our graphical model. Lastly, we interpret these results, which are more fine-grained than prior weak supervision analyses, in terms of end model performance and label model tradeoffs. All proofs are provided in
\ifarxiv
Appendix \ref{sec:proofs}.
\else
Appendix D.
\fi

In
\ifarxiv
Appendix \ref{subsubsec:online_theory},
\else
Appendix C.3.1,
\fi
we give two further results for the online variant of the algorithm: selecting an optimal window size to minimize sampling error, and providing a guarantee on end model performance even in the presence of distributional drift, sample noise, and model misspecification.

\textbf{Sampling Error} 
We first control the error in estimating the label model parameters $\bm{\hat{\mu}}$. The noise comes from sampling in the empirical estimates of moments and probabilities used by Algorithm \ref{alg:full}.

\begin{theorem}
Let $\bm{\hat{\mu}}$ be an estimate of $\bm{\mu}$ produced by Algorithm \ref{alg:full} using $n$ unlabeled data points. Then, assuming that cliques in $G_{dep}$ are limited to $3$ vertices, 
\begin{align*}
\E{}{\| \bm{\hat{\mu}} - \bm{\mu} \|_2} \le \frac{1}{a^5_{\min}} \left( 3.19 C_1 \sqrt{\frac{m}{n}} + \frac{6.35 C_2}{\sqrt{r}} \frac{m}{\sqrt{n}}\right),
\end{align*}
where $a_{\min} > 0$ is a lower bound on the absolute value of the accuracies of the sources, and $r$ is the minimum frequency at which sources abstain, if they do so.
\label{thm:sampling_offline}
\end{theorem}

If no sources abstain, $\sqrt{r}$ is not present in the bound. For higher-order cliques, the error scales in $m$ with the size of the largest clique. In the case of full conditional independence, only the first term in the bound is present, so the error scales as $\mathcal{O}\left(\sqrt{\frac{m}{n}}\right)$.

\textbf{Optimality} We show that our method is sample optimal in both $n$ and $m$ up to constant factors for certain graphical models.
We bound the minimax risk for the parameter estimates to be $\Omega \left(\frac{m}{\sqrt{n}} \right)$ via Assouad's Lemma \cite{Yu1997}. This bound holds for any binary Ising model used in our framework, but in particular it is tight when our observed variables are all conditionally independent and do not abstain.  

\begin{theorem}
Let
%\begin{align*}
$\mathcal{P} = \Big\{P(Y, \bm{v}) = \frac{1}{Z} \exp \big( \theta_Y Y + \sum_{i = 1}^m \theta_i v_i Y \big), %\; \\
%&
\theta \in \mathbb{R}^{m + 1} \Big\}$
be a family of distributions.
%\end{align*}
Using $L_2$ norm estimation of the minimax risk, the sampling error is
lower bounded as
\begin{align*}
\inf_{\bm{\hat{\mu}}} \sup_{P \in \mathcal{P}} &\E{P}{||\bm{\hat{\mu}} - \bm{\mu}(P)||_2} \ge \frac{e_{min}}{8} \sqrt{\frac{m}{n}}.
\end{align*}
Here $\bm{\mu}(P)$ is the set of label model parameters corresponding to a distribution $P$, and $e_{min}$ is the minimum eigenvalue of $\Cov{}{Y, \bm{v}}$ for distributions in $\mathcal{P}$.

\label{thm:lower_bound}
\end{theorem}

\textbf{Generalization Bound} We provide a bound quantifying the performance gap between the end model parametrization that uses outputs of our label model and the optimal end model parametrization over the true distribution of labels.

Let $P_{\bm{\hat{\mu}}}(\cdot | \bm{\lf})$ be the probabilistic output of our learned label model parametrized by $\bm{\hat{\mu}}$ given some source labels $\bm{\lf}$. Define a loss function $L(w, \bm{X}, \bm{Y}) \in [0, 1]$, where $w$ parametrizes the end model $f_w \in \mathcal{F}: \mathcal{X} \rightarrow \mathcal{Y}$, and choose $\hat{w}$ such that  
\begin{align*}
\hat{w} = \argmin{w}{\frac{1}{n} \sum_{i = 1}^n \E{\bm{\widetilde{Y}} \sim P_{\bm{\hat{\mu}}}(\cdot | \bm{\lf}(\bm{X}^i))}{L(w, \bm{X}^i, \bm{\widetilde{Y}})}}.
\end{align*}

While previous approaches \cite{Ratner19} make the strong assumption that there exists some $\bm{\mu}$ such that sampling $(\bm{X}, \bm{\widetilde{Y}})$ from $P_{\bm{\mu}}$
is equivalent to sampling from $\mathcal{D}$, 
our generalization error bound accounts for potential model misspecification:

\begin{theorem}
Let $w^* = \argmin{w}{\E{(\bm{X}, \bm{Y}) \sim \mathcal{D}}{L(w, \bm{X}, \bm{Y})}}$. There exists a $\hat{w}$ computed from the outputs of our label model such that the generalization error for $\bm{Y}$ satisfies
\ifsinglecolumn
\begin{align*}
\E{\mathcal{D}}{L(\hat{w}, \bm{X}, \bm{Y}) - L(w^*, \bm{X}, \bm{Y})} 
\le  \; \gamma(n) + \frac{8 |\mathcal{Y}|}{e_{min}} ||\bm{\hat{\mu}} - \bm{\mu}||_2 + \delta(\mathcal{D}, P_{\bm{\mu}}),
\end{align*}
\else
\begin{align*}
&\E{\mathcal{D}}{L(\hat{w}, \bm{X}, \bm{Y}) - L(w^*, \bm{X}, \bm{Y})} \\
&\qquad \le  \; \gamma(n) + \frac{8 |\mathcal{Y}|}{e_{min}} ||\bm{\hat{\mu}} - \bm{\mu}||_2 + \delta(\mathcal{D}, P_{\bm{\mu}}),
\end{align*}
\fi
where $\delta(\mathcal{D}, P_{\bm{\mu}}) = 2\sqrt{2 \, KL (\mathcal{D}(\bm{Y} | \bm{X}) \;|| \;P_{\bm{\mu}}(\bm{Y} | \bm{X}))}
$, $e_{min}$ is the minimum eigenvalue of $\Cov{}{\bm{Y}, \bm{v}}$ over the construction of the binary Ising model, and $\gamma(n)$ is a decreasing function that bounds the error from performing empirical risk minimization to learn $\hat{w}$.
\label{thm:gen_offline}
\end{theorem}

\textbf{Interpreting the Bounds} The generalization error in Theorem \ref{thm:gen_offline} has two components, involving the noise awareness of the model and the model misspecification. Using the sampling error result, the first two terms $\gamma(n)$ and $||\bm{\hat{\mu}} - \bm{\mu}||_2$ scale in $\mathcal{O}(1/\sqrt{n})$, which can be tight by Theorem~\ref{thm:lower_bound} and is the same asymptotic rate as supervised approaches.

The third term $\delta(\mathcal{D}, P_{\bm{\mu}})$ is a divergence between our model and $\mathcal{D}$. Richer models can represent more distributions and have a smaller KL term, but may suffer a higher sample complexity. This tradeoff suggests the importance of selecting an appropriately constrained graphical model in practice.

%!TEX root = ../main.tex

\section{Evaluation}
\label{sec:evaluation}

%!TEX root = ../main.tex

\begin{table*}[t]
    \centering
    \tiny
    \begin{tabular}{@{}rlccccccccccccc@{}}
        \toprule
        \multicolumn{5}{c}{} & \multicolumn{5}{c}{\textbf{End Model Performance (F1), Label Model Training Time (s)}}                                                                               & \multicolumn{4}{c}{\textbf{Lift, Speedup}} \\
        \cmidrule(l){6-10} \cmidrule(l){11-14} 
        & \textbf{Task}                & $D$                & $m$                 & \textbf{Prop}         & \textbf{TS}     & \textbf{MV}    & \textbf{DP}    & \textbf{SDP}            & \textbf{\sysx\ (l.m. in paren.)} & \textbf{TS} & \textbf{MV} & \textbf{DP} & \textbf{SDP}  \\
        \midrule
        \parbox[t]{0mm}{\multirow{7}{*}{\rotatebox[origin=c]{90}{\textbf{Benchmarks}}}}
        & \multirow{2}{*}{\spouse}     & \multirow{2}{*}{1} & \multirow{2}{*}{9}  & \multirow{2}{*}{0.07} & 20.4 $\pm$ 0.2  &19.3 $\pm$ 0.01 & 44.7 $\pm$ 1.7 & --                      & \textbf{49.6 $\pm$ 2.4} (47.0)   & +29.3       & +30.3       & +4.9        & --            \\
        &                              &                    &                     &                       & --              & --             & 7.5 $\pm$ 0.9  & --                      & \textbf{0.017 $\pm$ 0.003}       & --          & --          & 440$\times$ & --            \\
        \cmidrule(l){2-14}
        & \multirow{2}{*}{\spam}       & \multirow{2}{*}{1} & \multirow{2}{*}{10} & \multirow{2}{*}{0.49} & 91.5            & 88.3           & 91.8           & --                      & \textbf{92.3} (89.1)             & +0.8        & +4.0        & +0.5        & --            \\
        &                              &                    &                     &                       & --              & --             & 0.76 $\pm$ 0.1 & --                      & \textbf{0.014 $\pm$ 0.002}       & --          & --          & 54$\times$  & --            \\
        \cmidrule(l){2-14}
        & \multirow{2}{*}{\weather}    & \multirow{2}{*}{1} & \multirow{2}{*}{103}& \multirow{2}{*}{0.53} & 74.6            & 87.3           & 87.3           & --                      & \textbf{88.9} (77.6)             & +14.3       & +1.6        & +1.6        & --            \\
        &                              &                    &                     &                       & --              & --             & 0.78 $\pm$ 0.1 & --                      & \textbf{0.150 $\pm$ 0.03}        & --          & --          & 5.2$\times$ & --            \\
        \midrule
        \parbox[t]{0mm}{\multirow{10}{*}{\rotatebox[origin=c]{90}{\textbf{Video Analysis}}}}
        & \multirow{2}{*}{\interview}  & \multirow{2}{*}{6} & \multirow{2}{*}{24} & \multirow{2}{*}{0.03} & 80.0 $\pm$ 3.4  & 58.0 $\pm$ 5.3 &  8.7 $\pm$ 0.2 & 92.0 $\pm$ 2.2          & 91.9 $\pm$ 1.6 (\textbf{93.0})   & +11.9       & +33.9       & +83.2       & -0.1          \\
        &                              &                    &                     &                       & --              & --             & 31.5 $\pm$ 1.0 & 256.6 $\pm$ 5.4         & \textbf{0.423 $\pm$ 0.04}        & --          & --          &74.5$\times$ &   607$\times$ \\
        \cmidrule(l){2-14}
        & \multirow{2}{*}{\commercial} & \multirow{2}{*}{6} & \multirow{2}{*}{24} & \multirow{2}{*}{0.32} & 90.9 $\pm$ 1.0  & 91.8 $\pm$ 0.2 & 90.5 $\pm$ 0.4 & 89.8 $\pm$ 0.5          & \textbf{92.3 $\pm$ 0.4} (88.4)   & +1.4        & +0.5        & +1.8        & +2.5          \\
        &                              &                    &                     &                       & --              & --             & 23.3 $\pm$ 1.0 & 265.6 $\pm$ 6.2         & \textbf{0.067 $\pm$ 0.01}        & --          & --          & 350$\times$ & 4,000$\times$ \\
        \cmidrule(l){2-14}
        & \multirow{2}{*}{\tennis}     & \multirow{2}{*}{14}& \multirow{2}{*}{84} & \multirow{2}{*}{0.34} & 57.6 $\pm$ 3.4  & 80.2 $\pm$ 1.0 & 82.5 $\pm$ 0.3 & 80.6 $\pm$ 0.7          & \textbf{82.8 $\pm$ 0.4} (82.0)   & +25.2       & +2.6        & +0.3        & +2.2          \\
        &                              &                    &                     &                       & --              & --             & 41.1 $\pm$ 1.9 & 398.4 $\pm$ 7.5         & \textbf{0.199 $\pm$ 0.04}        & --          & --          & 210$\times$ & 2,000$\times$ \\
        \cmidrule(l){2-14}
        & \multirow{2}{*}{\basketball} & \multirow{2}{*}{8} & \multirow{2}{*}{32} & \multirow{2}{*}{0.12} & 26.8 $\pm$ 1.3  & 8.1 $\pm$ 5.4  & 7.7 $\pm$ 3.3  & \textbf{38.2 $\pm$ 4.1} & 37.9 $\pm$ 1.9 (27.9)            & +11.1       & +29.8       & +30.2       & -0.3          \\
        &                              &                    &                     &                       & --              & --             & 28.7 $\pm$ 2.0 & 248.6 $\pm$ 7.7         & \textbf{0.092 $\pm$ 0.03}        & --          & --          & 310$\times$ & 2,700$\times$ \\
        \bottomrule \\
        \vspace{-3em}
        \end{tabular}
        \caption{
        \sysx\ performance in terms of F1 score (first row of each task), and
        label model training time in seconds (second row).
        We report mean $\pm$ standard deviation across five random weight
        initializations of the end model (except for \spam\ and \weather, which
        use logistic regression).
        Improvement in terms of mean end model lift, speedup in terms of mean
        runtime.
        We compare \sysx's end model and label model (label
        model in parentheses) against traditionally supervised (TS) end models trained on the labeled dev set,
        majority vote (MV), data programming (DP) and sequential
        data programming (SDP).
        $D$: number of related elements modeled (contiguous sequences of frames
        for video tasks).
        $m$: number of supervision sources.
        Prop: proportion of positive examples.
        }
    \label{table:offline}
\end{table*}

The primary goal of our evaluation is to validate that \sysx\ can achieve the
same or higher quality as state-of-the-art weak supervision frameworks 
(Section~\ref{sec:eval_quality}) while
learning label model parameters orders of magnitude faster
(Section~\ref{sec:eval_speedup}).
We also evaluate the online extension and discuss how online learning can be
preferable to offline learning in the presence of distributional shift over
time (Section~\ref{sec:eval_online}).

\paragraph{Datasets}
We evaluate \sysx\ on three benchmark datasets and four video analysis tasks.
Each dataset consists of a large (\num{187--64,130}) unlabeled training set, a
smaller (\num{50--9,479}) hand-labeled \textit{development set}, and a held-out
test set.
We use the unlabeled training set to train the label model and end model, and
use the labeled development set for a) training a
traditional supervision baseline, and b) for hyperparameter tuning of the label and end
models.
More details about each task and the experiments in 
\ifarxiv
Appendix \ref{sec:extexp}.
\else
Appendix E.
\fi

\textit{Benchmark Tasks.}
We draw three benchmark weak supervision datasets from a previous
evaluation of a state-of-the-art weak supervision framework~\cite{Ratner18}.
\spouse\ seeks to identify mentions of spouse relationships in
a set of news articles~\cite{corney2016million},
\spam\ classifies whether YouTube comments are spam~\cite{alberto2015tubespam},
and \weather\ is a weather sentiment task from
Crowdflower~\cite{CrowdflowerWeather}.

\textit{Video Analysis Tasks.}
We use video analysis as another driving task: video data is large and
expensive to label, and modeling temporal dependencies is
important for quality but introduces significant slowdowns in label model parameter
recovery~\cite{sala2019multiresws}.
\interview\ and \commercial\ identify interviews with Bernie Sanders
and commercials in a corpus of TV news, respectively~\cite{fu2019rekall,
InternetArchive}.
\tennis\ identifies tennis rallies during a match from broadcast footage.
\basketball\ identifies basketball videos in a subset of
ActivityNet~\cite{caba2015activitynet}.

\subsection{Quality}
\label{sec:eval_quality}
We now validate that end models trained with labels generated by \sysx\
achieve the same or higher quality as previous state-of-the-art weak
supervision frameworks.
We also discuss the relative performance of \sysx's label model compared to the
end model, and ablations of our method.

\paragraph{End Model Quality}
To evaluate end model quality, we use \sysx\ to generate labels for the
unlabeled training set and compare the end models trained with these labels
against four baselines:
\begin{enumerate}[itemsep=0.5pt,topsep=0pt]
    \item \textit{Traditional Supervision} \textbf{[TS]}: We train the end
    model using the small hand-labeled development set.
    \item \textit{Majority Vote} \textbf{[MV]}: We generate training labels
    over the unlabeled training set using majority vote.
    \item \textit{Data Programming} \textbf{[DP]}: We use data programming, a
    state-of-the-art weak supervision framework that models each data point
    separately~\cite{Ratner19}.
    \item \textit{Sequential Data Programming} \textbf{[SDP]}:
    For the video tasks, we also use a state-of-the-art sequential weak
    supervision framework, which models sequences of frames~\cite{sala2019multiresws}.
\end{enumerate}
Table~\ref{table:offline} shows our results.
We achieve the same or higher end model quality compared to previous weak
supervision frameworks.
Since \sysx\ does not rely on SGD to learn
label model parameters, there are fewer hyperparameters to tune, which can help
us achieve higher quality than previous reported results.

\paragraph{Label Model vs. End Model Performance}
Table~\ref{table:offline} also shows the performance of \sysx's label
model.
In four of the seven tasks, the end model outperforms the label
model, since it can learn new features directly from the input data that are
not available to the noisy sources.
For example, the sources in the \commercial\ task rely on simple visual
heuristics like the presence of black frames (in our dataset,
commercials tend to be book-ended on either side by black frames); the end
model, which is a deep network, is able to pick up on subtler features over the
pixel space.
In three tasks, however, the label model nearly matches or slightly outperforms
the end model.
In these cases, the sources have access to features that are difficult
for an end model to learn with the amount of unlabeled data available.
For example, the sources in the \interview\ task rely on an identity
classifier that has learned to identify Bernie Sanders from thousands of
examples.

\paragraph{Ablations}
We describe the results of two ablation studies (detailed results in
\ifarxiv
Appendix \ref{sec:supp_ablation}).
\else
Appendix E.4).
\fi
In the first study, we replace abstentions with random votes instead of
augmenting $G_{dep}$.
This results in a
degradation of \avgliftrandomabstains\ points, demonstrating the importance of
allowing supervision sources to abstain.
In the second study, we examine the effect of using individual triplet
assignments instead of taking the median or mean over all possible assignments.
On average, taking random assignments results in a degradation of \avgliftsingletriplet\
points compared to taking an aggregate.
Furthermore, there is a large degree of variance in label model performance
when using individual triplet assignments.
While the best assignments can match \sysx, bad assignments result in
significantly worse performance.

\subsection{Speedup}
\label{sec:eval_speedup}
We now evaluate the speedup that \sysx\ provides over previous weak
supervision frameworks.
Table~\ref{table:offline} shows measurements of how long it takes to train each
label model.
Since \sysx\ learns source accuracies and correlations with a closed-form
solution, it runs orders of magnitude faster than previous weak supervision
frameworks, which rely on multiple iterations of stochastic gradient descent
and thus scale superlinearly in the data.
Speedup varies due to the optimal number of iterations for DP and SDP, which
are SGD-based (number of iterations is tuned for accuracy), but \sysx\ runs
up to \maxspeedupdpbench\ times
faster than data programming on benchmark tasks, and up to \maxspeedupdugong\
times faster than sequential data programming on the video tasks (where
modeling sequential dependencies results in much slower performance).

%!TEX root = ../neurips_2019.tex

\begin{table}[t]
    \centering
    \ifsinglecolumn
    \else
    \tiny
    \fi
    \begin{tabular}{@{}lccccc@{}}
        \toprule
        \multicolumn{1}{c}{} & \multicolumn{3}{c}{\textbf{Streaming End Model (F1)}}         & \multicolumn{2}{c}{\textbf{Improvement}} \\
        \cmidrule(l){2-4} \cmidrule(l){5-6} 
        \textbf{Task}        & \textbf{TS}     & \textbf{MV}    & \textbf{\sysx}             & \textbf{TS} & \textbf{MV}  \\
        \midrule
        \interview           & 41.9 $\pm$ 4.0  & 37.8 $\pm$ 9.5 & \textbf{53.5 $\pm$ 0.5}    & +11.6       & +15.7        \\
        \commercial          & 56.5 $\pm$ 1.7  & 78.9 $\pm$ 14.5& \textbf{93.0 $\pm$ 0.5}    & +36.5       & +14.1        \\
        \tennis              & 41.5 $\pm$ 1.7  & 81.6 $\pm$ 0.6 & \textbf{82.7 $\pm$ 0.4}    & +25.2       & +1.1         \\
        \basketball          & 20.7 $\pm$ 4.2  & 22.0 $\pm$ 11.3& \textbf{26.7 $\pm$ 0.3}    & +6.0        & +4.7         \\
        \bottomrule \\
        \ifsinglecolumn
        \else
        \vspace{-3em}
        \fi
        \end{tabular}
        \caption{
        We compare performance of an end model trained with an online pass over
        the training set, and then the test set with labels from \sysx,
        against a model trained with majority vote (MV) labels over the training and test set, and a 
        traditionally supervised (TS) model trained with ground truth labels over the test set.
        We report mean $\pm$ standard deviation from five random weight
        initializations.
        }
    \label{table:online}
\end{table}

\subsection{Online Weak Supervision}
\label{sec:eval_online}
We now evaluate the ability of our online extension to simultaneously train a
label model and end model online for our video analysis tasks.
We also use synthetic experiments to demonstrate when training a model online
can be preferable to training a model offline.

\paragraph{Core Validation}
We first validate our online extension by using the \sysx\ PyTorch layer to
simultaneously train a label model and end model online for our video analysis tasks.
We train first on the training set and then on the test set (using probabilistic labels for both).
We compare against online traditional supervision (TS) and majority vote (MV)
baselines.
Since the training set is unlabeled, the TS model is trained only on the
ground-truth test set labels, while the MV baseline uses majority vote to label the
training and test sets.
To mimic the online setting, each datapoint is only seen once during training.

Table~\ref{table:online} shows our results.
Our method outperforms MV by up to \maxliftmvonline\ F1 points, and
TS by up to \maxlifttsonline\ F1 points.
Even though TS is trained on ground-truth test set labels,
it underperforms both other methods because it only does a single pass over the
(relatively small) test set.
MV and \sysx, on the other hand, see many more examples in the weakly-labeled
training set before having to classify the test set.

The online version of \sysx\ often underperforms its offline equivalent
(Table~\ref{table:offline}),
since the online model can only perform a single iteration of SGD with
each datapoint.
However, in 2 cases, the online model overperforms the offline
model, for two reasons: a) the training set is large enough to make up the
difference in having multiple epochs with SGD, and b) online training over
the test set enables continued specialization to the test set.

\paragraph{Distributional Drift Over Time}
We also study the effect of distributional drift over time using synthetic experiments.
Distributional drift can mean that label model
parameters learned on previous data points may not describe future data points.
Figure~\ref{fig:micro} shows the results of online vs. offline training in two
settings with different amounts of drift.
On the left is a setting with limited drift; in this setting, the
offline model learns better parameters than the online model, since it has
access to more data, all of which is representative of the test set.
On the right is a setting with large amounts of periodic drift; in this
setting, the offline model cannot learn parameters that work for all data
points.
But the online model, which only learns parameters for a recent window of data
points, is able to specialize to the periodic shifts.

\begin{figure}[t]
  \centering
  \includegraphics[width=3.25in]{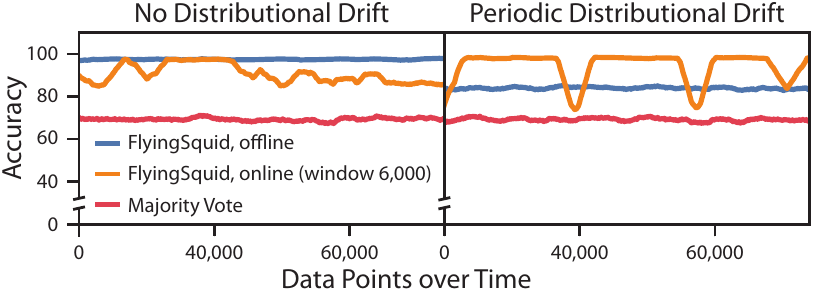}
  \ifsinglecolumn
  \else
  \vspace{-2em}
  \fi
  \caption{
  When there is large distributional drift, online learning can outperform
  offline learning by adapting over time (synthetic).
  }
  \label{fig:micro}
\end{figure}

%!TEX root = ../main.tex

\section{Related Work}
\label{sec:related}

\paragraph{Latent Variable Estimation}
Latent variable estimation is a classic problem in machine learning, used for hidden Markov Models, Markov random
fields, topic modeling, and
more~\cite{wainwright2008graphical, koller2009probabilistic}.
General algorithms do not admit closed-form
solutions; classical techniques like expectation maximization and Gibbs sampling can
require many iterations to converge, while techniques like tensor decomposition
run the expensive power method~\cite{anandkumar2014tensor}.
We show that the weak supervision setting allows us to break down
the parameter estimation problem into subproblems with closed-form solutions.

Our solution is similar to previous methods that have exploited triplets of
conditionally-independent variables to solve latent variable
estimation~\cite{joglekar2013evaluating, chaganty2014estimating}.
\citet{joglekar2013evaluating}~focuses on the explicit context of crowdsourcing
and is equivalent to a simplified version of Algorithm~\ref{alg:triplet} when
all the label sources are conditionally independent from each other and do not
abstain.
In contrast, our work handles a wider variety of use cases critical for
weak supervision (such as sources that can abstain) and develops theoretical characterizations for downstream
model behavior.
\citet{chaganty2014estimating}~shows how to estimate the canonical parameters
of a wide class of graphical models by applying tensor decomposition to recover
conditional parameters.
By comparison, our work is more specialized, which lets us replace tensor
decomposition with a non-iterative closed-form solution, even for
non-binary variables.
A more detailed comparison against both of these methods is available in
\ifarxiv
Appendix \ref{sec:extended_related}.
\else
Appendix A.
\fi

\paragraph{Weak Supervision}
Our work is related to several such techniques, such as
distant supervision~\cite{mintz2009distant,craven1999constructing,
hoffmann2011knowledge,takamatsu2012reducing}, co-training
methods~\cite{blum1998combining}, pattern-based
supervision~\cite{gupta2014improved} and feature annotation~\cite{mann2010generalized,zaidan2008modeling,liang2009learning}.
Recently, weak supervision frameworks rely on latent graphical models
and other methods to systematically integrate multiple noisy
sources~\cite{Ratner16, Ratner18, bach2017learning, bach2018snorkel,
guan2018said, khetan2017learning, sheng2020gmail, re2019overton}.
Two recent approaches have proposed new methods for modeling sequential
dependencies in particular, which is important in applications like
video~\cite{zhan2019sequentialws, sala2019multiresws, safranchik2020weakly}.
These approaches largely rely on iterative methods like stochastic
gradient descent, and do not run closed-form solutions to latent variable estimation.

\paragraph{Crowdsourcing}
Our work is related to crowdsourcing (crowd workers can be thought of as
noisy label sources).
A common approach in crowdsourcing is filtering crowd workers using a small set
of gold tasks, or filtering based on number of previous tasks completed or with
monetary incentives~\cite{rashtchian2010collecting, shaw2011designing,
sorokin2008utility, downs2010your, mitra2015comparing,kittur2008crowdsourcing}.
In contrast, in our setting, we do not have access to ground truth data to
estimate source accuracies, and we cannot filter out noisy sources \textit{a
priori}.
Other techniques can estimate worker accuracies without ground truth
annotations, but assume that workers are independent~\cite{karger2011iterative}.
We can also directly model crowd workers using our label model, as in the
\weather\ task.

\paragraph{Online Learning}
Training models online traditionally requires hand
labels~\cite{cesa2006prediction, shalev2012online}, but recent
approaches like~\citet{mullapudi2019online} train models online using a
student-teacher framework (training a student network online based on the
outputs of a more powerful teacher network).
In contrast, our method does not rely on a powerful network that has been
pre-trained to carry out the end task.
In both traditional and newer distillation settings, a critical challenge is
updating model parameters to account for domain shift~\cite{shalev2012online}.
For our online setting, we deal with distributional drift via a standard rolling window.

%!TEX root = ../main.tex

\section{Conclusion}
\label{sec:conc}

We have proposed a method for latent variable estimation by
decomposing it into minimal subproblems with closed-form solutions.
We have used this method to build \sysx, a new weak supervision framework that
achieves the same or higher quality as previous approaches while running orders
of magnitude faster, and presented an extension to online learning embodied in
a novel \sysx\ layer.
We have proven generalization and sampling error bounds
and shown that our method can be sample optimal.
In future work, we plan to extend our insights to more problems where
closed-form latent variable estimation can result in faster algorithms or new
applications---problems such as structure learning and data augmentation.

%\ifarxiv
\section*{Acknowledgments}

We thank Avanika Narayan for helping with the Tennis dataset, and Avner May for helpful discussions.
We gratefully acknowledge the support of DARPA under Nos. FA86501827865 (SDH) and FA86501827882 (ASED); NIH under No. U54EB020405 (Mobilize), NSF under Nos. CCF1763315 (Beyond Sparsity), CCF1563078 (Volume to Velocity), and 1937301 (RTML); ONR under No. N000141712266 (Unifying Weak Supervision); the Moore Foundation, NXP, Xilinx, LETI-CEA, Intel, IBM, Microsoft, NEC, Toshiba, TSMC, ARM, Hitachi, BASF, Accenture, Ericsson, Qualcomm, Analog Devices, the Okawa Foundation, American Family Insurance, Google Cloud, Swiss Re,
Brown Institute for Media Innovation, the HAI-AWS Cloud Credits for Research program,
Department of Defense (DoD) through the National Defense Science and
Engineering Graduate Fellowship (NDSEG) Program, 
Fannie and John Hertz Foundation,
National Science Foundation Graduate Research Fellowship under Grant No. DGE-1656518,
Texas Instruments Stanford Graduate Fellowship in Science and Engineering,
and members of the Stanford DAWN project: Teradata, Facebook, Google, Ant Financial, NEC, VMWare, and Infosys. The U.S. Government is authorized to reproduce and distribute reprints for Governmental purposes notwithstanding any copyright notation thereon. Any opinions, findings, and conclusions or recommendations expressed in this material are those of the authors and do not necessarily reflect the views, policies, or endorsements, either expressed or implied, of DARPA, NIH, ONR, or the U.S. Government.

%\fi

\bibliography{main}
\bibliographystyle{icml2020}

\ifarxiv
\onecolumn
\appendix

First, we provide an extended discussion of related work.
Next, we provide a glossary of terms and notation that we use throughout this
paper for easy summary.
Next, we discuss additional algorithmic details, and we give the proofs of our
main results (each theorem).
Finally, we give additional experimental details.

\section{Extended Related Work}
\label{sec:extended_related}

The notion of the ``triplet" of (conditionally) independent variables as the source of minimal signal in latent variable models was observed and exploited in two innovative works, both using moments to deal with the challenge of the latent variable. These are
\begin{itemize}
\item \citet{joglekar2013evaluating}, in the explicit context of crowdsourcing, and
\item \citet{chaganty2014estimating}, for estimating the parameters of certain latent variable graphical models.
\end{itemize}

The ``3-Differences Scheme" described in 3.1 of \citet{joglekar2013evaluating} is equivalent to our approach in Algorithm 1 %~\ref{alg:triplet} 
in the basic case where there are no abstains and the signs of the accuracies are non-negative. \citet{joglekar2013evaluating} focuses on crowdsourcing, and thus offers two contributions for this setting: (i) computing confidence intervals for worker accuracies and (ii) a set of techniques for extending the three-voters case by collapsing multiple voters into a pair `super-voters' in order to build a better triplet for a particular worker. Both of these are useful directions for extensions of our work. In contrast, our approach focuses on efficiently handling the non-binary abstains case critical for weak supervision and develops theoretical characterizations for the downstream model behavior when using our generated labels.

A more general approach to learning latent variable graphical models is described in \citet{chaganty2014estimating}. Here there is an explicit description of the ``three-views" approach. It is shown how to estimate the canonical parameters of a remarkably wide class of graphical models (e.g., both directed and undirected) by applying the tensor decomposition idea (developed in \citet{anandkumar2014tensor}) to recover conditional parameters. By comparison, our work is more specialized, looking at undirected (in fact, specifically Ising) models in the context of weak supervision. The benefits of this specialization are that we can replace the use of the tensor power iteration technique with a non-iterative closed-form solution, even for non-binary variables. Nevertheless, the techniques in \citet{chaganty2014estimating} can be useful for weak supervision as well, and their pseudolikelihood approach to recover canonical parameters suggests that forward methods of inference could be used in our label model. We also note that closed-form triplet methods can be used to estimate \emph{part} of the parameters of a more complex exponential family model (where some variables are involved in pairwise interactions at most, others in more complex patterns), so that resorting to tensor power iterations can be minimized. 

A further work that builds on the approach of \citet{chaganty2014estimating} is \citet{raghunathan2016estimation}, where moments are used in combination with a linear technique. However, the setting here is different from weak supervision. The authors of \citet{raghunathan2016estimation} study \emph{indirect} supervision. Here, for any unlabeled data point $x$, the label $y$ is not seen, but a variable $o$ is observed. So far this framework resembles weak supervision, but in the indirect setting, the supervision distribution $S(o | y)$ is known---while for weak supervision, it is not. Instead, in \citet{chaganty2014estimating}, the $S$ distribution is given for two particular applications: local privacy and a light-weight annotation scheme for POS tagging.

\section{Glossary}
\label{sec:gloss}

The glossary is given in Table~\ref{table:glossary} below.
\begin{table*}[h]
\centering
\small
\begin{tabular}{l l}
\toprule
Symbol & Used for \\
\midrule
$\bm{X}$ & Unlabeled data vector, $\bm{X} = [X_1, X_2, \ldots, X_D] \in \mathcal{X}$ \\
$\bm{X}^i$ & $i$th unlabeled data vector \\
$X_i$ & $i$th data element  \\
$D$ & Length of the unlabeled data vector \\
$\bm{Y}$ & Latent, ground-truth label vector, $\bm{Y} = [Y_1, Y_2, \ldots, Y_D] \in \mathcal{Y}$, also referred to as hidden variables  \\
$\bm{Y}^i$ & $i$th ground-truth label vector \\
$Y_i$ & Ground-truth label for $i$th task, $Y_i \in \{-1, +1\}$  \\
$\mathcal{D}$ & Distribution from which we assume $(\bm{X}, \bm{Y})$ data points are sampled i.i.d. \\
$S_i$ & $i$th weak supervision source\\
$m$ & Number of weak supervision sources \\
$\lf_i$ & Label of $S_i$ for $\bm{X}$ where $\lf_i \in \{-1, 0, 1\}$; all $m$ labels per $\bm{X}$ collectively denoted $\bm{\lf}$ \\
$n$ & Number of data vectors \\
$\bm{\widetilde{Y}}$ & Probabilistic training labels for a label vector \\
$f_w$ & Discrimative classifier used as end model, parametrized by $w$ \\
$G_{dep}$ & Source dependency graph \\ %describing the correlation structure among tasks in a graph\\
$G$ & Augmented graph $G = (V, E)$ used for binary Ising model, where $V=\{\bm{Y}, \bm{v}\}$ \\
$\bm{v}$ & Observed variables of the graphical model corresponding to $\bm{\lf}$ \\
$L$ & Label matrix containing $n$ samples of source labels $\lf_1, \ldots, \lf_m$\\
$\mathcal{L}$ & Augmented label matrix computed from $L$ \\
$Y^{dep}(i)$ & Task that $\lf_i$ labels \\
$Y(i)$ & Hidden variable that the observed variable $v_i$ acts on \\
$\mathcal{C}_{dep}$ & Cliqueset (maximal and non-maximal) of $G_{dep}$ \\
$\tilde{\mathcal{C}}_{dep}, \mathcal{S}_{dep}$  & The maximal cliques and separator sets of the junction tree over $G_{dep}$ \\
$\bm{\mu}$ & The label model parameters collectively over all $\mu_C$, $\mu_S$, the marginal distributions of $C \in \tilde{\mathcal{C}}_{dep}$, $S \in \mathcal{S}_{dep}$ \\
$P(\bm{\bar{Y}})$ & Class prior for the $\bm{Y}$ label vector \\
$a_i$ & $\E{}{v_i Y(i)}$, the unobservable mean parameters of binary Ising model $G$ \\
$\Omega_G$ & Set of vertices in $V$ to which the triplet method can be applied \\
$\mathcal{C}$ & Cliqueset (maximal and non-maximal) of $G$ \\
$a_C$ & The expectation over the product of observed variables in clique $C \in \mathcal{C}$ and $Y(C)$\\
$a_{C_{dep}}$ & The expectation over the product of sources in clique $C_{dep} \in \mathcal{C}_{dep}$ and $Y^{dep}(C_{dep})$\\
\toprule
\end{tabular}
\caption{
	Glossary of variables and symbols used in this paper.
}
\label{table:glossary}
\end{table*}

\section{Further Algorithmic Details}
\label{sec:extalg}

In this section, we present more details on the main algorithm, extensions to more complex models, and the online variant.

\subsection{Core Algorithm}
\label{subsec:core_alg}

We first present the general binary Ising model and the proof of Proposition 1 %~\ref{lemma:triplet} 
that follows from this construction. We also prove another independence property over this general class of Ising models that can be used to factorize expectations over arbitrarily large cliques. Next, we detail the exact setup of the graphical model when sources can abstain, as well as the special case when they never abstain, and define the mappings necessary to convert between values over $\bm{v}, G$ and $\bm{\lf}, G_{dep}$. We then formalize the linear transformation from $a_{C_{dep}}$ to $\mu_{C_{dep}}$, and finally we explain the \textsc{ResolveSigns} function used in Algorithm 1. %~\ref{alg:triplet}. 

First, we give the explicit form of the density for the Ising model we use. Given the graph $G = (V,E)$, we can write the corresponding joint distribution of $\bm{Y}, \bm{v}$ as
\begin{align}
f_G(\bm{Y}, \bm{v}) = \frac{1}{Z} \exp \Big(\sum_{k = 1}^D \theta_{Y_k} Y_k + \sum_{(Y_k, Y_l) \in E} \theta_{Y_k, Y_l} Y_k Y_l + \sum_{v_i \in \bm{v}} \theta_i v_i Y(i) + \sum_{(v_k, v_l) \in E} \theta_{k, l} v_k v_l \Big),
\label{eq:vising}
\end{align}
where $Z$ is the partition function, and the $\theta$ terms collectively are the canonical parameters of the model. Note that this is the most general definition of the binary Ising model with multiple dependent hidden variables and observed variables that we use.

\subsubsection{Proof of Proposition 1}%\ref{lemma:triplet}}
\label{subsubsec:prop1proof}

%\subsection{Proof of Proposition \ref{lemma:triplet}}

We present the proof of Proposition $1$, which is the underlying independence property of \eqref{eq:vising} that enables us to use the triplet method. We aim to show that for any $a, b \in \{-1, +1\}^2$, 
\begin{align}
P\big(v_i Y(i) = a, v_j Y(i) = b\big) = P(v_i Y(i) = a) \cdot P(v_j Y(i) = b),
\label{eq:prop1IMS}
\end{align} 

where $v_i \independent v_j | Y(i)$. For now, assume that $Y(j) \neq Y(i)$. 

Because $v_i$ and $v_j$ are conditionally independent given $Y(i)$, we have that $P(v_i = a, v_j = b | Y(i) = 1) = P(v_i = a | Y(i) = 1) \cdot P(v_j = b | Y(i) = 1)$, and similarly for $v_i = -a, v_j = -b$ conditional on $Y(i) = -1$. Then
\begin{align}
&P(v_i = a, v_j = b, Y(i) = 1) \cdot P(Y = 1) = P(v_i = a, Y(i) = 1) \cdot P(v_j = b, Y(i) = 1) \nonumber \\
&P(v_i = -a, v_j = -b, Y(i) = -1) \cdot P(Y = -1) = P(v_i = -a, Y(i) = -1) \cdot P(v_j = -b, Y(i) = -1).
\label{eq:cond_ind}
\end{align} 

Note that terms in \eqref{eq:prop1IMS} can be split depending on if $Y(i)$ is $1$ or $-1$, so proving independence of $v_i Y(i)$ and $v_j Y(i)$ is equivalent to
\begin{align*}
&P(v_i = a, v_j = b, Y(i) = 1) + P(v_i = -a, v_j = -b, Y(i) = -1) \\
&= \left(P(v_i = a, Y(i) = 1) + P(v_i = -a, Y(i) = -1) \right) \cdot \left(P(v_j = b, Y(i) = 1) + P(v_j = -b, Y(i) = -1) \right).
\end{align*} 

We substitute \eqref{eq:cond_ind} into the right hand side. After rearranging, our equation to prove is
\begin{align*}
&P(v_i = a, v_j = b, Y(i) = 1) \cdot P(Y(i) = -1) + P(v_i = -a, v_j = -b, Y(i) = -1) \cdot P(Y(i) = 1) \\
&= P(v_i = -a, Y(i) = -1) \cdot P(v_j = b, Y(i) = 1) + P(v_i = a, Y(i) = 1) \cdot P(v_j = -b, Y(i) = -1).
\end{align*}

Due to symmetry of the terms above, it is thus sufficient to prove
\begin{align}
P(v_i = a, v_j = b, Y(i) = 1) \cdot P(Y(i) = -1) = P(v_i = -a, Y(i) = -1) \cdot P(v_j = b, Y(i) = 1).
\label{eq:IMS}
\end{align}

Let $N(v_i)$ be the set of $v_i$'s neighbors in $\bm{v}$, and $N(Y_i)$ be the set of $Y_i$'s neighbors in $\bm{Y}$. Let $\mathcal{S}$ be the event space for the hidden and observed variables, such that each element of the set $\mathcal{S}$ is a sequence of $+1$s and $-1$s of length equal to $|V|$. Denote $\mathcal{S}(v_i, v_j, Y(i))$ to be the event space for $V$ besides $v_i$, $v_j$, and $Y(i)$; we also have similar definitions used for $\mathcal{S}(Y(i)), \mathcal{S}(v_i, Y(i))$, $\mathcal{S}(v_j, Y(i))$.

Our approach is to write each probability in \eqref{eq:IMS} as a summation of joint probabilities over $\mathcal{S}(v_i, Y(i)), \mathcal{S}(v_j, Y(i)),$ and $\mathcal{S}(v_i, v_j, Y(i))$ using \eqref{eq:vising}. To do this more efficiently, we can factor each joint probability defined according to \eqref{eq:vising} into a product over \textit{isolated variables} and a product over \textit{non-isolated variables}. Recall that our marginal variables are $v_i$, $v_j$ and $Y(i)$. Define the set of non-isolated variables to be the marginal variables, plus all variables that interact directly with the marginal variables according to the potentials in the binary Ising model. Per this definition, the non-isolated variables are $V_{NI} = \{v_i, v_j, Y(i), Y(j), N(Y(i)), N(v_i), N(v_j), v_{Y(i)}\}$ where $v_{Y(i)} = \{v: Y(v) = Y(i) \}$ and the isolated variables are all other variables not in this set, $V_I = V \backslash V_{NI}$. We can thus factorize each probability into a term $\psi(\cdot)$ corresponding to factors of the binary Ising model that only have isolated variables and a term $\zeta(\cdot)$ coresponding to factors that have non-isolated variables.
\begin{align*}
P(v_i = a, v_j = b, Y(i) = 1) &= \frac{1}{Z} \sum_{s^{(a, b)} \in \mathcal{S}(v_i, v_j, Y(i))} \psi(s^{(a, b)}) \cdot \zeta(v_i = a, v_j = b, Y(i) = 1, s^{(a, b)}) \\
P(Y(i) = -1) &= \frac{1}{Z} \sum_{s^{(Y)} \in \mathcal{S}(Y(i))} \psi(s^{(Y)}) \cdot \zeta(Y(i) = -1, s^{(Y)}) \\
P(v_i = -a, Y(i) = -1) &= \frac{1}{Z} \sum_{s^{(a)} \in \mathcal{S}(v_i, Y(i))} \psi(s^{(a)}) \cdot \zeta(v_i = -a, Y(i) = -1, s^{(a)}) \\
P(v_j = b, Y(i) = 1) &= \frac{1}{Z} \sum_{s^{(b)} \in \mathcal{S}(v_j, Y(i))} \psi(s^{(b)}) \cdot \zeta(v_j  = b, Y(i) = 1, s^{(b)})
\end{align*}

To be precise, $\psi(\cdot)$ is
\begin{align*}
\psi(s^{(a, b)}) = \exp \Big(\sum_{\mathclap{\substack{Y_k \notin N(Y(i)) \\ \cup Y(i) \cup Y(j)}}} \theta_{Y_k} Y_k^{(a, b)} + \sum_{\mathclap{\substack{Y_k, Y_l \notin \\ N(Y(i)) \cup Y(i) \cup Y(j)}}} \theta_{Y_k, Y_l} Y_k^{(a, b)} Y_l^{(a, b)} + \sum_{\mathclap{\substack{Y(k) \notin N(Y(i)) \cup Y(i) \cup Y(j), \\ k \notin N(v_j) \cup v_j}}} \theta_k v_k^{(a, b)} Y(k)^{(a, b)} + \sum_{\mathclap{\substack{v_k, v_l \notin N(v_i) \cup v_i \\ \cup N(v_j) \cup v_j}}} \theta_{l, k} v_k^{(a, b)} v_l^{(a, b)} \Big),
\end{align*}

where $s^{(a, b)} = \{Y_1^{(a, b)}, \dots, Y_D^{(a, b)}, v_1^{(a, b)}, \dots \}$, and similar definitions hold for $s^{(a)}, s^{(b)}$, and $s^{(Y)}$. Then, \eqref{eq:IMS} is equivalent to showing
\begin{align*}
&\sum_{s^{(a, b)}, s^{(Y)}} \psi(s^{(a, b)}) \cdot \psi(s^{(Y)}) \cdot  \zeta(v_i = a, v_j = b, Y(i) = 1, s^{(a, b)}) \cdot \zeta(Y(i) = -1, s^{(Y)}) \\
= \; &\sum_{s^{(a)}, s^{(b)}} \psi(s^{(a)}) \cdot \psi(s^{(b)}) \cdot  \zeta(v_i = -a, Y(i) = -1, s^{(a)}) \cdot \zeta(v_j = b, Y(i) = 1, s^{(b)}).
\end{align*}

We can show this by finding values of $s^{(a)}$ and $s^{(b)}$ that correspond to each $s^{(a, b)}$ and $s^{(Y)}$. Note that the $\psi$ terms will cancel each other out if we directly set $s^{(a)}[V_I] = s^{(Y)}[V_I]$ and $s^{(b)}[V_I] = s^{(a, b)}[V_I]$.  Therefore, we want to set $s^{(a)}[V_{NI}]$ and $s^{(b)}[V_{NI}]$ such that the products of $\zeta$s are equivalent. We write them out explicitly first:
\begin{align*}
\zeta(v_i = a, v_j = b, Y(i) = 1, s^{(a, b)}) &= \exp \Big(\theta_{Y(i)} + \sum_{\mathclap{Y_k \in N(Y(i)) \cup Y(j)}} \theta_{Y_k} Y_k^{(a, b)}  + \sum_{\mathclap{Y_k \in N(Y(i))}} \theta_{Y_k, Y(i)} Y_k^{(a, b)} + \sum_{\mathclap{\substack{Y_k \in N(Y(i)) \cup Y(j), \\ Y_l \notin N(Y(i)) \cup Y(i) \cup Y(j)}}} \theta_{Y_k, Y_l} Y_k^{(a, b)} Y_l^{(a, b)} \\
&+ \theta_i a + \theta_j b Y(j)^{(a, b)} + \sum_{\mathclap{\substack{k \neq i, j, \\Y(k) = Y(i)}}} \theta_k v_k^{(a, b)} + \; \; \sum_{\mathclap{\substack{k \neq i, j, \\ Y(k) \in N(Y(i)) \cup Y(j) \\ | k \in N(v_j)}}} \theta_k v_k^{(a, b)} Y(k)^{(a, b)} + \sum_{\mathclap{v_k \in N(v_i)}} \theta_{i, k} a v_k^{(a, b)} \\
&+ \sum_{\mathclap{v_k \in N(v_j)}} \theta_{j, k} b v_k^{(a, b)}\Big)
\end{align*}
\begin{align*}
\zeta(Y(i) = -1, s^{(Y)}) &= \exp \Big(-\theta_{Y(i)} + \sum_{\mathclap{Y_k \in N(Y(i)) \cup Y(j)}} \theta_{Y_k} Y_k^{(Y)} - \sum_{\mathclap{Y_k \in N(Y(i))}} \theta_{Y_k, Y(i)} Y_k^{(Y)} + \sum_{\mathclap{\substack{Y_k \in N(Y(i)) \cup Y(j), \\ Y_l \notin N(Y(i)) \cup Y(i) \cup Y(j)}}} \theta_{Y_k, Y_l} Y_k^{(Y)} Y_l^{(Y)} \\
&- \theta_i v_i^{(Y)} + \theta_j v_j^{(Y)} Y(j)^{(Y)} - \sum_{\mathclap{\substack{k \neq i, j, \\Y(k) = Y(i)}}} \theta_k v_k^{(Y)} + \; \; \sum_{\mathclap{\substack{k \neq i, j, \\ Y(k) \in N(Y(i)) \cup Y(j) \\| k \in N(v_j)}}} \theta_k v_k^{(Y)} Y(k)^{(Y)} + \sum_{\mathclap{\substack{v_k \in N(v_i), \\ v_l \neq v_i}}} \theta_{k, l} v_k^{(Y)} v_l^{(Y)} \\
&+ \sum_{\mathclap{\substack{v_k \in N(v_j), \\ v_l \neq v_j}}} \theta_{k, l} v_k^{(Y)} v_l^{(Y)} \Big)
\end{align*}
\begin{align*}
\zeta(v_i = -a, Y(i) = -1, s^{(a)}) &= \exp \Big(-\theta_{Y(i)} + \sum_{\mathclap{Y_k \in N(Y(i)) \cup Y(j)}} \theta_{Y_k} Y_k^{(a)} - \sum_{\mathclap{Y_k \in N(Y(i))}} \theta_{Y_k, Y(i)} Y_k^{(a)} + \sum_{\mathclap{\substack{Y_k \in N(Y(i)) \cup Y(j), \\ Y_l \notin N(Y(i)) \cup Y(i) \cup Y(j)}}} \theta_{Y_k, Y_l} Y_k^{(a)} Y_l^{(a)} \\
&+ \theta_i a + \theta_j v_j^{(a)} Y(j)^{(a)} - \sum_{\mathclap{\substack{k \neq i, j, \\ Y(k) = Y(i)}}} \theta_k v_k^{(a)} + \;\; \sum_{\mathclap{\substack{k \neq i, j, \\ Y(k) \in N(Y(i)) \cup Y(j) \\| k \in N(v_j)}}} \theta v_k^{(a)} Y(k)^{(a)} + \sum_{\mathclap{\substack{v_k \in N(v_j), \\ v_l \neq v_j}}} \theta_{k, l} v_k^{(a)} v_l^{(a)} \\
&- \sum_{\mathclap{v_k \in N(v_i)}} \theta_{i, k} a v_k^{(a)} \Big)
\end{align*}
\begin{align*}
\zeta(v_j = b, Y(i) = 1, s^{(b)}) &= \exp \Big(\theta_{Y(i)} + \sum_{\mathclap{Y_k \in N(Y(i)) \cup Y(j)}} \theta_{Y_k} Y_k^{(b)} + \sum_{\mathclap{Y_k \in N(Y(i))}} \theta_{Y_k, Y(i)} Y_k^{(b)} + \sum_{\mathclap{\substack{Y_k \in N(Y(i)) \cup Y(j), \\ Y_l \notin N(Y(i)) \cup Y(i) \cup Y(j)}}} \theta_{Y_k, Y_l} Y_k^{(b)} Y_l^{(b)} \\
&+ \theta_i v_i^{(b)} + \theta_j b Y(j)^{(b)} + \sum_{\mathclap{\substack{k \neq i, j, \\ Y(k) = Y(i)}}} \theta_k v_k^{(b)} +  \; \; \sum_{\mathclap{\substack{k \neq i, j, \\ Y(k) \in N(Y(i)) \cup Y(j) \\| k \in N(v_j)}}} \theta_k v_k^{(b)} Y(k)^{(b)} + \sum_{\mathclap{\substack{v_k \in N(v_i), \\ v_l \neq v_i}}} \theta_{k, l} v_k^{(b)} v_l^{(b)} \\
&+ \sum_{\mathclap{v_k \in N(v_j)}} \theta_{j, k} b v_k^{(b)} \Big)
\end{align*}

We present a simple mapping from $s^{(a, b)}$ and $s^{(Y)}$ to $s^{(a)}$ and $s^{(b)}$ such that $\zeta(v_i = a, v_j = b, Y(i) = 1, s^{(a, b)}) \cdot \zeta(Y(i) = -1, s^{(Y)}) = \zeta(v_i = -a, Y(i) = -1, s^{(a)}) \cdot \zeta(v_j = b, Y(i) = 1, s^{(b)})$ holds:
\begin{center}
\begin{tabular}{l | c c}
 & $s^{(a)}$ & $s^{(b)}$ \\
 \hline %\cline{2-3}
$v_i$ & $-$ & $-v_i^{(Y)}$\\
$v_j$ & $v_j^{(Y)}$ &  $-$ \\ 
$Y_k \in N(Y(i)) \cup Y(j)$ & $Y_k^{(Y)}$ & $Y_k^{(a, b)}$ \\
$v_k \in N(v_i)$ & $-v_k^{(a, b)}$ & $-v_k^{(Y)}$ \\
$v_k \in N(v_j)$ & $v_k^{(Y)}$ & $v_k^{(a, b)}$ \\
$v_{Y(i)}$ & $-v_k^{(a, b)}$ & $-v_k^{(Y)}$
\end{tabular}
\end{center}

With this construction of $s^{(a)}$ and $s^{(b)}$, we have shown that $v_i Y(i)$ and $v_j Y(i)$ are independent. (In the case that $Y(j) = Y(i)$, the proof is almost exactly the same).

\subsubsection{Handling Larger Cliques}
\label{subsubsec:large_cliques}

We discuss how arbitrarily large cliques can be factorized into mean parameters and observable statistics to compute values of $a_C$ in Algorithm 2. %~\ref{alg:full}. 
This is due to the following general independence property that arises from construction of the Ising model in \eqref{eq:vising}:
\begin{proposition}
For a clique $C$ of $v_k$'s all connected to a single $Y(C)$, we have that $\prod_{k \in C} v_k  \independent Y(C)$ if $|C|$ is even, and $ \prod_{k \in C} v_k Y(C) \independent Y(C)$ if $|C|$ is odd. 
\label{prop:larger_cliques}
\end{proposition}

Therefore, if $|C|$ is even, then $a_C = \E{}{\prod_{k \in C} v_k} \cdot \E{}{Y(C)}$. If $|C|$ is odd, then $a_C = \E{}{\prod_{k \in C} v_k} / \E{}{Y(C)}$. 

\begin{proof}
We assume that there is only one hidden variable $Y$, although generalizing to the case where $D > 1$ is straightforward because our proposed independence property only acts on the hidden variable associated with a clique of observed variables.

We first prove the case where $|C|$ is even. We aim to show that for any $a, b \in \{-1, +1\}^2$,
\begin{align*}
P\Big(\prod_{k \in C}v_k = a, Y = b \Big) = P\Big(\prod_{k \in C} v_k = a \Big) P(Y = b).
\end{align*}

Using the concept of isolated variables and non-isolated variables earlier, the set of all observed variables $V_I$ besides those in $C$ and their neighbors can be ignored. Furthermore, suppose that $\mathcal{S}^{(C, a)}$ is the set of all $k \in C$ such that  $\prod_{k \in C} v_k = a$. For example, if $C = \{i, j\}$ and $a = -1$, $\mathcal{S}^{(C, -1)} = \{ (v_i, v_j) = (1, -1), (-1, 1) \}$. We write out each of the above probabilities as well as the partition function $Z$:
\begin{align*}
P\Big(\prod_{i \in C} v_i = a, Y = b \Big) &= \frac{1}{Z} \sum_{s^{(a, b)} \in \mathcal{S}(C, Y)} \psi\big(s^{(a, b)}\big) \sum_{s^{(C_1, a)} \in \mathcal{S}^{(C)}} \exp \Big(\theta_Y b +\sum_{i \in C} \theta_i b s_{v_i}^{(C_1)} + \sum_{i \notin C} \theta_i b v_i^{(a, b)} \\
&+ \sum_{(i, j) \in C} \theta_{i, j} s_{v_i}^{(C_1)} s_{v_j}^{(C_1)} + \sum_{i \in C} \sum_{j \in N(v_i) \backslash v_C} \theta_{i, j} s_{v_i}^{(C_1)} v_j^{(a, b)}\Big)
\end{align*}
\begin{align*}
P\Big(\prod_{i \in C} v_i = a\Big) &= \frac{1}{Z}  \sum_{s^{(a)}\in \mathcal{S}(C)} \psi\big(s^{(a)}\big) \sum_{s^{(C_2, a)} \in \mathcal{S}^{(C)}} \exp \Big(\theta_Y Y^{(a)} + \sum_{i \in C} \theta_i s_{v_i}^{(C_1)} Y^{(a)} + \sum_{i \notin C} \theta_i v_i^{(a)} Y^{(a)} \\
&+ \sum_{(i, j) \in C} \theta_{i, j} s_{v_i}^{(C_2)} s_{v_j}^{(C_2)} + \sum_{i \in C} \sum_{j\in N(v_i) \backslash v_C} \theta_{i,j} s_{v_i}^{(C_1)} v_j^{(a)} \Big)
\end{align*}
\begin{align*}
P(Y = b) &= \sum_{s^{(b)} \in \mathcal{S}(Y)} \psi\big(s^{(b)}\big) \exp \Big(\theta_Y b + \sum_{i \in C} \theta_i b v_i^{(b)} + \sum_{i \notin C} \theta_i v_i^{(b)} Y^{(b)} \\
&+ \sum_{(i, j) \in C} \theta_{i,j} v_i^{(b)} v_j^{(b)} + \sum_{i \in C} \sum_{j \in N(v_i) \backslash v_C} \theta_{i,j} v_i^{(b)} v_j^{(b)} \Big) \\
Z &= \sum_{s^{(z)} \in \mathcal{S}} \psi \big(s^{(z)}\big) \exp \Big(\theta_Y Y^{(z)} + \sum_{i\in C} \theta_i v_i^{(z)} Y^{(z)} + \sum_{i \notin C} \theta_i v_i^{(z)} Y^{(z)} + \sum_{(i, j) \in C} \theta_{i,j} v_i^{(z)} v_j^{(z)} \\
&+ \sum_{i \in C} \sum_{j \in N(v_i) \backslash v_C} \theta_{i,j} v_i^{(z)} v_j^{(z)} \Big)
\end{align*}

We want to show that we can map from each $s^{(a, b)}$, $s^{(z)}$ and $s^{(C_1)}$ to a respective $s^{(a)}, s^{(b)}$, and $s^{(C_2)}$. The $\psi(\cdot)$ terms can be ignored since we can just directly set $s^{(a)}[V_I] = s^{(a, b)}[V_I]$ and $s^{(b)}[V_I] = s^{(z)}[V_I]$. Using the above expressions for probabilities and the cumulant function, our desired statement to prove for each $s^{(a, b)}$, $s^{(z)}$ and $s^{(C_1)}$ is
\begin{align}
&\exp \Big(\theta_Y (b + Y^{(z)}) + \sum_{i \in C} \theta_i \big(b s_{v_i}^{(C_1)}  + v_i^{(z)} Y^{(z)}\big) + \sum_{i \notin C} \theta_i \big(b v_i^{(a, b)} + v_i^{(z)} Y^{(z)}\big) \nonumber \\
&+ \sum_{(i, j) \in C} \theta_{i,j}\big(s_{v_i}^{(C_1)} s_{v_j}^{(C_1)} + v_i^{(z)} v_j^{(z)}\big) + \sum_{i \in C} \sum_{j \in N(v_i) \backslash v_C} \theta_{i,j}\big(s_{v_i}^{(C_1)} v_k^{(a, b)} + v_i^{(z)} v_k^{(z)}\big) \Big) \nonumber \\
= \;& \exp\Big(\theta_Y \big(b + Y^{(a)}\big) + \sum_{i \in C} \theta_i \big(s_{v_i}^{(C_2)} Y^{(a)} + b v_i^{(b)}\big) + \sum_{i\notin C} \theta_i \big(v_i^{(a)} Y^{(a)} + b v_i^{(b)}\big) \nonumber \\
&+ \sum_{(i, j) \in C} \theta_{i,j} \big(s_{v_i}^{(C_2)} s_{v_j}^{(C_2)} + v_i^{(b)} v_j^{(b)}\big) + \sum_{i \in C}\sum_{j \in N(v_i) \backslash v_C}\theta_{i,j} \big(s_{v_i}^{(C_2)} v_j^{(a)} + v_i^{(b)} v_j^{(b)}\big) \Big)
\label{eq:parity_proof}
\end{align}

We can ensure that the above expression is satisfied with the following relationship between $s^{(a, b)}$, $s^{(z)}, s^{(C_1)}$ and $s^{(a)}, s^{(b)}, s^{(C_2)}$. If $Y^{(z)} = b$, then we set $Y^{(a)} = b$, $s_{v_i}^{(C_2)} = s_{v_i}^{(C_1)}$ for $i \in C$, and $v_i^{(b)} = v_i^{(z)}, v_i^{(a)} = v_i^{(a, b)}$ for all $v_i$. If $Y^{(z)} = -b$, then we set $Y^{(a)} = -b$, $s_{v_i}^{(C_2)} = -s_{v_i}^{(C_1)}$ for $i \in C$, and $v_i^{(b)} = - v_i^{(z)}, v_i^{(a)} = - v_i^{(a, b)}$ for all $v_i$. However, note that setting either all $s_{v_i}^{(C_2)}$ to be $s_{v_i}^{(C_1)}$ or $-s_{v_i}^{(C_1)}$ means that both $s^{(C_1)}$ and $-s^{(C_1)}$ are in $\mathcal{S}^{(C)}$. This is only true when $|C|$ is even because $\prod_{i \in C} (-v_i) = (-1)^{|C|} \prod_{i \in C} v_i = (-1)^{|C|} a$. 

Our proof approach is similar when $|C|$ is odd. We aim to show that for any $a, b \in \{-1, +1\}^2$,
\begin{align*}
P\Big(\prod_{k \in C} v_k Y = a, Y = b\Big) = P\Big(\prod_{k \in C} v_k Y = a \Big) P(Y = b).
\end{align*}

$P(\prod_{k \in C} v_k Y = a, Y = b)$ can be written as $P(\prod_{k \in C} v_k = \frac{a}{b}, Y = b)$, which follows the same format of the probability we used for the case where $|C|$ is even. We will end up with a desired equation to prove that is identical to \eqref{eq:parity_proof}, except that we must modify $s^{(C_1)}$ and $s^{(C_2)}$. $s^{(C_1)}$ is now from the set $\mathcal{S}^{(C, a/b)}$, and $s^{(C_2)}$ is from the set $\mathcal{S}^{(C, a/b)}$ when $Y^{(a)} = b$ and from the set $s^{(C, -a/b)}$ when $Y^{(a)} = -b$. We can set $s^{(a)}$, $s^{(b)}$, and $s^{(C_2)}$ the exact same way as before; in particular, $s_{v_i}^{(C_2)} = s_{v_i}^{(C_1)}$ when $Y^{(a)} = b$ and $s_{v_i}^{(C_2)} = -s_{v_i}^{(C_1)}$ when $Y^{(a)} = -b$. Both $s_{v_i}^{(C_1)}, Y^{(a)} = b$ and $-s_{v_i}^{(C_1)}, Y^{(a)} = -b$ satisfy $\prod_{i \in C} v_i Y = a$, since $\prod_{i \in C} (-v_i) (-Y) = (-1)^{|C| + 1} \prod_{i \in C} v_i Y = a$ when $|C|$ is odd. 
\end{proof}

\subsubsection{Augmenting the dependency graph} 
\label{subsubsec:augment}

We define the graphical model particular to how $G_{dep}$ is augmented, which gives way to a concise mapping between each $a_C$ and $a_{C_{dep}}$.

In the case where no sources can abstain at all, $\lf_i$ takes on values $\{\pm 1\}$ and thus the augmentation is not necessary. We have that $G = G_{dep}$, $\bm{v} = \bm{\lf}$, and the graphical model's joint distribution \eqref{eq:vising} reduces to 
\begin{align}
f_G &(Y, \bm{\lambda}) =  \frac{1}{Z} \exp \Big(\sum_{k = 1}^D \theta_{Y_k} Y_k + \sum_{(Y_k, Y_l) \in E} \theta_{Y_k, Y_l} Y_k Y_l + \sum_{i = 1}^m \theta_{i} \lf_i Y(i) + \sum_{(\lf_i, \lf_j) \in E} \theta_{i, j} \lf_i \lf_j \Big).
\label{eq:exp_fam}
\end{align}

All of Algorithm 2 %~\ref{alg:full} 
will be done on $\{\bm{Y}, \bm{\lf}\}$. While the triplet method is still used for recovering mean parameters, the mapping from $a_C$ to $a_{C_{dep}}$ is trivial, and the linear transformation back to $\mu_{C_{dep}}$ will have terms containing $\lf_i = 0$ reduced to $0$.

In the case where sources abstain, we have discussed how to generate $\bm{v}$ from $\bm{\lf}$ and $G$ from $G_{dep}$, of which an example is shown in Figure \ref{fig:abstains}. Most importantly, we suppose that when $\lf_i = 0$, we set $(v_{2i - 1}, v_{2i})$ to either $(1, 1)$ or $(-1, -1)$ with equal probability such that
\begin{align}
P\big((v_{2i - 1}, v_{2i}) = (1, 1), V \backslash \{v_{2i - 1}, v_{2i} \}\big) = P\big((v_{2i - 1}, v_{2i}) = (-1, -1), V \backslash \{v_{2i - 1}, v_{2i} \}\big) = \frac{1}{2} P(\lf_i = 0, V \backslash \{v_{2i - 1}, v_{2i}\}).
\label{eq:abstains_split}
\end{align}
\tikzstyle{place4}=[rectangle,draw=black!100,dashed,fill=white!100,thick, inner sep=12.5pt,fill opacity=0.2]

\definecolor{tableaured}{RGB}{225,87,89}
\definecolor{tableauyellow}{RGB}{237,201,73}
\definecolor{tableaugreen}{RGB}{90,161,80}
\definecolor{tableaupurple}{RGB}{176,123,161}
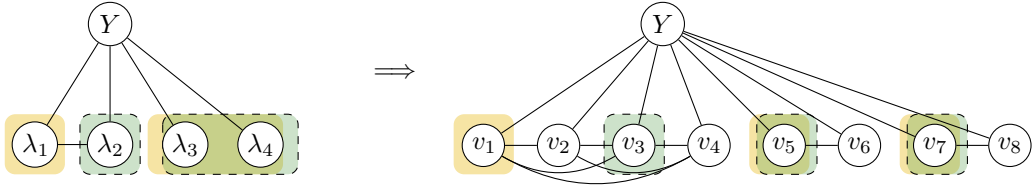
\begin{figure*}[t]
\centering
\begin{tikzpicture}[-,>=stealth',level/.style={sibling distance = 1cm/#1,
  level distance = 1.5cm}, baseline=(current bounding box.north)] 
\draw [fill=tableauyellow, draw opacity=0.0, fill opacity=0.5, thick, rounded corners] (-1.40, -1.2) rectangle (-0.60, -2.0);
\draw [fill=tableauyellow, draw opacity=0.0, fill opacity=0.5, thick, rounded corners] (0.5, -1.2) rectangle (2.3, -2.0);
%\draw [fill=tableaugreen, draw opacity=1.0, fill opacity=0.3, thin, rounded corners, dashed] (.55, -1.2) rectangle (2.45, -2.0);
%\draw [fill=tableaugreen, draw opacity=1.0, fill opacity=0.3, thin, rounded corners, dashed] (4.05, -1.2) rectangle (3.15, -2.0);
\draw [fill=tableaugreen, draw opacity=1.0, fill opacity=0.3, thin, rounded corners, dashed] (0.7, -1.2) rectangle (2.5, -2.0);
\draw [fill=tableaugreen, draw opacity=1.0, fill opacity=0.3, thin, rounded corners, dashed] (-0.4, -1.2) rectangle (0.4, -2.0);

\node [arn_n] (Y) {$Y$} ;
\node [arn_n] [below =1cm of Y] (L2){$\lf_2$}   ;
\node [arn_n] [left of=L2] (L1){$\lf_1$}   ;
\node [arn_n] [ right of=L2] (L3){$\lf_3$}   ;
\node [arn_n] [ right of=L3] (L4){$\lf_4$}   ;
%\node [arn_n] [ right=1cm of L4] (Lm) {$\lf_m$}   ;
%\node [] [ right=0.2cm of L4] {$...$}   ;
\path (Y) edge            (L1);
\path (L2) edge				(L1);
\path (Y) edge            (L2);
\path (Y) edge            (L3);
\path (Y) edge			  (L4);
%\node[place4](avg) at (0.0,-2.0) {}; %{\quad \quad\quad\quad\quad Averaging \quad\quad\quad\quad\quad};
\end{tikzpicture}
\hspace{2em}
\begin{tikzpicture}[-,>=stealth',level/.style={sibling distance = 1cm/#1,
  level distance = 1.5cm}, baseline=(current bounding box.north)] 
\draw [fill=tableauyellow, draw opacity=0.0, fill opacity=0.5, thick, rounded corners] (-2.78, -1.2) rectangle (-1.98, -2.0);
\draw [fill=tableaugreen, draw opacity=1.0, fill opacity=0.3, thin, rounded corners, dashed] (-0.78, -1.2) rectangle (0.02, -2.0);
\draw [fill=tableauyellow, draw opacity=0.0, fill opacity=0.5, thick, rounded corners] (1.15, -1.2) rectangle (1.95, -2.0);
\draw [fill=tableaugreen, draw opacity=1.0, fill opacity=0.3, thin, rounded corners, dashed] (1.25, -1.2) rectangle (2.05, -2.0);
\draw [fill=tableauyellow, draw opacity=0.0, fill opacity=0.5, thick, rounded corners] (3.15, -1.2) rectangle (3.95, -2.0);
\draw [fill=tableaugreen, draw opacity=1.0, fill opacity=0.3, thin, rounded corners, dashed] (3.25, -1.2) rectangle (4.05, -2.0);

%\draw [fill=tableaugreen, draw opacity=1.0, fill opacity=0.3, thin, rounded corners, dashed] (4.05, -1.2) rectangle (3.15, -2.0);
\node [arn_n] (Y) {$Y$} ;
\node [arn_n] [below right=1.7cm of Y] (v3^1){$v_5$}   ;
\node [arn_n] [right of=v3^1] (v3^-1){$v_6$}   ;
\node [arn_n] [right of=v3^-1] (v4^1){$v_7$}   ;
\node [arn_n] [right of=v4^1] (v4^-1){$v_8$}   ;
\node [arn_n] [left of=v3^1] (v2^-1){$v_4$}   ;
\node [arn_n] [left of=v2^-1] (v2^1){$v_3$}   ;
\node [arn_n] [left of=v2^1] (v1^-1) {$v_2$}   ;
\node [arn_n] [left of=v1^-1] (v1^1) {$v_1$}   ;
\node []      [above left=0.8cm of v1^1] {$\Longrightarrow$};
\path (Y) edge            (v1^1);
\path (Y) edge            (v1^-1);
\path (Y) edge            (v2^1);
\path (Y) edge            (v2^-1);
\path (Y) edge            (v3^1);
\path (Y) edge            (v3^-1);
\path (v3^1) edge			(v3^-1);
\path (v1^1) edge			(v1^-1);
\path (v2^1) edge			(v2^-1);
\path (v1^-1) edge			(v2^1);
\path (Y) edge				(v4^1);
\path (Y) edge				(v4^-1);
\path (v4^1) edge			(v4^-1);
\path (v1^1) edge[bend right=30] (v2^1); 
\path (v1^1) edge[bend right=30] (v2^-1); 
\path (v1^-1) edge[bend right=30] (v2^-1); 
%\node[place4](avg) at (0.0,-2.0) {}; %{\quad \quad\quad\quad\quad Averaging \quad\quad\quad\quad\quad};
\end{tikzpicture}
\caption{Example of mapping from $G_{dep}$ to $G$. Left: $G_{dep}$, where boxes indicate valid triplet groupings of sources. Right: $G$, where boxes indicate the triplets of observed variables that are sufficient to recover all mean parameters.}
\label{fig:abstains}
\end{figure*}

The joint distribution over $\{\bm{Y}$, $\bm{v}\}$ follows from \eqref{eq:vising}:
\begin{align}
f_{G}(\bm{Y}, \bm{v}) = \frac{1}{Z} \exp \bigg(&\sum_{k = 1}^D \theta_{Y_k} Y_k + \; \;  \sum_{\mathclap{(Y_k, Y_l) \in E}} \; \; \theta_{Y_k, Y_l} Y_k Y_l + \sum_{i = 1}^m \theta_i \begin{bmatrix} 1 & -1  \end{bmatrix} \begin{bmatrix} v_{2i - 1} \\ v_{2i} \end{bmatrix} Y^{dep}(i) \nonumber \\
+ &\sum_{i = 1}^m \theta_{i,i} v_{2i - 1} v_{2i} + \; \; \sum_{\mathclap{i, j: (\lf_i, \lf_j )\in E_{dep}}} \; \; \theta_{i,j} \begin{bmatrix} v_{2i - 1} & v_{2i} \end{bmatrix} \begin{bmatrix} 1 & -1 \\ -1 & 1 \end{bmatrix} \begin{bmatrix} v_{2j - 1} \\ v_{2j} \end{bmatrix} \bigg),
\label{eq:abstain}
\end{align}

where $E_{dep}$ is $G_{dep}$'s edge set. Note that this graphical model has the same absolute values of the canonical parameters for both $v_{2i-1} Y^{dep}(i)$ and for all four terms $(v_{2i - 1}, v_{2i}) \times (v_{2j - 1}, v_{2j})$ due to the balancing in \eqref{eq:abstains_split}. As a result, the mean parameters also exhibit the same symmetry, which we show in the following lemma.

\begin{lemma} For each $\lf_i$, we have that $\E{}{\lf_i Y^{dep}(i)} = \E{}{v_{2i - 1} Y^{dep}(i)} = - \E{}{v_{2i} Y^{dep}(i)}$.
\end{lemma}
\begin{proof}
First, we can write out $\E{}{\lf_i Y^{dep}(i)}$ as
\begin{align*}
\E{}{\lf_i Y^{dep}(i)} &= P(\lf_i Y^{dep}(i) = 1) - P(\lf_i Y^{dep}(i) = -1) = P(\lf_i Y^{dep}(i) = 1) \\
&- (1 - P(\lf_i Y^{dep}(i) = 1) - P(\lf_i Y^{dep}(i) = 0)) \\
&= 2P(\lf_i Y^{dep}(i) = 1) + P(\lf_i = 0) - 1.
\end{align*}

We know that if we have $v_{2i - 1} = 1$ or $v_{2i} = -1$, then $\lf_i$ is either $1$ or $0$, but never $-1$; similarly, $v_{2i - 1} = -1$ and $v_{2i} = 1$ imply that $\lf_i \neq 1$. We write out $\E{}{v_{2i - 1} Y^{dep}(i)}$:
\begin{align*}
\E{}{v_{2i - 1} Y^{dep}(i)} &= 2 \left(P(v_{2i - 1} = 1, Y^{dep}(i) = 1) + P(v_{2i - 1} = -1, Y^{dep}(i) = -1) \right) - 1 \\
&= 2 \big(P((v_{2i - 1}, v_{2i}) = (1, 1), Y^{dep}(i) = 1) + P(\lf_i = 1, Y^{dep}(i) = 1) \\
&+ P(\lf_i = -1, Y^{dep}(i) = -1) + P((v_{2i - 1}, v_{2i}) = (-1, -1), Y^{dep}(i) = -1) \big) - 1 \\
&= 2 \Big(P(\lf_i Y^{dep}(i) = 1) + \frac{1}{2}P(\lf_i = 0, Y^{dep}(i) = 1) + \frac{1}{2}P(\lf_i = 0, Y^{dep}(i) = -1) \Big) - 1 \\
&= 2 P(\lf_i Y^{dep}(i) = 1) + P(\lf_i = 0) - 1 = \E{}{\lf_i Y^{dep}(i)}.
\end{align*}

Similarly, $\E{}{v_{2i} Y^{dep}(i)}$ is
\begin{align*}
\E{}{v_{2i} Y^{dep}(i)} &= 2 \big(P((v_{2i - 1}, v_{2i}) = (1, 1), Y^{dep}(i) = 1) + P(\lf_i = -1, Y^{dep}(i) = 1) \\
&+ P(\lf_i = 1, Y^{dep}(i) = -1) + P((v_{2i - 1}, v_{2i}) = (-1, -1), Y^{dep}(i) = -1) \big) - 1 \\
&= 2 \Big(P(\lf_i Y^{dep}(i) = -1) + \frac{1}{2}P(\lf_i = 0, Y^{dep}(i) = 1) + \frac{1}{2}P(\lf_i = 0, Y^{dep}(i) = -1)\Big) - 1 \\
&= 2P(\lf_i Y^{dep}(i) = -1) + P(\lf_i = 0) - 1 \\
&= P(\lf_i Y^{dep}(i) = -1) - (1 - P(\lf_i = 0) - P(\lf_i Y^{dep}(i) = -1)) \\
&= P(\lf_i Y^{dep}(i) = -1) - P(\lf_i Y^{dep}(i) = 1) = - \E{}{\lf_i Y^{dep}(i)}.
\end{align*}
\end{proof}

The triplets in Algorithm 1 %~\ref{alg:triplet} 
thus only need to be computed over exactly half of $\bm{v}$, each corresponding to one source, as shown in Figure \ref{fig:abstains}. Moreover, this augmentation method for $\bm{v}$ and $G$ allows us to conclude for any clique of sources $C_{dep} \in \mathcal{C}_{dep}$,
\begin{align*}
\mathbb{E} \bigg[\prod_{k \in C_{dep}} v_{2k - 1} Y^{dep}(C_{dep})\bigg] = \mathbb{E} \bigg[\prod_{k \in C_{dep}} \lf_k Y^{dep}(C_{dep})\bigg].
\end{align*}

In general, the expectation over a clique in $G_{dep}$ containing $\{\lf_i\}_{i \in C_{dep}}$ is equal to the expectation over the corresponding clique $C$ in $G$ containing $\{v_{2i - 1}\}_{i \in C_{dep}}$ such that $a_C = a_{C_{dep}}$.

\subsubsection{Linear Transformation to Label Model Parameters}
\label{subsubsec:prod_to_joint}
To convert these $a_{C_{dep}}$ into $\mu_{C_{dep}}$, we present a way to linearly map from these product probabilities and expectations back to marginal distributions, focusing on the unobservable distributions over a clique of sources and a task that the sources vote on. We first restate our example stated in Section 3.2.  %\ref{subsec:model_mapping}. 
Define $\mu_i(a, b) = P(Y^{dep}(i) = a, \lf_i = b)$ for $a \in \{-1, 1\}$ and $b \in \{-1, 0, 1\}$. We can set up a series of linear equations and denote it as $A_1 \mu_i = r_i$:
\begin{align}
\begin{bmatrix}
1 & 1 & 1 & 1 & 1 & 1 \\ 
1 & 0 & 1 & 0 & 1 & 0 \\
1 & 1 & 0 & 0 & 0 & 0 \\
1 & 0 & 0 & 0 & 0 & 1 \\
0 & 0 & 1 & 1 & 0 & 0 \\
0 & 0 & 1 & 0 & 0 & 0
\end{bmatrix} \begin{bmatrix}
\mu_i(1, 1) \\ \mu_i(-1, 1) \\ \mu_i(1, 0) \\ \mu_i(-1, 0) \\ \mu_i(1, -1) \\ \mu_i(-1, -1)
\end{bmatrix} = \begin{bmatrix}
1 \\ P(Y^{dep}(i) = 1) \\ P(\lambda_i = 1) \\ P(\lf_i Y^{dep}(i) = 1) \\ P(\lambda_i = 0) \\ P(\lambda_i = 0, Y^{dep}(i) = 1)
\end{bmatrix}.
\label{eq:prod_to_joint}
\end{align}

Note that four entries on the right of the equation are observable or known. $P(\lf_i Y^{dep}(i) = 1)$ can be written in terms of $a_i$, and by construction of $(v_{2i - 1}, v_{2i})$ and \eqref{eq:abstains_split}, we can factorize $P(\lf_i = 0, Y^{dep}(i) = 1)$ into observable terms:
\begin{align*}
P(\lf_i = 0, Y^{dep}(i) = 1) &= P((v_{2i - 1}, v_{2i}) = (1, 1), Y^{dep}(i) = 1) + P((v_{2i - 1}, v_{2i}) = (-1, -1), Y^{dep}(i) = 1) \\
&= (P((v_{2i - 1}, v_{2i}) = (1, 1)) + P((v_{2i - 1}, v_{2i}) = (-1, -1))) P(Y^{dep}(i) = 1) \\
&= P(\lambda_i = 0) P(Y^{dep}(i) = 1).
\end{align*}

Here we use the fact that $v_{2i - 1} v_{2i}$ and $Y^{dep}(i)$ are independent by Proposition \ref{prop:larger_cliques}. We can verify that $A_1$ is invertible, so $\mu_i(a, b)$ can be obtained from this system.

There is a way to extend this system to the general case. We form a system of linear equations $A_s \mu_C = r_C$ for each clique  of sources $C$ in $G_{dep}$, where $s = |C|$ is the number of weak sources $\lf_i$ in the clique and $\mu_C$ is the marginal distribution over these $s$ sources and $1$ task. $A_s$ is a $2 (3^s) \times 2 (3^s)$ matrix of $0s$ and $1$s that will help map from $r_C$, a vector of probabilities known from prior steps of the algorithms or from direct estimation, to the desired label model parameter $\mu_C$. Define
\begin{align*}
A_0 = \begin{bmatrix}
1 & 1 \\
1 & 0
\end{bmatrix} \qquad B_0 = \begin{bmatrix}
0 & 0 \\
0 & 1
\end{bmatrix}
\end{align*}
\begin{align*}
D = \begin{bmatrix}
1 & 1 & 1 \\
1 & 0 & 0 \\
0 & 1 & 0
\end{bmatrix} \qquad E = \begin{bmatrix}
0 & 0 & 0 \\
0 & 0 & 1 \\
0 & 0 & 0
\end{bmatrix}
\end{align*}

Then $A_s$ and $B_s$ can be recursively constructed with
\begin{align*}
A_s &= D \otimes A_{s - 1} + E \otimes B_{s - 1} \\
B_s &= E \otimes A_{s - 1} + D \otimes B_{s - 1}, 
\end{align*}
where $\otimes$ is the Kronecker product. To define $r_C$, we first specify an ordering of elements of $\mu_C$. Let the last $\lf_{C_s}$ in the joint probability $\mu_C$ take on value $\lf_{C_s} = 1$ for the first $2 \times 3^{s - 1}$ entries, $\lf_{C_s} = 0$ for the next $2 \times 3^{s - 1}$ entries, and $\lf_{C_s} = -1$ for the last $2 \times 3^{s - 1}$ entries. In general, the $i$th $\lf_{C_i}$ in $\mu_C$ will alternate among $1, 0, -1$ every $2 \times 3^{i - 1}$ entries. Finally, the $Y(i)$ entry of $\mu_C$ alternates every other value between $1$ and $-1$. 

The ordering of $r_C$ follows a similar structure. If we rename the $Y$ and $\lf$ variables to $z_1, \dots, z_{s + 1}$ for generality, each entry $r_C(U, Z)$ is equal to $P(\prod_{z_i \in Z} z_i = 1, z_j = 0 \; \forall z_j \in U)$, where $U \cap Z = \emptyset$, and $U \subseteq C \backslash Y(i)$, $Z \subseteq C$. We also write $r_C(\emptyset, \emptyset) = 1$. The entries of $r_C$ will alternate similarly to $\mu_C$, for each $\lf_{C_i}$, the first $2 \times 3^{i - 1}$ terms will not contain $\lf_{C_i}$ in either $U$ or $Z$, the second $2 \times 3^{i-1}$ terms will have $\lf_{C_i} \in Z$, and the last $2 \times 3^{i - 1}$ terms will have $\lf_{C_i} \in U$. For $Y(i)$, elements of $r_C$ will alternate every other value between not having $Y(i)$ in $Z$ and having $Y(i)$ in $Z$. \eqref{eq:prod_to_joint} illustrates an example of the orderings for $\mu_C$ and $r_C$. 

Furthermore, we also have the system $B_s \mu_C = r^B_C$, where $r^B_C(U, Z) = P(\prod_{z_i \in Z} z_i = -1, z_j = 0 \; \forall z_j \in U)$ when $Z \neq \emptyset$, and $r_B^C(U, \emptyset) = 0$. The ordering of $r^B_C$ is the same as that of $r_C$.

\begin{lemma}
With the setup above, $A_s \mu_C = r_C$.
\end{lemma}

\begin{proof}
We prove that $A_s \mu_C = r_C$ and $B_s \mu_C = r^B_C$ by induction on $s$. For the base case $s = 0$, we examine a clique over just a single $Y$:
\begin{align*}
\begin{bmatrix}
1 & 1 \\ 1 & 0
\end{bmatrix} \begin{bmatrix}
P(Y = 1) \\ P(Y = -1)
\end{bmatrix} = \begin{bmatrix}
1 \\ P(Y = 1)
\end{bmatrix} \qquad \begin{bmatrix}
0 & 0 \\ 0 & 1
\end{bmatrix} \begin{bmatrix}
P(Y = 1) \\ P(Y = -1)
\end{bmatrix} = \begin{bmatrix}
0 \\ P(Y = -1)
\end{bmatrix},
\end{align*}
which are both clearly true. Next, we assume that $A_k \mu_C = r_C$ and $B_k \mu_C = r^B_C$ for $s = k$. We want to show that $A_{k + 1} \mu_{C'} = r_{C'}$ and $B_{k + 1} \mu_{C'} = r^B_{C'}$ for a larger clique $C'$ where $C \subset C'$ and $|C'| = s + 1$. By construction of $A_{k + 1}$ and $B_{k + 1}$,
\begin{align*}
A_{k + 1} = \begin{bmatrix}
A_k & A_k & A_k \\ A_k & 0 & B_k \\ 0 & A_k & 0
\end{bmatrix} \qquad
B_{k + 1} = \begin{bmatrix}
B_k & B_k & B_k \\
B_k & 0 & A_k \\
0 & B_k & 0
\end{bmatrix}.
\end{align*}

$\mu_{C'}$, $r_{C'}$, and $r^B_{C'}$ can be written as 
\begin{align*}
\mu_{C'} &= \begin{bmatrix}\mu_C(\lf_{C'_{k + 1}} = 1)) \\ \mu_C(\lf_{C'_{k + 1}} = 0) \\ \mu_C(\lf_{C'_{k + 1}} = -1)\end{bmatrix} \qquad r_{C'} = \begin{bmatrix}
r_C \\ r_C(\lf_{C'_{k + 1}} \in Z') \\ r_C(\lf_{C'_{k + 1}} \in U') 
\end{bmatrix} \qquad r^B_{C'} = \begin{bmatrix}
r^B_C \\ r^B_C(\lf_{C'_{k + 1}} \in Z') \\ r^B_C(\lf_{C'_{k + 1}} \in U') 
\end{bmatrix},
\end{align*}

where $\mu_C(\lf_{C'_{k+ 1}} = 1) = P(Y, \lf_{C_1}, \dots, \lf_{C_k}, \lf_{C'_{k + 1}}= 1)$, $r_C(\lf_{C'_{k + 1}} \in Z') = r_C(U, Z \cup \{\lf_{C'_{k+ 1}} \})$, and so on. $U', Z'$ for $C'$ are constructed similarly to $U, Z$ for $C$.

Then the three equations for $A_k$ we want to show are
\begin{align*}
&A_k (\mu_C(\lf_{C'_{k + 1}} = 1) + \mu_C(\lf_{C'_{k + 1}} = 0) + \mu_C(\lf_{C'_{k + 1}} = -1)) = r_C \\
&A_k( \mu_C(\lf_{C'_{k + 1}} = 1)) + B_k (\mu_C(\lf_{C'_{k + 1}} = -1)) = r_C(\lf_{C'_{k + 1}} \in Z') \\
&A_k (\mu_C(\lf_{C'_{k + 1}} = 0)) = r_C(\lf_{C'_{k + 1}} \in U').
\end{align*}
The first equation is true because $\lf_{C'_{k+1}}$ is marginalized out to yield $A_k \mu_C = r_C$, which is true by our inductive hypothesis. In the third equation, the term $\lf_{C'_{k + 1}} = 0$ is added as a joint probability to all probabilities in $\mu_C$ and $r_C$, so this also holds by the inductive hypothesis. In the second equation, $A_k(\mu_C(\lf_{C'_{k+1}} = 1))$ is equal to $r_C$ with each probability having $\lf_{C'_{k+1}} = 1$ as an additional joint probability, and similarly $B_k (\mu_C(\lf_{C'_{k + 1}} = -1))$ is equal to $r^B_C$ with each nonzero probability having $\lf_{C'_{k+1}} = -1$ as an additional joint probability. For entries where $Z \neq \emptyset$, summing these up yields 
\begin{align*}
&P\Big(\prod_{z_i \in Z} z_i = 1, \lf_{C'_{k+1}} = 1, z_j = 0 \; \forall z_j \in U\Big) + P\Big(\prod_{z_i \in Z} z_i = -1, \lf_{C'_{k+1}} = -1, z_j = 0 \; \forall z_j \in U\Big) \\
=  &P\Big(\prod_{z_i \in Z} z_i \lf_{C'_{k + 1}}= 1, z_j = 0 \; \forall z_j \in U\Big). 
\end{align*}

And when $Z = \emptyset$, we have $P(\lf_{C'_{k+1}} = 1, z_j = 0 \; \forall z_j \in U)$, so all together these probabilities make up $r_C(\lf_{C'_{k+1}} \in Z')$.

The three equations for $B_k$ are similar:
\begin{align*}
&B_k (\mu_C(\lf_{C'_{k + 1}} = 1) + \mu_C(\lf_{C'_{k + 1}} = 0) + \mu_C(\lf_{C'_{k + 1}} = -1)) = r^B_C \\
&B_k( \mu_C(\lf_{C'_{k + 1}} = 1)) + A_k (\mu_C(\lf_{C'_{k + 1}} = -1)) = r^B_C(\lf_{C'_{k + 1}} \in Z') \\
&B_k (\mu_C(\lf_{C'_{k + 1}} = 0)) = r^B_C(\lf_{C'_{k + 1}} \in U').
\end{align*}
Again, the first and third equations are clearly true using the inductive hypothesis, and the second equation is also true when we decompose $\prod_{z_i \in Z'} z_i = -1$ into $\prod_{z_i \in Z} z_i = 1, \lf_{C'_{k + 1}} = -1$ and $\prod_{z_i \in Z} z_i = -1, \lf_{C'_{k + 1}} = 1$.

We complete this proof by induction to conclude that $A_s \mu_C = r_C$ and $B_s \mu_C = r^B_C$, showing a recursive approach for mapping from $r_C$ to $\mu_C$ for any clique or separator set $C$.
\end{proof}

Finally, we note that each $r_C$ is made up of computable terms. Entries of the form $r_C(\emptyset, Z) = P(\prod_{z_i \in Z} z_i = 1)$ are immediately calculated from $a_c$ for cliqes $c \subseteq C$, and entries where $Y(i) \notin Z$ can be directly estimated. Entries where $Z = \{Y(i)\}, U \neq \emptyset$ can be factorized into known or directly estimated probabilities, and all other entries can be computed by calculating each $a_c$ conditional on $U$.

As an example, to construct $r_{ij}$ for a clique $\{\lf_i, \lf_j, Y^{dep}(i, j)\}$, the only entries of $r_{ij}$ that are unobservable from the data are $P(\lambda_i Y^{dep}(i, j) = 1), \; P(\lambda_j Y^{dep}(i, j) = 1)$, $\; P(\lambda_i \lambda_j Y^{dep}(i, j) = 1), \;P(\lambda_i = 0, Y^{dep}(i, j) = 1), \;P(\lambda_j = 0, Y^{dep}(i, j) = 1),\; P(\lambda_i =0, \lambda_j Y^{dep}(i, j) = 1), \;P(\lambda_j = 0, \lambda_i Y^{dep}(i, j) = 1), $ and $P(\lambda_i = 0, \lambda_j = 0, Y^{dep}(i, j) = 1)$. We have discussed how to estimate all but the last three. 

To estimate $P(\lambda_i = 0, \lambda_j Y^{dep}(i, j) = 1)$, we can write this as
\begin{align*}
P(\lf_j Y^{dep}(i, j) = 1, \lf_i = 0) &= P(\lambda_j Y^{dep}(i, j) = 1 | \lambda_i = 0) P(\lambda_i = 0) \\
&= \frac{1 + \E{}{\lambda_j Y^{dep}(i, j) | \lambda_i = 0} - P(\lambda_j = 0 | \lf_i = 0)}{2} \cdot P(\lambda_i = 0) \\
&= \frac{1}{2}P(\lf_i = 0) +  \frac{1}{2}\E{}{\lambda_j Y^{dep}(i, j) | \lambda_i = 0} P(\lf_i = 0) + \frac{1}{2} P(\lf_j = 0, \lf_i = 0).
\end{align*}

We can solve $\E{}{\lambda_j Y^{dep}(i, j) | \lambda_i = 0}$ using the triplet method conditional on samples where $\lf_i$ abstains. $P(\lambda_i = 0, \lambda_j = 0, Y^{dep}(i, j) = 1)$ can be written as $P(\lambda_i = 0, \lambda_j = 0) P(Y^{dep}(i, j) = 1)$, of which all probabilities are observable, by Proposition \ref{prop:larger_cliques}.

\subsubsection{\textsc{ResolveSigns}} 
\label{subsubsec:resolvesigns}
This function is used to determine the signs after we have recovered the magnitudes of accuracy terms such as $|\mathbb{E}[v_iY(i)]|$. One way to implement this function is to use one known accuracy sign per $Y$. We observe that if we know the sign of $a_i = \mathbb{E}[v_i Y(i)]$, then we are able to obtain the sign of any other term  $a_j = \mathbb{E}[v_j Y(j)]$ where $Y(j) = Y(i)$. If $v_i$ and $v_j$ are conditionally independent given $Y(i)$, we directly use $a_i a_j = \E{}{v_i v_j}$ and knowledge of $a_i$'s sign to get the sign of $a_j$. If $v_i$ and $v_j$ are not conditionally independent given $Y(i)$, we need two steps to recover the sign: for some $v_k$ that is conditionally independent of both $v_i$ and $v_j$ given $Y(i)$, we first use $a_i \E{}{v_k Y(i)} = \E{}{v_i v_k}$ to get the sign of $\E{}{v_k Y(i)}$. Then we use $a_j \E{}{v_k Y(i)} = \E{}{v_j v_k}$ to get the sign of $a_j$. Therefore, knowing the sign of one accuracy per $Y$ is sufficient to recover all signs.

The \textsc{ResolveSigns} used in Algorithm 1 uses another approach and follows from the assumption that on average per $Y$, the accuracies $a_i$ are better than zero. We apply this procedure to the sets of accuracies corresponding to each hidden variable; for each set, we have two sign choices, and we check which of these two produces a non-negative sum for the accuracies. In the common case where there is just one task, there are only two choices to check overall.

\subsection{Extensions to More Complex Graphical Models}
\label{subsec:extensions}

Recall that our Ising model is constructed for binary task labels, with sufficient conditional independence on $G$ and $G_{dep}$ such that $\Omega_G = V$, and without singleton potentials. We address how to extend our method when each of these conditions do not hold.

\paragraph{Multiclass Case} We have given an algorithm for binary classes for $\bm{Y}$ (and ternary for the sources, since these can also abstain). To extend this to higher-class cases, we can apply a one-versus-all reduction repeatedly to apply our core algorithm.

\paragraph{Extension to More Complex Graphs}
In Algorithm 1, %~\ref{alg:triplet}, 
we rely on the fact $\Omega_G = V$ to compute all accuracies. However, certain $a_i$'s cannot be recovered when there are fewer than $3$ conditionally independent subgraphs in $G$, where a subgraph $V_a$ is defined as a set of vertices such that if $v_i \in V_a$ and $v_j \notin V_a$, $v_i \independent v_j | Y(i)$.  Instead, when there are only $1$ or $2$ subgraphs, we use another independence property, which states that $v_i Y(i) \independent Y(i)$ for all $v_i$. This means that $\E{}{v_i Y(i)} \cdot \E{}{Y(i)} = \E{}{v_i Y(i)^2} = \E{}{v_i}$, and thus $a_i = \frac{\E{}{v_i}}{\E{}{Y(i)}}$. This independence property does not require us to choose triplets of sources; instead we can directly divide to compute $a_i$. However, this approach fails in the presence of singleton potentials and can be very inaccurate when $\E{}{Y(i)}$ is close to $0$. One can use this independence property in addition to Proposition 1 %~\ref{lemma:triplet} 
on $G$ with $2$ conditionally independent subgraphs, and when $G$ only consists of $1$ subgraph, we require that there are no singleton potentials on any of the sources.

\paragraph{Dealing with Singleton Potentials}

Our current Ising model does not include singleton potentials except on $Y_i$ terms. However, we can handle cases where sources are modeled to have singleton potentials. Proposition 1 %\ref{lemma:triplet} 
holds as long as either $v_i$ or $v_j$ belongs to a subgraph that has no potentials on individual observed variables. Therefore, the triplet method is able to recover mean parameters as long as we have at least two conditionally independent subgraph with no singleton potentials on observed variables. For example, just two sources conditionally independent of all the others with no singleton potential suffices to guarantee that this modified graphical model still allows for our algorithm to recover label model parameters.

In the case where we have singleton potentials on possibly every source, we have the following alternative approach. We use a slightly different parametrization and a quadratic version of the triplet method. Instead of tracking mean parameters (and thus accuracies like $\E{}{v_i Y(i)}$, we shall instead directly compute parameters that involve \emph{class-conditional} probabilities. These are, in particular, for $v_i$,
\[\mu_i =
  \begin{bmatrix}
    P(v_i = 1|Y(i) = 1) & P(v_i = 1|Y(i) = -1)  \\
    P(v_i = -1|Y(i) = 1) & P(v_i = -1|Y(i) = -1)
  \end{bmatrix}.\] 
Note that these parameters are minimal (the terms $P(v_i = 0 | Y(i) = \pm 1)$, indicating the conditional abstain rate, are determined by the columns above.

We set
\[O_{ij} =
  \begin{bmatrix}
    P(\lf_i = 1| \lf_j = 1) & P(\lf_i = 1|\lf_j = -1)  \\
    P(\lf_i = -1|\lf_j = 1) & P(\lf_i = -1|\lf_j = -1)
  \end{bmatrix}
  \text{    and    }
  P =
  \begin{bmatrix}
    P(Y=1) & 0  \\
    0 & P(Y=-1)
  \end{bmatrix}.
  \]

For a pair of conditionally independent sources, we have that 
\begin{align}
\mu_i P \mu_j^T = O_{ij}.
\label{eq:newparam}
\end{align}
 Because we can observe terms like $O_{ij}$, we can again form triplets with $i,j,k$ as before, and solve. Note that this alternative parametrization does not depend on the presence or absence of singleton potentials in the Ising model, only on the conditional independences directly defined by it. 

Moreover, there is a closed form solution to the resulting system of non-linear equations. To see this, consider the following. Note that 
\[P(v_i = 1|Y(i) = -1) = \frac{P(v_i = 1)}{P(Y(i)=-1)} - \frac{P(v_i =1 |Y(i)=1)P(Y(i)=1)}{P(Y(i)=-1)}.\]
Note that everything is observable (or known, for class balances), so that we can write the top row of $\mu_i$ as a function of a single variable. That is, we set $\alpha = P(v_i =1 |Y(i)=1)$, $c_i = \frac{P(v_i = 1)}{P(Y(i)=-1)} $ and $d_i = \frac{P(Y(i)=1)}{P(Y(i)=-1)}$. Then, the top row of $\mu_i$ becomes $[\alpha \quad c_i - d_i \alpha]$, and $c_i$ and $d_i$ are known.

Next, consider some triplets $i,j,k$, with corresponding $\mu$'s. Similarly, we set the top-left corner in the corresponding $\mu$'s to be $\alpha, \beta, \gamma$, and the corresponding terms for the top-right corner are $c_i, c_j, c_k$ and $d_i, d_j, d_k$. Then, by considering the upper-left position in \eqref{eq:newparam}, we get the system
\begin{align*}
(1+d_i d_j) \alpha \beta + c_i c_j - c_i d_j \beta - c_j d_i \alpha &= O_{ij}/P(Y=1), \\
(1+d_i d_k)\alpha \gamma + c_i c_k - c_i d_k \gamma - c_k d_i\alpha &= O_{ik}/P(Y=1), \\
(1+d_j d_k)\beta \gamma + c_j c_k - c_j d_k \gamma - c_k d_j \beta &= O_{jk}/P(Y=1). \\
\end{align*}
To solve this system, we express $\alpha$ and $\gamma$ in terms of $\beta$, using the first and third equations, and then we can plug these into the second and multiply (for example, when using $\alpha$, by $((1+d_i d_j )\beta - c_jd_i)^2$) to obtain a quadratic in terms of $\beta$. Solving this quadratic and selecting the correct root, then obtaining the remaining parameters ($\alpha, \gamma$) and filling in the rest of the $\mu_i, \mu_j, \mu_k$ terms completes the procedure. Note that we have to carry out the triplet procedure here twice per $\mu_i$, since there are two rows. Lastly, we can convert probabilities over $\bm{v}$ into equivalent probabilities over $\bm{\lf}$ as discussed in Appendix \ref{subsubsec:augment}.

\subsection{Online Algorithm}
\label{subsec:online}

The online learning setting presents new challenges for weak supervision. In the offline setting, the weak supervision pipeline has two distinct components: first, computing all probabilistic labels for a dataset and then using them to train an end model. In the online setting however, samples are introduced one by one, so we see each $\bm{X}^i$ only once and are not able to store it. 

Fortunately, Algorithm 1 %~\ref{alg:triplet}
and Algorithm 2 %~\ref{alg:full} 
both rely on computing estimates of expected moments over the observable weak sources. Since these are just averages, we can efficiently produce an estimate of the label model parameters at each time step. For each new sample, we update the averages of the moments using a rolling window and use them to output its probabilistic label; then the end model is trained on this sample, and the data point itself is no longer needed for further computation. Our method is fast enough that we can ``interleave'' the two components of the weak supervision pipeline, in comparison to \citet{Ratner19} and \citet{sala2019multiresws}, which require a full covariance matrix inversion and SGD.   

The online learning environment is also subject to \textit{distributional drift} over time, where old samples may come from very different distributions compared to more recent samples. Formally, define distributional drift as the following property: for $(\bm{X}^t, \bm{Y}^t) \sim P_t$, the KL-divergence between $P_i$ and $P_{i + 1}$ is less than $KL(P_t, P_{t + 1}) \le \Delta$ for any $t$. If there were no distributional drift, i.e., $\Delta = 0$, we would invoke Algorithm 1 %\ref{alg:triplet} 
or 2 %\ref{alg:full} 
at each time step $t$ for the new sample's output label,  where the estimates of $\Ehat{v_i v_j}$ and other observable moments would be cumulatively over $t$ rather than $n$. However, because of distributional drift, it is important to prioritize most recent samples. We propose a rolling window of size $W$, which can be optimized theoretically, to average over rather than all past $t$ samples. Algorithm \ref{alg:online} describes the general meta-algorithm for the online setting.

\begin{algorithm}[t]
	\caption{Online Weak Supervision}
	\begin{algorithmic}
		\STATE \textbf{Input:} dependency graph $G_{dep}$, window $W$ for rolling averages
		\FOR{$t = 1, 2, \dots$:}
			\STATE Receive source output vector $l_t$ and distribution prior $P_t(\bm{\bar{Y}})$.
			\STATE Run Algorithm~1 and Algorithm~2 with estimates computed over $W$ samples $l_{t - W + 1 : t}$ and their augmented equivalents to output $\bm{\hat{\mu}}_t$.
			\STATE Use junction tree formula to produce probabilistic output $\bm{\widetilde{Y}}^t \sim P_{\bm{\hat{\mu}}_t} (\, \cdot \, | l_t)$.
			\STATE Use $\bm{\widetilde{Y}}^t$ to update $w_t$, the parametrization of the end model $f_w$.
		\ENDFOR
	\end{algorithmic}
	\label{alg:online}
\end{algorithm}

\subsubsection{Theoretical Analysis} 
\label{subsubsec:online_theory}

Similar to the offline setting, we analyze our method for online label model parameter recovery and provide bounds on its performance. First, we derive a bound on the sampling error $||\bm{\mu}_t - \bm{\hat{\mu}}_t ||_2$ in terms of the window size $W$, concluding that there exists an optimal $W^*$ to minimize this error. Then, we present an online generalization result that describes how well our end model can ``track'' new samples coming from a drifting distribution.

\paragraph{Controlling the Online Sampling Error with $W$}

The sampling error at each time step $t$ $||\bm{\mu}_t - \bm{\hat{\mu}}_t ||_2$ is dependent on the window size $W$ which we average samples over to produce estimates. On one hand, a small window will ensure that the estimate will be computed using samples from distributions close to $P_t$, but using few samples results in a high empirical estimation error. On the other hand, a larger window will allow us to use many samples; however, samples farther in the past will be from distributions that may not be similar to $P_t$. Hence, $W$ must be selected to minimize both the effect of using drifting distributions and the estimation error in the number of samples used.

\begin{theorem}
Let $\bm{\hat{\mu}}_t$ be an estimate of $\bm{\mu}_t$, the label model parameters at time $t$, over $W$ previous samples from the product distribution $\mathbf{Pr}_W = \prod_{i =t - W + 1}^t P_i$, which suffers a $\Delta$-distributional drift. Then, still assuming cliques in $G_{dep}$ are limited to $3$ vertices,
\begin{align*}
\E{\mathbf{Pr}_W}{||\bm{\hat{\mu}}_t - \bm{\mu}_t||_2} = \frac{1}{a^5_{\min}} \left(3.19 C_1 \sqrt{\frac{m}{W}} + \frac{6.35 C_2}{\sqrt{r}} \frac{m}{\sqrt{W}} \right) + \frac{2c (|\mathcal{C}_{dep}| + |\mathcal{S}_{dep}|) \Delta W^{3/2}}{\sqrt{6 \alpha_{P_t}}}.
\end{align*}
where $\alpha_{P_t}$ is the minimum non-zero probability that $P_t$ takes.  
A global minimum for the sampling error as a function of $W$ exists, so the window size can be set such that $W^* = \argmin{W}{\E{}{||\bm{\hat{\mu}}_t - \bm{\mu}_t||_2}}$.
\end{theorem}

\begin{proof}
Denote $P_t^W =\underbrace{P_t \times \dots P_t}_{W}$. We first bound the difference between $\E{\mathbf{Pr}_W}{||\bm{\hat{\mu}}_t - \bm{\mu}_t||_2}$ and $\E{P_t^W}{||\bm{\hat{\mu}}_t - \bm{\mu}_t||_2}$.
\begin{align*}
\Big| \E{\mathbf{Pr}_W}{\|\bm{\hat{\mu}}_t - \bm{\mu}_t\|_2} - \E{P_t^W}{\|\bm{\hat{\mu}}_t - \bm{\mu}_t\|_2} \Big| &= \Big| \; \;\sum_{\mathclap{\{x_i\}_{i = t - W + 1}^t}} \;\; \|\bm{\hat{\mu}}_t - \bm{\mu}_t \|_2 \cdot (\mathbf{Pr}_W (x_{t -w + 1}, \dots, x_t) - P_t^W(x_{t - w + 1}, \dots, x_t)) \Big| \\
& \le \max \|\bm{\hat{\mu}}_t - \bm{\mu}_t \|_2 \cdot \sum_{\mathclap{\{x_i\}_{i = t - W + 1}^t}} \;\; |\mathbf{Pr}_W (x_{t -w + 1}, \dots, x_t) - P_t^W(x_{t - w + 1}, \dots, x_t) | \\
&= \max \|\bm{\hat{\mu}}_t - \bm{\mu}_t \|_2 \cdot 2 TV(\mathbf{Pr}_W, P_t^W).
\end{align*}

Since the label model parameters are all probabilities, $\|\bm{\hat{\mu}}_t - \bm{\mu}_t \|_2$ is bounded by $c \cdot (|\mathcal{C}_{dep}| + |\mathcal{S}_{dep}|)$, where $c$ is a constant. To compute $TV(\mathbf{Pr}_W, P_t^W)$, we use Pinsker's inequality and tensorization of the KL-divergence:
\begin{align*}
TV(\mathbf{Pr}_W, P_t^W) &\le \sqrt{\frac{1}{2} KL(\mathbf{Pr}_W || P_t^W)} = \sqrt{\frac{1}{2} KL(P_{t - W + 1} \times \dots \times P_t || P_t \times \dots \times P_t)} \\
&= \sqrt{\frac{1}{2} \sum_{i = t - W + 1}^t KL(P_i || P_t)}.
\end{align*}

Each $KL(P_i || P_t)$ can be bounded above by $\frac{2}{\alpha_{P_t}} TV(P_i, P_t)^2$ by the inverse of Pinsker's inequality, where $\alpha_{P_t} = \min_{x \in \mathcal{X}, P_t(x) > 0} P_t(x)$. Since the triangle inequality is satisfied for total variation distance, $TV(P_i, P_t) \le \Delta (t - i)$. Plugging this back in, we get
\begin{align*}
TV(\mathbf{Pr}_W, P_t^W) &\le \sqrt{\frac{1}{2} \cdot \frac{2}{\alpha_{P_t}} \Delta^2 \sum_{i = t - W + 1}^t (t - i)^2} = \sqrt{\frac{\Delta^2}{\alpha_{P_t}} \sum_{i = 0}^{W - 1} i^2} \\
&= \sqrt{\frac{\Delta^2}{\alpha_{P_t}} \cdot \frac{(W - 1)W(2W - 1)}{6}} \le \frac{\Delta W^{3/2}}{\sqrt{6\alpha_{P_t}}} .
\end{align*}

Therefore,
\begin{align*}
\Big| \E{\mathbf{Pr}_W}{\|\bm{\hat{\mu}}_t - \bm{\mu}_t\|_2} - \E{P_t^W}{\|\bm{\hat{\mu}}_t - \bm{\mu}_t\|_2} \Big| &\le \frac{2c (|\mathcal{C}_{dep}| + |\mathcal{S}_{dep}|) \Delta W^{3/2}}{\sqrt{6 \alpha_{P_t}}}.
\end{align*}

Furthermore, the offline sampling error result applies over $P_t^W$, so $\E{P_t^W}{\|\bm{\hat{\mu}}_t - \bm{\mu}_t\|_2} \le \frac{1}{a^5_{\min}} \left(3.19 C_1 \sqrt{\frac{m}{W}} + \frac{6.35 C_2}{\sqrt{r}} \frac{m}{\sqrt{W}} \right)$. Hence,
\begin{align*}
\E{\mathbf{Pr}_W}{\|\bm{\hat{\mu}}_t - \bm{\mu}_t\|_2} \le \frac{1}{a^5_{\min}} \left(3.19 C_1 \sqrt{\frac{m}{W}} + \frac{6.35 C_2}{\sqrt{r}} \frac{m}{\sqrt{W}} \right) + \frac{2c (|\mathcal{C}_{dep}| + |\mathcal{S}_{dep}|) \Delta W^{3/2}}{\sqrt{6 \alpha_{P_t}}},
\end{align*}

and we set a window size $W^*$ to minimize this expression.
\end{proof}

\paragraph{Online Generalization Bound}

We provide a bound quantifying the gap in probability of incorrectly classifying an unseen $t+1$th sample between our learned end model parametrization and an optimal end model parametrization.

Because the online learning setting is subject to distributional drift over time, our methods must be able to predict the next time step's label with some guarantee despite the changing environment. The $\Delta$ drift is aggravated by $(1)$ potential model misspecification for each $P_t$ and $(2)$ sample noise. However, we are able to take into account these additional conditions by modeling the overall drift $\Delta^{\mu}$ to be a combination of intrinsic distributional drift $\Delta$, model misspecification, and estimation error of parameters. 

Recall that $\bm{X}^i \sim P_i$ is drawn from the true distribution at time $i$, while $\bm{\widetilde{Y}}_i \sim P_{\bm{\hat{\mu}}_i}(\cdot |\bm{\lf}(\bm{X}^i))$ is the probabilistic output of our label model. Define the joint distribution of a sample to be $(\bm{X}^i, \bm{\widetilde{Y}}^i) \sim P_{i, \bm{\hat{\mu}}_i}$. At each time step $t$, our goal is train our end model $f_w \in \mathcal{F}$ and evaluate its performance against the true $(\bm{X}^t, \bm{Y}^t) \sim P_t$, given that we have $t - 1$ previous samples drawn from $P_{i,\bm{\hat{\mu}}_i}$.

We define a binary loss function $L(w, x, y) = | f_w(x) - y|$ and choose $\hat{w}_t$ to minimize over the past $s$ samples such that
\begin{align*}
\hat{w}_t = \argmin{w}{\frac{1}{s}\sum_{i = t - s}^{t - 1} L(w, \bm{X}^i, \bm{\widetilde{Y}}^i)}.
\end{align*}
We present a new generalization result that bounds the probability that $f_{\hat{w}_t}(\bm{X}^t)$ does not equal the true $\bm{Y}^t$ and also accounts for model misspecification and error from parameter estimation.

\begin{theorem}
Define $\Delta^{\mu} := d_{TV}(P_{i, \bm{\hat{\mu}}_i}, P_{i +1, \bm{\hat{\mu}}_{i + 1}})$ to be the distributional drift between the two samples and $D^{\mu} := \max_i d_{TV}(P_i, P_{i,\bm{\hat{\mu}}_i})$ to be an upper bound for the total variational distance between the true distribution and the noise aware misspecified distribution. If $\Delta^{\mu} \le \frac{c(\epsilon - 8D^{\mu})^3}{\mathrm{VCdim}(\mathcal{F})}$ for some constant $c > 0$, there exists a $\hat{w}_t$ computed over the past $s = \Big\lfloor \frac{\epsilon - 8D^{\mu}}{16 \Delta^{\mu}}\Big\rfloor$ samples such that, for any time $t > s$ and $\epsilon \in (8D^{\mu}, 1)$,
\begin{align*}
\mathbf{Pr}_{\bm{\hat{\mu}}, t} (L(\hat{w}_t, \bm{X}^t, \bm{Y}^t) = 1) \le \epsilon + \min_{w^*} P_t(L(w^*, \bm{X}^t, \bm{Y}^t) = 1),
\end{align*}
where $\mathbf{Pr}_{\bm{\hat{\mu}}, t} = \prod_{i = t- s}^{t - 1} P_{i, \bm{\hat{\mu}}_i} \cdot P_t$. Furthermore, 
\begin{align*}
D^{\mu} \leq \sqrt{\frac{1}{2} \max_i KL(P_i(\bm{Y} | \bm{X} ) \;||\; P_{\bm{\mu}_i}(\bm{Y}|\bm{X}))} + m^{\frac{1}{4}}\sqrt{\frac{1}{e_{min}} \max_i || \bm{\mu}_i - \bm{\hat{\mu}}_i||_2}.
\end{align*} %and thus is also a function of model misspecification and sampling error.
\label{thm:gen_online}
\end{theorem}

\begin{proof}
We adapt Theorem $2$ from \citet{Long1999}. Choose $\epsilon \le 1$. Let $s = \Big\lfloor \frac{\epsilon - 8D^{\mu}}{16(\Delta + 2D^{\mu})}\Big\rfloor$ and $\Delta^{\mu} = \Delta + 2D^{\mu} \le \frac{(\epsilon - 8D^{\mu})^3}{5000000d}$, where $d$ is the end model's VC dimension. Let $L(w, x, y) = |y - f_w(x)| \in \{0, 1\}$, where $f_w(x)$ is the output of the end model parametrized by $w$ when given input $x$.

At time $t$, the sequence of inputs to the end model so far is $(\bm{X}^1, \bm{\widetilde{Y}}^1), (\bm{X}^2, \bm{\widetilde{Y}}^2), \dots (\bm{X}^{t - 1}, \bm{\widetilde{Y}}^{t - 1})$, where $(\bm{X}^i, \bm{\widetilde{Y}}^i) \sim P_{i, \bm{\hat{\mu}}_i}$. We evaluate the end model's performance by using a parametrization $w_t$ that is a function of the $t-1$ inputs so far and computing $L(w_t, \bm{X}^t, \bm{Y}^t)$ where $(\bm{X}^t, \bm{Y}^t) \sim P_t$. In particular, let $w_t^* = \mathrm{argmin}_w \mathbb{E}_{(\bm{X}^t, \bm{Y}^t) \sim P_t}[L(w, \bm{X}^t, \bm{Y}^t)]$, and $\hat{w}_t = \mathrm{argmin}_w \frac{1}{s} \sum_{i = t-s}^{t-1} L(w, x_i, \tilde{y}_i)$ where $x_i, \tilde{y}_i$ are the values of the random variables $\bm{X}^i$ and $\bm{\widetilde{Y}}_i$.

Suppose that $TV(P_i, P_{i + 1}) \le \Delta$. Then $TV(P_{i, \bm{\hat{\mu}}_i}, P_{i + 1, \bm{\hat{\mu}}_{i + 1}})$ is 
\begin{align*}
TV(P_{i, \bm{\hat{\mu}}_i}, P_{i + 1, \bm{\hat{\mu}}_{i + 1}}) \le TV(P_{i, \bm{\hat{\mu}}_i}, P_i) + \Delta + TV(P_{i + 1}, P_{i + 1, \bm{\hat{\mu}}_{i + 1}}) \le \Delta + 2D^{\mu} = \Delta^{\mu}.
\end{align*}

Let $\beta \ge 6\Delta^{\mu} s + 4D^{\mu}$, and $\alpha = \frac{\beta}{2} - 2D^{\mu} \ge 3 \Delta^{\mu} s$. Note that $TV(P_{i, \bm{\hat{\mu}}_i}, P_{t, \bm{\hat{\mu}}_t}) \le \Delta^{\mu} s = \kappa$ for any $i = t - s, \dots, t - 1$. Denote $\mathbf{Pr}_{\bm{\hat{\mu}}} = \prod_{i = t-s}^{t-1} P_{i, \bm{\hat{\mu}}_i}$.  Then by Lemma $12$ of \citet{Long1999},
\begin{align*}
\mathbf{Pr}_{\bm{\hat{\mu}}} \Big\{\exists w: \Big| \frac{1}{s} \sum_{i = t - s}^{t - 1} L(w, \bm{X}^i, \bm{\widetilde{Y}}^i) - \E{(\bm{X}^t, \bm{\widetilde{Y}}^t) \sim P_{t, \bm{\hat{\mu}}_t}}{L(w, \bm{X}^t, \bm{\widetilde{Y}}^t)}\Big| > \alpha \Big\} \le 8 \cdot 41^d \exp \left(-\frac{(\alpha - \kappa)^2 s}{1600} \right).
\end{align*}

For any real numbers $a, b, c$, and $x > y$, if $|a - b| \ge x$ and $|b - c| \le y$, then $|a - b| - |b - c| \ge x - y$ and thus $|a - c| = |a - b + b - c| \ge ||a - b| - |b - c|| \ge x - y$. Applying this,
\begin{align*}
&\mathbf{Pr}_{\bm{\hat{\mu}}}  \Big\{\exists w: \Big| \frac{1}{s} \sum_{i = t - s}^{t - 1} L(w, \bm{X}^i, \bm{\widetilde{Y}}^i) - \E{(\bm{X}^t, \bm{Y}^t) \sim P_t}{L(w, \bm{X}^t, \bm{Y}^t)}\Big| > \alpha + 2D^{\mu}, \\
&\Big|\mathbb{E}_{(\bm{X}^t, \bm{Y}^t) \sim P_t}[L(w, \bm{X}^t, \bm{Y}^t)] - \E{(\bm{X}^t, \bm{\widetilde{Y}}^t) \sim P_{t, \bm{\hat{\mu}}_t}}{L(w, \bm{X}^t, \bm{\widetilde{Y}}^t)} \Big| < 2D^{\mu} \Big\} \\
\le & \; \mathbf{Pr}_{\bm{\hat{\mu}}}  \Big\{\exists w: \Big| \frac{1}{s} \sum_{i = t - s}^{t - 1} L(w, \bm{X}^i, \bm{\widetilde{Y}}^i) - \E{(\bm{X}^t, \bm{\widetilde{Y}}^t) \sim P_{t, \bm{\hat{\mu}}_t}}{L(w, \bm{X}^t, \bm{\widetilde{Y}}^t)}\Big| > \alpha \Big\} \le 8 \cdot 41^d \exp \left(-\frac{(\alpha - \kappa)^2 }{1600} \right).
\end{align*}

By Lemma \ref{lemma:expectation}, the difference in the expected loss $\mathbb{E}[L(w, \bm{X}^t, \bm{Y}^t)]$ when $\bm{X}^t, \bm{Y}^t$ is from $P_t$ versus $P_{t, \bm{\hat{\mu}}_t}$ is always less than $2D^{\mu}$, so the above becomes
\begin{align*}
&\mathbf{Pr}_{\bm{\hat{\mu}}}  \Big\{\exists w: \Big| \frac{1}{s} \sum_{i = t - s}^{t - 1} L(w, \bm{X}^i, \bm{\widetilde{Y}}^i) - \E{(\bm{X}^t, \bm{Y}^t) \sim P_t}{L(w, \bm{X}^t, \bm{Y}^t)}\Big| > \alpha + 2D^{\mu} \Big\} \\
&\le 8 \cdot 41^d \exp \left(-\frac{(\alpha - \kappa)^2 s}{1600} \right).
\end{align*}

We can write this in terms of $\beta$. Note that $\Delta^{\mu} s \le \frac{\beta}{6} - \frac{2D^{\mu}}{3}$. The RHS is equivalent to
\begin{align*}
&8 \cdot 41^d \exp  \left(-\frac{(\alpha - \kappa)^2 m}{1600} \right) = 8 \cdot 41^d \exp  \left(-\frac{s}{1600}\left(\frac{\beta}{2} - 2D^{\mu} - \Delta^{\mu} s\right)^2  \right) \\
&\le 8 \cdot 41^d \exp  \left(-\frac{s}{1600}\left(\frac{\beta}{2} - 2D^{\mu} - \frac{\beta}{6} + \frac{2D^{\mu}}{3}\right)^2  \right) = 8 \cdot 41^d \exp  \left(-\frac{s}{14400}(\beta - 4D^{\mu})^2  \right).
\end{align*}

So the probability becomes
\begin{align*}
&\mathbf{Pr}_{\bm{\hat{\mu}}}  \Big\{\exists w: \Big| \frac{1}{s} \sum_{i = t - s}^{t - 1} L(w, \bm{X}^i, \bm{\widetilde{Y}}^i) - \E{(\bm{X}^t, \bm{Y}^t) \sim P_t}{L(w, \bm{X}^t, \bm{Y}^t)}\Big| > \frac{\beta}{2} \Big\} \le 8 \cdot 41^d \exp  \left(-\frac{s}{14400}(\beta - 4D^{\mu})^2  \right).
\end{align*}

Next, note that the probability that at least one of $\hat{w}_t$ or $w^*_t$ satisfies $\Big| \frac{1}{s} \sum_{i = t - s}^{t - 1} L(w, \bm{X}^i, \bm{\widetilde{Y}}^i) - \E{(\bm{X}^t, \bm{Y}^t) \sim P_i}{L(w, \bm{X}^t, \bm{Y}^t)}\Big| > \frac{\beta}{2}$ is less than the probability that there exists a $w$ that satisfies the above inequality. In general, if $|a - b| > \beta$, then $|a| > \frac{\beta}{2}$ or $|b| > \frac{\beta}{2}$ (or both). Then
\begin{align*}
&\mathbf{Pr}_{\bm{\hat{\mu}}}  \Big\{ \Big|\frac{1}{s} \sum_{i = t-s}^{t-1} L(w^*_t, \bm{X}^i, \bm{\widetilde{Y}}^i) - \mathbb{E}_{(\bm{X}^t, \bm{Y}^t) \sim P_t}[L(w^*_t, \bm{X}^t, \bm{Y}^t)]  \\
&-\frac{1}{s} \sum_{i = t - s}^{t - 1} L(\hat{w}_t, \bm{X}^i, \bm{\widetilde{Y}}^i) + \mathbb{E}_{(\bm{X}^t, \bm{Y}^t) \sim P_t}[L(\hat{w}_t, \bm{X}^t, \bm{Y}^t)]\Big| > \beta \Big\} \\
\le &\; \mathbf{Pr}_{\bm{\hat{\mu}}} \Big\{ \Big|\frac{1}{s} \sum_{i = t-s}^{t-1} L(w^*_t, \bm{X}^i, \bm{\widetilde{Y}}^i) - \mathbb{E}_{(\bm{X}^t, \bm{Y}^t) \sim P_t}[L(w^*_t, \bm{X}^t, \bm{Y}^t)]| > \frac{\beta}{2}, \; \cup  \\
&\Big|-\frac{1}{s} \sum_{i = t - s}^{t - 1} L(\hat{w}_t, \bm{X}^i, \bm{\widetilde{Y}}^i) + \mathbb{E}_{(\bm{X}^t, \bm{Y}^t) \sim P_t}[L(\hat{w}_t, \bm{X}^t, \bm{Y}^t)]\Big| > \frac{\beta}{2}\Big\} \\
\le &\;8 \cdot 41^d \exp  \left(-\frac{s}{14400}(\beta - 4D^{\mu})^2  \right).
\end{align*}

By definition of $w^*_t$ and $\hat{w}_t$, $\frac{1}{s} \sum_{i = t-s}^{t-1} L(w^*_t, \bm{X}^i, \bm{\widetilde{Y}}^i) > \frac{1}{s} \sum_{i = t - s}^{t - 1} L(\hat{w}_t, \bm{X}^i, \bm{\widetilde{Y}}^i)$ and $\E{(\bm{X}^t, \bm{Y}^t) \sim P_t}{L(\hat{w}_t, \bm{X}^t, \bm{Y}^t)} > \mathbb{E}_{(\bm{X}^t, \bm{Y}^t) \sim P_t}[L(w^*_t, \bm{X}^t, \bm{Y}^t)]$. Therefore,
\begin{align*}
&\mathbf{Pr}_{\bm{\hat{\mu}}}  \Big\{ \mathbb{E}_{(\bm{X}^t, \bm{Y}^t) \sim P_t}[L(\hat{w}_t, \bm{X}^t, \bm{Y}^t)]-  \mathbb{E}_{(\bm{X}^t, \bm{Y}^t) \sim P_t}[L(w^*_t, \bm{X}^t, \bm{Y}^t)]  > \beta \Big\} \\
&\le 8 \cdot 41^d \exp  \left(-\frac{s}{14400}(\beta - 4D^{\mu})^2  \right).
\end{align*}

Now we apply Lemma 13 from \citet{Long1999}. Define 
\begin{align*}
\phi(\beta) = \begin{cases} 8 \cdot 41^d \exp  \left(-\frac{s}{14400}(\beta - 4D^{\mu})^2  \right) & \beta \ge 6 \Delta^{\mu} s + 4 D^{\mu} \\
1 & o.w. \end{cases}.
\end{align*}

Let $a_0 = 0$ and $a_1 = 6 \Delta^{\mu} s + 4D^{\mu}$. For all other $a_i$ where $i > 1$ until some $a_n$ where $a_{n + 1} > 1$, define $a_i = \sqrt{\frac{14400 (\ln 8 + (\ln 41)d + i \ln 2)}{s}} + 4D^{\mu}$. This way, $\phi(a_{i > 1}) = 2^{-i}$. Then Lemma $13$ states
\begin{align*}
&\mathbb{E}_{\{(\bm{X}^i, \bm{\widetilde{Y}}^i) \sim P_{i, \bm{\hat{\mu}}_i} \}_{i = t-s}^{t-1}}[P_t(L(\hat{w}_t, \bm{X}^t, \bm{Y}^t) = 1) - P_t(L(w^*_t, \bm{X}^t, \bm{Y}^t) = 1)] \\
&\le 1 \cdot a_1 + \sum_{i = 1}^{\infty} \left(\sqrt{\frac{14400 (\ln 8 + (\ln 41)d + i \ln 2)}{s}} + 4 D^{\mu}\right) 2^{-i} \\
&\le 6\Delta^{\mu} s + 4D^{\mu} + 341 \sqrt{\frac{d}{s}} + 4D^{\mu} = 6\Delta^{\mu} s + 8D^{\mu} + 341 \sqrt{\frac{d}{s}}.
\end{align*}

Plugging in our values of $s$ and $\Delta^{\mu}$, we get that $6 \Delta^{\mu} s + 8 D^{\mu} + 341 \sqrt{\frac{d}{s}} \le \epsilon$. Therefore, if the drift between two consecutive samples is less than $TV(P_{i, \bm{\hat{\mu}}_i}, P_{i+ 1, \bm{\hat{\mu}}_{i+1}}) \le \Delta^{\mu} \le \frac{(\epsilon - 8D^{\mu})^3}{5000000d}$, there exists an algorithm that computes a $\hat{w}_t$ over the past $s = \Big\lfloor \frac{\epsilon - 8D^{\mu}}{16(\Delta + 2D^{\mu})}\Big\rfloor$ inputs to the end model, such that 
\begin{align*}
\mathbf{Pr}_{\bm{\hat{\mu}}, t} (L(\hat{w}_t, \bm{X}^t, \bm{Y}^t) = 1) \le \epsilon + \min_{w^*} P_t(L(w^*, \bm{X}^t, \bm{Y}^t) = 1),
\end{align*}
where $D^{\mu} \le \sqrt{\frac{1}{2} \max_i \mathbb{E}_{\bm{X} \sim P_i}[KL(P_i(\bm{Y} | \bm{X} ) \;||\; P_{\bm{\mu}_i}(\bm{Y}|\bm{X}))]} + m^{1/4}\sqrt{\frac{1}{\sigma_{min}} \max_i|| \bm{\mu}_i - \bm{\hat{\mu}}_i||_2}$ by Lemma \ref{lemma:Dmu}.
\end{proof}

\begin{lemma} The difference in the expected value of $L(w, \bm{X}, \bm{Y})$ when samples are drawn from $P_{t, \bm{\hat{\mu}}_t}$ versus $P_t$ is
\begin{align*}
\Big| \mathbb{E}_{(\bm{X}^t, \bm{\widetilde{Y}}^t) \sim P_{t, \bm{\hat{\mu}}_t}}[L(w, \bm{X}^t, \bm{\widetilde{Y}}^t)] - \mathbb{E}_{(\bm{X}^t, \bm{Y}^t) \sim P_t}[L(w, \bm{X}^t, \bm{Y}^t)] \Big| \le 2D^{\mu}.
\end{align*}
\begin{proof}
We use the definition of total variation distance:
\begin{align*}
&\Big| \mathbb{E}_{(\bm{X}^t, \bm{\widetilde{Y}}^t)) \sim P_{t, \bm{\hat{\mu}}_t}}[L(w, \bm{X}^t, \bm{\widetilde{Y}}^t] - \mathbb{E}_{(\bm{X}^t, \bm{Y}^t) \sim P_t}[L(w, \bm{X}^t, \bm{Y}^t)] \Big| \\
=\; & \Big| \sum_{x, y} L(w, x, y)(P_{t, \bm{\hat{\mu}}_t}(x, y) - P_t(x, y)) \Big|\\
\le\; &  \sum_{x, y} L(w, x, y) |P_{t, \bm{\hat{\mu}}_t}(x, y) - P_t(x, y)| \\
\le \; & \sum_{x, y} |P_{t, \bm{\hat{\mu}}_t}(x, y) - P_t(x, y)| = 2 TV(P_{t, \bm{\hat{\mu}}_t}, P_t) \le 2D^{\mu}.
\end{align*}
\end{proof}
\label{lemma:expectation}
\end{lemma}

\begin{lemma}
\begin{align*}
D^{\mu} \le \sqrt{\frac{1}{2}\max_i  KL(P_i(\bm{Y}|\bm{X}) \; || \; P_{\bm{\mu}_i}(\bm{Y}|\bm{X}))} + m^{1/4}\sqrt{\frac{1}{\sigma_{min}} \max_i ||\bm{\mu}_i - \bm{\hat{\mu}}_i ||_2}.
\end{align*}
Here, $\sigma_{min}$ is the minimum singular value of the covariance matrix $\Sigma$ of the variables $V = \{\bm{Y}, \bm{v} \}$ in the graphical model.
\label{lemma:Dmu}
\end{lemma}

\begin{proof}

We first use the triangle inequality on TV distance to split $D^{\mu}$ into two KL-divergences.
\begin{align*}
D^{\mu} &\le \max_i TV(P_{i, \bm{\hat{\mu}}_i}, P_i) \le \max_i TV(P_{i, \bm{\hat{\mu}}_i}, P_{i, \bm{\mu}_i}) + \max_i TV(P_{i, \bm{\mu}_i}, P_i) \\
&\le \sqrt{\frac{1}{2} \max_i KL(P_{i, \bm{\mu}_i}|| P_{i, \bm{\hat{\mu}}_i})} + \sqrt{\frac{1}{2} \max_i KL(P_i || P_{i, \bm{\mu}_i})}.
\end{align*}

To simplify the first divergence, we use the binary Ising model definition in \eqref{eq:vising}, which for simplicity we write as $f_G(\bm{Y}, \bm{v}) = \frac{1}{Z}\exp(\theta^T \phi(V))$, where $\phi(V)$ is the vector of all potentials. 
\begin{align*}
KL(P_{i, \bm{\mu}_i}|| P_{i, \bm{\hat{\mu}}_i}) &= (\hat{\theta}_i - \theta_i)^T \mathbb{E}[\phi(V)] + \ln \frac{\hat{Z}}{Z} \le |\hat{\theta}_i - \theta_i|_1 + \ln \frac{\hat{Z}}{Z} \le \sqrt{m}||\hat{\theta}_i - \theta_i||_2 + \ln \frac{\sum_{s \in \mathcal{S}} \exp(\hat{\theta}_i^T \phi(s))}{\sum_{s \in \mathcal{S}} \exp(\theta_i^T \phi(s))} \\
&\le \sqrt{m}||\hat{\theta}_i - \theta_i||_2 + \frac{1}{\hat{Z}}\sum_{s \in \mathcal{S}} \exp(\hat{\theta}_i^T \phi(s))\ln \frac{\exp(\hat{\theta}_i^T \phi(s))}{ \exp(\theta_i^T \phi(s))}\\
&\le \sqrt{m}||\hat{\theta}_i - \theta_i||_2 + \frac{1}{\hat{Z}}\sum_{s \in \mathcal{S}} \exp(\hat{\theta}_i^T \phi(s)) ((\hat{\theta}_i - \theta_i)^T \phi(s) ) \\
&\le \sqrt{m}||\hat{\theta}_i - \theta_i||_2 + \frac{1}{\hat{Z}}\sum_{s \in \mathcal{S}} \exp(\hat{\theta}_i^T \phi(s)) \sqrt{m}||\hat{\theta}_i - \theta_i||_2 \le 2\sqrt{m} ||\hat{\theta}_i - \theta_i||_2 \\
&\le \frac{2\sqrt{m}}{\sigma_{min}} ||\bm{\hat{\mu}}_i - \bm{\mu}_i||_2.
\end{align*}
Here we used $\phi(s), \E{}{\phi(V)} \in [-1, +1]$, the log sum inequality, and Lemma \ref{lemma:fenchel}. The second divergence can be simplified into a conditional KL-divergence.
\begin{align*}
KL(P_i || P_{i, \bm{\mu}_i}) &= \sum_{x, y} P_i(x, y) \ln \frac{P_i(x, y)}{P_{i, \bm{\mu}_i}(x, y)} = \sum_{x, y} P_i(x, y) \ln \frac{P_i(y|x) P_i(x)}{P_{i, \bm{\mu}_i}(y | x) P_{i, \bm{\mu}_i}(x)}  \\
&=  \sum_{x, y} P_i(x, y) \ln \frac{P_i(y|x) P_i(x)}{P_{\bm{\mu}_i}(y | x)P_i(x)} = \sum_{x} P_i(x) \sum_y P_i(y|x) \ln \frac{P_i(y|x)}{P_{\bm{\mu}_i}(y | x)} \\
&=  \sum_{x} P_i(x) KL(P_i(\bm{Y}|x) || P_{\bm{\mu}_i}(\bm{Y}|x)) = KL(P_i(\bm{Y}|\bm{X}) \; || \; P_{\bm{\mu}_i} (\bm{Y}|\bm{X})),
\end{align*}
where 
\begin{align*}
KL(P_i(\bm{Y}|\bm{X}) \; || \; P_{\bm{\mu}_i} (\bm{Y}|\bm{X})) = \mathbb{E}_{P_i} [KL(P_i(\bm{Y}|x) \; || \; P_{\bm{\mu}_i} (\bm{Y}|x))].
\end{align*}

\end{proof}

This result suggests that, with a small enough $\Delta^{\mu}$, our parametrization of the end model using past data will perform only $\epsilon$ worse in probability than the best possible parametrization of the end model on the next data point. Furthermore, note that $s$ is decreasing in $D^{\mu}$; more model misspecification and sampling error intuitively suggests that we want to use fewer previous data points to compute $\hat{w}_t$, so again having a simple yet suitable graphical model allows the end model to train on more data for better prediction.

\section{Proofs of Main Results}
\label{sec:proofs}

\subsection{Proof of Theorem 1 (Sampling Error)}%\ref{thm:sampling_offline} (Sampling Error)} %\ref{thm:sampling_offline}

We first present three concentration inequalities - one on the accuracies estimated via the triplet method, and the other two on directly observable values. Afterwards, we discuss how to combine these inequalities into a sampling error result for $\bm{\mu}$ when $G_{dep}$ has small cliques of size $3$ or less.

\paragraph{Estimation error for $a_i$ using Algorithm 1}%\ref{alg:triplet}}%\ref{alg:triplet}}

\begin{lemma}
Denote $M$ as the second moment matrix over all observed variables, e.g. $M_{ij} = \E{}{v_i v_j}$. Let $\hat{a}$ be an estimate of the $m$ desired accuracies $a$ using $\hat{M}$ computed from $n$ samples. Define $a_{\min} = \min \{ \min_i |\hat{a}_i|, \min_i |a_i|\}$, and assume $\text{sign}(a_i) = \text{sign}(\hat{a}_i)$ for all $a_i$. Furthermore, assume that the number of samples $n$ is greater than some $n_0$ such that $a_{\min} > 0$, and $\hat{M}_{ij} \neq 0$. Then the estimation error of the accuracies is
\begin{align*}
\Delta_a = \mathbb{E}[\|\hat{a} - a \|_2] \le C_a \frac{1}{a^5_{\min}} \sqrt{\frac{m}{n}},
\end{align*}
for some constant $C_a$.
\label{lemma:a}
\end{lemma}

\begin{proof}
We start with a few definitions. Denote a triplet as $T_i(1), T_i(2), T_i(3)$, and in total suppose we need $\tau$ number of triplets. Recall that our estimate of $a$ can be obtained with 
\begin{align*}
|\hat{a}_{T_i(1)}| = \left(\frac{|\hat{M}_{T_i(1) T_i(2)}| |\hat{M}_{T_i(1) T_i(3)}|}{|\hat{M}_{T_i(2) T_i(3)}|}\right)^{\frac{1}{2}}.
\end{align*}

Because we assume that signs are completely recoverable, 
\begin{align}
\|\hat{a} - a\|_2 = \| |\hat{a}| - |a| \|_2 \le \left( \sum_{i = 1}^{\tau} (|\hat{a}_{T_i(1)}| - |a_{T_i(1)}|)^2 + (|\hat{a}_{T_i(2)}| - |a_{T_i(2)}|)^2 + (|\hat{a}_{T_i(3)}| - |a_{T_i(3)}|)^2 \right)^{\frac{1}{2}}.
\label{eq:acc1}
\end{align}

Note that $|\hat{a}_i^2 - a_i^2| = |\hat{a}_i - a_i| |\hat{a}_i + a_i |$. By the reverse triangle inequality, $(|\hat{a}_i| - |a_i|)^2 = \|\hat{a}_i| - |a_i\|^2 \le |\hat{a}_i - a_i|^2 = \left( \frac{|\hat{a}_i^2 - a_i^2 |}{|\hat{a}_i + a_i|}\right)^2 \le \frac{1}{4a_{\min}^2} |\hat{a}^2_i - a^2_i|^2$, because $|\hat{a}_i + a_i| = |\hat{a}_i | + |a_i| \ge 2 a_{\min}$. For ease of notation, suppose we examine a particular $T_i = \{1, 2, 3\}$. Then
\begin{align}
(|\hat{a}_1| - |a_1|)^2 &\le \frac{1}{4a_{\min}^2}|\hat{a}_1^2 - a_1^2|^2 = \frac{1}{c^2} \Bigg| \frac{|\hat{M}_{12}| |\hat{M}_{13}|}{|\hat{M}_{23}|} - \frac{|M_{12}| |M_{13}|}{|M_{23}|} \Bigg|^2 \nonumber \\
&= \frac{1}{4a_{\min}^2} \Bigg| \frac{|\hat{M}_{12}| |\hat{M}_{13}|}{|\hat{M}_{23}|} - \frac{|\hat{M}_{12}| |\hat{M}_{13}|}{|M_{23}|} + \frac{|\hat{M}_{12}| |\hat{M}_{13}|}{|M_{23}|} - \frac{|\hat{M}_{12}| |M_{13}|}{|M_{23}|} +  \frac{|\hat{M}_{12}| |M_{13}|}{|M_{23}|} - \frac{|M_{12}| |M_{13}|}{|M_{23}|} \Bigg|^2 \nonumber \\
&\le\frac{1}{4a_{\min}^2} \left(\Big|\frac{\hat{M}_{12} \hat{M}_{13}}{\hat{M}_{23} M_{23}}\Big| \|\hat{M}_{23}| - |M_{23}| | + \Big|\frac{\hat{M}_{12}}{M_{23}} \Big| \|\hat{M}_{13}| - |M_{13}\| + \Big|\frac{M_{13}}{M_{23}} \Big| \|\hat{M}_{12}| - |M_{12}\|\right)^2 \nonumber \\
&\le \frac{1}{4a_{\min}^2} \left(\Big|\frac{\hat{M}_{12} \hat{M}_{13}}{\hat{M}_{23} M_{23}}\Big| |\hat{M}_{23} - M_{23}| + \Big|\frac{\hat{M}_{12}}{M_{23}} \Big| |\hat{M}_{13} - M_{13}| + \Big|\frac{M_{13}}{M_{23}} \Big| |\hat{M}_{12} - M_{12}|\right)^2.
\label{eq:a_min}
\end{align}

Clearly, all elements of $\hat{M}$ and $M$ must be less than $1$. We further know that elements of $|M|$ are at least $a_{min}^2$, since $\E{}{v_i v_j} = \E{}{v_i Y} \E{}{v_j Y} \ge a_{\min}^2$. Furthermore, elements of $|\hat{M}|$ are also at least $a_{\min}^2$ because $|\hat{M}_{ij}| = \hat{a}_i \hat{a_j} \ge a_{\min}^2$ by construction of our algorithm. Define $\Delta{ij} = \hat{M}_{ij} - M_{ij}$. Then
\begin{align*}
(|\hat{a}_1| - |a_1|)^2 &\le \frac{1}{4 a_{\min}^2} \left(\frac{1}{a_{\min}^4} |\Delta_{23}| + \frac{1}{a_{\min}^2} |\Delta_{13}| + \frac{1}{a_{\min}^2} |\Delta_{12}|\right)^2 \\
&\le \frac{1}{4a_{\min}^2}(\Delta_{23}^2 + \Delta_{13}^2 + \Delta_{12}^2)\left(\frac{1}{a^8_{\min}} + \frac{2}{a^4_{\min}}\right).
\end{align*}

\eqref{eq:acc1} is now
\begin{align*}
\|\hat{a} - a\|_2 \le \left( \frac{3}{4a_{\min}^2}\left(\frac{1}{a_{\min}^8} + \frac{2}{a_{\min}^4} \right)\sum_{i = 1}^{\tau} \Big( \Delta_{T_i(1) T_i(2)}^2 + \Delta_{T_i(1) T_i(3)}^2 + \Delta_{T_i(2) T_i(3)}^2 \Big) \right)^{\frac{1}{2}}.
\end{align*}

To bound the maximum absolute value between elements of $\hat{M}$ and $M$, note that the Frobenius norm of the $3 \times 3$ submatrix defined over $T_i$ is 
\begin{align*}
\|\hat{M}_{T_i} - M_{T_i}\|_F = \left(2\left(\Delta^2_{T_i(1) T_i(2)} + \Delta^2_{T_i(1) T_i(3)} + \Delta^2_{T_i(2) T_i(3)}\right) \right)^{\frac{1}{2}}.
\end{align*}

Moreover, $\|\hat{M}_{T_i} - M_{T_i}\|_F = \sqrt{\sum_{j = 1}^3 \sigma_j^2(\hat{M}_{T_i} - M_{T_i})} \le \sqrt{3} \|\hat{M}_{T_i} - M_{T_i} \|_2$. Putting everything together,
\begin{align*}
\| \hat{a} - a\|_2 &\le \left( \frac{3}{4 a_{\min}^2}\left(\frac{1}{a_{\min}^8} + \frac{2}{a_{\min}^4} \right) \cdot \frac{1}{2} \sum_{i = 1}^{\tau} \|\hat{M}_{T_i} - M_{T_i} \|_F^2 \right)^{\frac{1}{2}} \\
&\le \left( \frac{3}{4 a_{\min}^2}\left(\frac{1}{a_{\min}^8} + \frac{2}{a_{\min}^4} \right) \cdot \frac{3}{2} \sum_{i = 1}^{ \tau} \|\hat{M}_{T_i} - M_{T_i} \|_2^2 \right)^{\frac{1}{2}}.
\end{align*}

Lastly, to compute $\mathbb{E}[\|\hat{a} - a\|_2]$, we use Jensen's inequality and linearity of expectation:
\begin{align*}
\mathbb{E}\| \hat{a} - a\|_2] \le \left( \frac{3}{4a_{\min}^2}\left(\frac{1}{a_{\min}^8} + \frac{2}{a_{\min}^4} \right) \cdot \frac{3}{2} \sum_{i = 1}^{\tau} \mathbb{E}[\|\hat{M}_{T_i} - M_{T_i} \|_2^2] \right)^{\frac{1}{2}}.
\end{align*}

We use the matrix Hoeffding inequality as described in \citet{Ratner19}, which says
\begin{align*}
P(\|\hat{M} - M\|_2 \ge \gamma) \le 2m\exp\left(-\frac{n\gamma^2}{32m^2}\right).
\end{align*} 

To get the probability distribution over $\|\hat{M} - M\|_2^2$, we just note that $P(\|\hat{M} - M\|_2 \ge \gamma) = P(\|\hat{M} - M\|_2^2 \ge \gamma^2)$ to get
\begin{align*}
P(\|\hat{M} - M\|_2^2 \ge \gamma) \le 2m \exp\left(-\frac{n \gamma}{32m^2} \right).
\end{align*}

From which we can integrate to get
\begin{align*}
\mathbb{E}[\|\hat{M}_{T_i} - M_{T_i}\|^2_2 ] = \int_0^{\infty} P(\|\hat{M_{T_i}} - M_{T_i}\|_2^2 \ge \gamma) d\gamma \le \frac{64(3)^3}{n}.
\end{align*}

Substituting this back in, we get
\begin{align*}
\mathbb{E}[\| \hat{a} - a\|_2] &\le \left( \frac{3}{4a_{\min}^2}\left(\frac{1}{a_{\min}^8} + \frac{2}{a_{\min}^4} \right) \cdot \frac{3\tau}{2} \frac{1728}{n} \right)^{\frac{1}{2}} \\
&\le \left( \frac{1944}{a_{\min}^2} \cdot \left(\frac{1}{a_{\min}^8} + \frac{2}{a_{\min}^4} \right) \cdot \frac{\tau}{n}\right)^{\frac{1}{2}}.
\end{align*}

Finally, note that
\begin{align*}
\frac{1}{a_{\min}^2} \cdot \left(\frac{1}{a_{\min}^8} + \frac{2}{a_{\min}^4} \right) = \frac{1}{a_{\min}^2} \cdot \frac{1 + 2a_{\min}^4}{a_{\min}^8} \le \frac{3}{a_{\min}^{10}}.
\end{align*}

Therefore, the sampling error for the accuracy is bounded by
\begin{align*}
\mathbb{E}[\|\hat{a} - a\|_2] \le \left(\frac{1944 \cdot 3}{a^{10}_{\min}} \cdot  \frac{\tau}{n} \right)^{\frac{1}{2}} \le C_a \frac{1}{a^5_{\min}} \sqrt{\frac{m}{n}}.
\end{align*}

This is because at most we will use a triplet to compute each relevant $a_i$, meaning that $\tau \le m$. The term $C_a$ here is $18\sqrt{6}$.

\end{proof}

\begin{remark}
Although a lower bound on accuracy $a_{\min}$ invariably appears in this result, the dependence on a single low-accuracy source $\lf_{\min}$ can be reduced. We improve our bound from having a $\frac{1}{a_{\min}^5}$ dependency to one additive term of order $\frac{1}{a_{\min} \sqrt{n}}$, while other terms are not dependent on $a_{\min}$ and are overall of order $\sqrt{\frac{m - 1}{n}}$. In \eqref{eq:a_min}, the $4 a_{\min}^2$ can be tightened to $4a_i^2$ for each $\lf_i$, and $M_{23}$ and $\hat{M}_{23}$ are not in terms of $a_{\min}$ if neither of the two labeling functions at hand are $\lf_{\min}$. Therefore, for any $\lf_i \neq \lf_{\min}$, we do not have a dependency on $a_{\min}$ if we ensure that the triplet used to recover its accuracy in Algorithm 1 does not include $\lf_{\min}$. Then only one term in our final bound will have a $\frac{1}{a_{\min} \sqrt{n}}$ dependency compared to the previous $\frac{1}{a_{\min}^5} \sqrt{\frac{m}{n}}$.
\end{remark}

\paragraph{Concentration inequalities on observable data}
\begin{lemma}
Define $p^{(i)}(x) = P(\lf_i = x)$ and $\hat{p}^{(i)}(x) = \frac{1}{n} \sum_{k = 1}^n \ind{L_k^{(i)} = x}$, and let $p(x), \hat{p}(x) \in \mathbb{R}^m$ denote the vectors over all $i$. Then
\begin{align*}
\Delta_p := \E{}{\|\hat{p}(x) - p(x)\|_2} \le \sqrt{\frac{m}{n}}.
\end{align*}
\label{lemma:p}
\end{lemma}

\begin{proof} Note that $\E{}{\ind{L_k^{(i)} = x}} = P(\lf_i = 1)$. Then using Hoeffding's inequality, we have that 
\begin{align*}
P(|\hat{p}^{(i)}(x) - p^{(i)}(x)| \ge \epsilon) \le 2 \exp\left(-\frac{2n^2 \epsilon^2}{n (1)^2} \right) \le 2 \exp \left(- 2n\epsilon^2 \right).
\end{align*}

This expression is equivalent to
\begin{align*}
P(|p^{(i)}(x) - p^{(i)}(x)|^2 \ge \epsilon) \le 2 \exp\left( -2n \epsilon \right).
\end{align*}

We can now compute $\E{}{|\hat{p}^{(i)}(x) - p^{(i)}(x)|^2}$:
\begin{align*}
\E{}{|\hat{p}^{(i)}(x) - p^{(i)}(x)|^2} &\le \int_0^{\infty} 2\exp \left(- 2n \epsilon \right) d\epsilon = -2 \cdot \frac{1}{2n} \exp \left(-2n\epsilon \right) \bigg|_0^{\infty} = \frac{1}{n}.
\end{align*}

The overall L2 error for $p(x)$ is then
\begin{align*}
\E{}{\| \hat{p}(x) - p(x) \|_2} &= \E{}{\Big(\sum_{i = 1}^m | \hat{p}^{(i)}(x) - p^{(i)}(x)|^2\Big)^{1/2}} \le \sqrt{\sum_{i = 1}^m \E{}{|\hat{p}^{(i)}(x) - p^{(i)}(x) |^2}} \le \sqrt{\frac{m}{n}}.
\end{align*}

\end{proof}

\begin{lemma}
Define $M(a, b)$ to be a second moment matrix where $M(a, b)_{ij} = \E{}{a_i b_j}$ for some random variables $a_i, b_j \in \{-1, 0, 1\}$ each corresponding to $\lf_i, \lf_j$. Let $\|\cdot\|_{ij}$ be the Frobenius norm over elements indexed at $(i, j)$, where $\lf_i$ and $\lf_j$ share an edge in the dependency graph. If $G_{dep}$ has $d$ conditionally independent subgraphs, the estimation error of $M$ is
\begin{align*}
\Delta_M := \mathbb{E}[\|\hat{M}(a, b) - M(a, b)\|_{ij}] &\le C_m \sqrt{\frac{d - 1 + (m - d + 1)^2}{n}} \le C_m \frac{m}{\sqrt{n}}.
\end{align*}
For some constant $C_m$.
\label{lemma:M}
\end{lemma} 

\begin{proof}
Recall that the subgraphs are defined as sets $V_1, \dots, V_d$, and let $E_1, \dots, E_d$ be the corresponding sets of edges within the subgraphs. We can split up the norm $\|\hat{M}(a, b) - M(a, b)\|_{ij}$ into summations over sets of edges.
\begin{align*}
\|\hat{M}(a, b) - M(a, b)\|_{ij} &= \Big(\sum_{(i, j) \in E_{dep}} (\hat{M}(a, b)_{ij} - M(a, b)_{ij})^2 \Big)^{\frac{1}{2}} = \Big(\sum_{k = 1}^d\sum_{(i, j) \in E_k} (\hat{M}(a, b)_{ij} - M(a, b)_{ij})^2 \Big)^{\frac{1}{2}} \\
&\le \Big(\sum_{k = 1}^d\sum_{i, j \in V_k} (\hat{M}(a, b)_{ij} - M(a, b)_{ij})^2 \Big)^{\frac{1}{2}} = \Big(\sum_{k = 1}^d \frac{1}{2} \|\hat{M}(a, b)_{V_k} - M(a, b)_{V_k} \|_F^2 \Big)^{\frac{1}{2}}.
\end{align*}

We take the expectation of both sides by using linearity of expectation and Jensen's inequality:
\begin{align*}
\mathbb{E}[\|\hat{M}(a, b) - M(a, b)\|_{ij}] \le \Big(\sum_{k = 1}^d \frac{1}{2} \mathbb{E}[\|\hat{M}(a, b)_{V_k} - M(a, b)_{V_k} \|_F^2] \Big)^{\frac{1}{2}}.
\end{align*}

We are able to modify Proposition A.3 of \citet{Bunea15} into a concentration inequality for the second moment matrix rather than the covariance matrix, which states that $\mathbb{E}[\|\hat{M}(a, b)_{V_k} - M(a, b)_{V_k} \|_F^2] \le (32 e^{-4} + e + 64) \left(\frac{4c_1 tr(M_{V_k})}{\sqrt{n}} \right)^2$ for some constant $c_1$. We are able to use this result because our random variables are sub-Gaussian and have bounded higher order moments. Then our bound becomes
\begin{align*}
\mathbb{E}[\|\hat{M}(a, b) - M(a, b)\|_{ij}] &\le  \Big(\sum_{k = 1}^d \frac{1}{2} (32 e^{-4} + e + 64)\frac{16c_1^2 |V_k|^2}{n}  \Big)^{\frac{1}{2}} \le  \Big(\frac{8c_1^2 (32 e^{-4} + e + 64)}{n}\sum_{k = 1}^d |V_k|^2  \Big)^{\frac{1}{2}}.
\end{align*}

$\sum_{k = 1}^d |V_k|^2$ is maximized when we have $d - 1$ sugraphs of size $1$ and $1$ subgraph of size $m - d + 1$, in which case the summation is $d - 1 + (m - d + 1)^2$. Intuitively, when there are more subgraphs, this value will be smaller and closer to an order of $m$ rather than $m^2$. Putting this together, our bound is 
\begin{align*}
\mathbb{E}[\|\hat{M}(a, b) - M(a, b)\|_{ij}] \le \Big(8 c_1^2(32 e^{-4} + e + 64) \frac{d - 1 + (m - d + 1)^2}{n}\Big)^{\frac{1}{2}} \le C_m \frac{m}{\sqrt{n}}.
\end{align*}
Where $C_m = \sqrt{8c_1^2 (32 e^{-4} + e + 64)}$.
\end{proof}

\paragraph{Estimating $\mathbf{\mu_i}$}

We first estimate $\mu_i = P(\lf_i, Y^{dep}(i))$ for all relevant $\lf_i$. For ease of notation, let $Y$ refer to $Y^{dep}(i)$ in this section. Denote $\bm{\mu}_i$ to be the vector of all $\mu_i$ across all $\bm{\lf}$. Note that
\begin{align*}
\|\bm{\hat{\mu}}_i - \bm{\mu}_i \|_2 \le \|diag_m(A_1^{-1})\|_2 \|\hat{\rho} - \rho \|_2.
\end{align*}

$\rho$ is the vector of all $r_i$ for $i = 1, \dots, m$, and $diag_m(A_1^{-1})$ is a block matrix containing $m$ $A_1^{-1}$ on its diagonal; note that the $2$-norm of a block diagonal matrix is just the maximum $2$-norm over all of the block matrices, which is $\|A_1^{-1}\|_2$. Recall that $r_i = [1 \hspace{0.5em} P(\lf_i = 1) \hspace{0.5em} P(\lf_i = 0) \hspace{0.5em} P(Y = 1) \hspace{0.5em} P(\lf_i Y = 1) \hspace{0.5em} P(\lf_i = 0, Y = 1)]^T$. For each term of $r_i$, we have a corresponding sampling error to compute over $\rho$:
\begin{itemize}
\item $P(\lf_i = 1)$: We need to compute $\hat{P}(\lf_i = 1) - P(\lf_i = 1)$ for each $\lf_i$. All together, the sampling error for this term is equivalent to $\|\hat{p}(1) - p(1)\|_2$. 
\item $P(\lf_i = 0)$: The sampling error over all $\hat{P}(\lf_i = 0) - P(\lf_i = 0)$ is equivalent to $\|\hat{p}(0) - p(0)\|_2$.
\item $P(\lf_i Y = 1)$: Since $a_i = \E{}{v_{2i - 1} Y} = \E{}{\lf_i Y} = P(\lf_i Y = 1) - P(\lf_i Y = -1) = 2P(\lf_i Y = 1) + P(\lf_i = 0) - 1$ and the sampling error over all $\hat{P}(\lf_i Y = 1) - P(\lf_i Y = 1)$ is at most $\frac{1}{2}\|(\hat{a} - a) - (\hat{p}(0) - p(0)) \|_2  \le \frac{1}{2} \left(\|\hat{a} - a\|_2 + \|\hat{p}(0) - p(0)\|_2\right)$.
\item $P(\lf_i = 0, Y = 1)$: This expression is equal to $P(\lf_i = 0)P(Y = 1)$, so the sampling error is $P(Y = 1) \|\hat{p}(0) - p(0) \|_2 \le \|\hat{p}(0) - p(0) \|_2$.
\end{itemize}
Putting these error terms together, we have an expression for the sampling error for $\rho$:
\begin{align*}
\| \hat{\rho} - \rho\|_2 &= \sqrt{\| \hat{p}(1) - p(1)\|_2^2 + 2 \|\hat{p}(0) - p(0) \|_2^2 + \frac{1}{4} (\|\hat{a} - a \|_2 + \|\hat{p}(0) - p(0) \|_2)^2}  \\
&\le \| \hat{p}(1) - p(1)\|_2 + \sqrt{2} \|\hat{p}(0) - p(0)\|_2 + \frac{1}{2}(\|\hat{a} - a\| + \|\hat{p}(0) - p(0) \|)\\
&= \| \hat{p}(1) - p(1)\|_2 + \Big(\frac{1}{2} + \sqrt{2}\Big)\|\hat{p}(0) - p(0)\|_2 + \frac{1}{2}\|\hat{a} - a\|_2, 
\end{align*}

where we use concavity of the square root in the first step. Therefore,
\begin{align*}
\E{}{\|\hat{\rho} - \rho\|_2} &\le \E{}{\| \hat{p}(1) - p(1)\|_2} + \Big(\frac{1}{2} + \sqrt{2}\Big)\E{}{\|\hat{p}(0) - p(0) \|_2} + \frac{1}{2}\E{}{\|\hat{a} - a\|_2} \\
&= \Big(\frac{3}{2} + \sqrt{2}\Big) \Delta_p + \frac{1}{2} \Delta_a.
\end{align*}

Plugging this back into our error for $\bm{\mu}_i$ and using Lemmas \ref{lemma:a} and \ref{lemma:p},
\begin{align*}
\E{}{\|\bm{\hat{\mu}}_i - \bm{\mu}_i \|_2} \le \|A_1^{-1}\|_2 \left(\left(\frac{3}{2} + \sqrt{2} \right)\sqrt{\frac{m}{n}} + \frac{C_a}{2a^5_{|min|}} \sqrt{\frac{m}{n}} \right).
\end{align*}

Therefore, if there are no cliques of size $3$ or greater in $G_{dep}$, the sampling error is $\mathcal{O}(\sqrt{m/n})$.

\paragraph{Estimating all $\mu_{ij}$}

Now we estimate $\mu_{ij} = P(\lf_i, \lf_j, Y^{dep}(i, j))$ for $\lf_i, \lf_j$ sharing an edge in $G_{dep}$.  For ease of notation, let $Y$ refer to $Y^{dep}(i, j)$ in this section. Denote $\bm{\mu}_{ij}$ to be the vector of all $\mu_{ij}$. Note that
\begin{align*}
\|\bm{\hat{\mu}}_{ij} - \bm{\mu}_{ij}\|_2 \le \|diag_{|E|}(A_2)^{-1} \|_2 \|\hat{\psi} - \psi\|_2 = \|A_2^{-1}\|_2 \|\hat{\psi} - \psi\|_2.
\end{align*}

%Define $M(x, y)$ to be a second moment matrix where $M(x, y)_{ij} = \E{}{X_i^{(x)} X_j^{(y)}}$, where $X_i^{(x)} = \ind{\lf_i = x}$, e.g. $M(x, y)_{ij} = P(\lf_i = x, \lf_j = y)$. Then $\hat{M}(x, y)$ is the estimate of the second moment matrix where each $\hat{M}(x, y)_{ij} = \frac{1}{n} \sum_{i = 1}^n \ind{L_k^{(i)} = x, L_k^{(j)} = y}$.

$\psi$ is the vector of all $r_{ij}$ for all $(i, j) \in E$. Recall that $a_i = \E{}{v_i Y}$, $a_{ij} = \E{}{v_i v_j Y}$. We also define $X_i^{(a)} = \ind{\lf_i = a }$ and $M(X^{(a)}, X^{(b)})_{ij} = \E{}{X_i^{(a)} X_j^{(b)}} = P(\lf_i = a, \lf_j = b)$. For each term of $r_i$, we have a corresponding estimation error to compute.
\begin{itemize}
\item $P(\lf_i = 1)$: We need to compute $\hat{P}(\lf_i = 1) - P(\lf_i = 1)$ over all $(i, j) \in E$, so the sampling error for this term is $\sqrt{\sum_{(i, j) \in E} (\hat{P}(\lf_i = 1) - P(\lf_i = 1))^2} \le \sqrt{\sum_{i = 1}^m m (\hat{P}(\lf_i = 1) - P(\lf_i = 1))^2} = \sqrt{m} \|\hat{p}(1) - p(1) \|_2$.
\item $P(\lf_i = 0)$: The sampling error is equivalent to $\sqrt{m}\|\hat{p}(0) - p(0)\|_2$.
\item $P(\lf_j = 1)$: The sampling error is equivalent to $\sqrt{m}\|\hat{p}(1) - p(1)\|_2$.
\item $P(\lf_j = 0)$: The sampling error is equivalent to $\sqrt{m}\|\hat{p}(0) - p(0)\|_2$.
\item $P(\lf_i \lf_j = 1)$: This probability can be written as $P(\lf_i = 1, \lf_j = 1) + P(\lf_i = -1, \lf_j = -1)$, so we would need to compute $\hat{P}(\lf_i = 1, \lf_j = 1) - P(\lf_i = 1, \lf_j = 1) + \hat{P}(\lf_i = -1, \lf_j = -1) - P(\lf_i = -1, \lf_j = -1)$. Then the sampling error is equivalent to $\| \hat{M}(X^{(1)}, X^{(1)}) - M(X^{(1)}, X^{(1)}) + \hat{M}(X^{(-1)}, X^{(-1)}) - M(X^{(-1)}, X^{(-1)})\|_{ij}$.
\item $P(\lf_i = 0, \lf_j = 1)$: Using the definition of $M$, the sampling error over all $(i, j) \in E$ for this is $\|\hat{M}(X^{(0)}, X^{(1)}) - M(X^{(0)}, X^{(1)})\|_{ij}$.
\item $P(\lf_i = 1, \lf_j = 0)$: Similarly, the sampling error is $\|\hat{M}(X^{(1)}, X^{(0)}) - M(X^{(1)}, X^{(0)})\|_{ij}$.
\item $P(\lf_i = 0, \lf_j = 0)$: Similarly, the sampling error is $\|\hat{M}(X^{(0)}, X^{(0)}) - M(X^{(0)}, X^{(0)})\|_{ij}$.
\item $P(\lf_i Y = 1)$: Similar to before, the sampling error is $\frac{1}{2}\sqrt{m}\left( \|\hat{a} - a\|_2 + \|\hat{p}(0) - p(0)\|_2 \right)$.
\item $P(\lf_i = 0, Y = 1)$: Similar to our estimate of $\bm{\mu_i}$, the sampling error is $\sqrt{m}\|\hat{p}(0) - p(0)\|_2$.
\item $P(\lf_j Y = 1)$: The sampling error is $\frac{1}{2} \sqrt{m} \left( \|\hat{a} - a\|_2 + \|\hat{p}(0) - p(0)\|_2 \right)$.
\item $P(\lf_j = 0, Y = 1)$: The sampling error is $\sqrt{m}\|\hat{p}(0) - p(0)\|_2$.
\item $P(\lf_i \lf_j Y = 1)$:  Note that $\E{}{\lf_i \lf_j Y} = 2P(\lf_i \lf_j Y = 1) + P(\lf_i \lf_j = 0) - 1$. Moreover, $\E{}{\lf_i \lf_j Y}$ can be expressed as $\E{}{Y} \cdot \E{}{\lf_i \lf_j}$. Then the sampling error over all $\hat{P}(\lf_i \lf_j Y = 1) - P(\lf_i \lf_j Y = 1)$ is at least $\frac{1}{2}\|\E{}{Y}(\Ehat{\lf_i \lf_j} - \E{}{\lf_i \lf_j}) - (\hat{P}(\lf_i \lf_j = 0) - P(\lf_i \lf_j = 0))\|_{ij}$. Furthermore, we can write $P(\lf_i \lf_j = 0)$ as $P(\lf_i = 0) + P(\lf_j = 0) - P(\lf_i = 0, \lf_j = 0)$, so our sampling error is now less than $\frac{1}{2}\|\hat{M}(\lf, \lf) - M(\lf, \lf)\|_{ij} + \frac{1}{2} \sqrt{m}\|\hat{p}(0) - p(0) \|_2 + \frac{1}{2} \sqrt{m}\|\hat{p}(0) - p(0) \|_2 + \frac{1}{2} \|\hat{M}(X^{(0)}, X^{(0)}) - M(X^{(0)}, X^{(0)}) \|_{ij}$.
\item $P(\lf_i = 0, \lf_j Y = 1)$: Note that this can be written as $\frac{1}{2} \left(P(\lf_i = 0) + \E{}{\lf_j Y | \lf_i =0} P(\lf_i = 0) - P(\lf_i = 0, \lf_j = 0) \right)$. Then the sampling error over all $\hat{P}(\lf_i = 0, \lf_j Y = 1) - P(\lf_i = 0, \lf_j Y = 1)$ is equivalent to 
\begin{align*}
&\frac{1}{2} \sqrt{m}\|\hat{p}(0) - p(0)\|_2 + \frac{1}{2} \|\Ehat{\lf_j Y | \lf_i = 0} \hat{P}(\lf_i = 0) - \E{}{\lf_j Y | \lf_i = 0} P(\lf_i = 0) \\
&- (\hat{M}(X^{(0)}, X^{(0)}) - M(X^{(0)}, X^{(0)})) \|_{ij} \\
= \;& \frac{1}{2} \sqrt{m}\|\hat{p}(0) - p(0)\|_2 + \frac{1}{2}\|\hat{M}(X^{(0)}, X^{(0)}) - M(X^{(0)}, X^{(0)})\|_{ij} + \frac{1}{2} \|\Ehat{\lf_j Y | \lf_i = 0} (\hat{P}(\lf_i = 0) - P(\lf_i = 0)) \\
&- (\E{}{\lf_j Y | \lf_i = 0} - \Ehat{\lf_j Y | \lf_i = 0}) P(\lf_i = 0)\|_{ij} \\
\le \; & \frac{\sqrt{m}}{2} \|\hat{p}(0) - p(0)\|_2 + \frac{1}{2}\|\hat{M}(X^{(0)}, X^{(0)}) - M(X^{(0)}, X^{(0)})\|_{ij} + \frac{\sqrt{m}}{2}\|\hat{p}(0) - p(0)\|_2 \\
&+ \frac{1}{2}\|\E{}{\lf_j Y | \lf_i = 0} - \Ehat{\lf_j Y | \lf_i = 0}\|_{ij} \\
= \; & \sqrt{m} \|\hat{p}(0) - p(0)\|_2 + \frac{1}{2}\|\hat{M}(X^{(0)}, X^{(0)}) - M(X^{(0)}, X^{(0)})\|_{ij} + \frac{1}{2}\|\E{}{\lf_j Y | \lf_i = 0} - \Ehat{\lf_j Y | \lf_i = 0}\|_{ij}
\end{align*}
\item $P(\lf_j = 0, \lf_i Y = 1)$: Symmetric to the previous case, the sampling error is
$\sqrt{m} \|\hat{p}(0) - p(0)\|_2 + \frac{1}{2}\|\hat{M}(X^{(0)}, X^{(0)}) - M(X^{(0)}, X^{(0)})\|_{ij} + \frac{1}{2}\|\E{}{\lf_j Y | \lf_i = 0} - \Ehat{\lf_j Y | \lf_i = 0}\|_{ij}$.
\item $P(\lf_i = 0, \lf_j = 0, Y = 1)$: This expression is equal to $P(\lf_i = 0, \lf_j = 0) P(Y = 1)$, so the sampling error is $P(Y = 1) \|\hat{M}(X^{(0)}, X^{(0)}) - M(X^{(0)}, X^{(0)}) \|_{ij} \le \|\hat{M}(X^{(0)}, X^{(0)}) - M(X^{(0)}, X^{(0)}) \|_{ij}$.
\end{itemize}

After combining terms and taking the expectation, we have that
\begin{align*}
\E{}{\|\hat{\psi} - \psi\|_2} &\le 2 \sqrt{2m} \Delta_p + 2 \Delta_M + 3 \Delta_M + \frac{1}{\sqrt{2}}(\sqrt{m} \Delta_a +  \sqrt{m} \Delta_p) + \sqrt{2m} \Delta_p + \frac{1}{2}(\Delta_M + 2\sqrt{m} \Delta_p + \Delta_M) \\
&+ \frac{1}{\sqrt{2}}(2\sqrt{m}\Delta_p + \|\Ehat{\lf_i Y | \lf_j = 0} - \E{}{\lf_i Y | \lf_j = 0}\|_{ij} + \Delta_M ) + \Delta_M \\
&= \left(7 + \frac{1}{\sqrt{2}}\right) \Delta_M + \left(\frac{9}{2}\sqrt{2m} + \sqrt{m} \right) \Delta_p + \sqrt{\frac{m}{2}} \Delta_a + \frac{1}{\sqrt{2}}\|\Ehat{\lf_i Y | \lf_j = 0} - \E{}{\lf_i Y | \lf_j = 0}\|_{ij}.
\end{align*}

For $\E{}{\lf_i Y | \lf_j = 0}$, this term is equal to $0$ when no sources can abstain. Otherwise, suppose that among the sources that do abstain, each label abstains with frequency at least $r$. Then $\|\Ehat{\lf_i Y | \lf_j = 0} - \E{}{\lf_i Y | \lf_j = 0}\|_{ij} \le \sqrt{m} \cdot \frac{C_a}{a^5_{\min}} \sqrt{\frac{m}{rn}}$ since there are $rn$ samples used to produce the estimate. Using Lemma \ref{lemma:a}, \ref{lemma:p}, and \ref{lemma:M}, we now get that
\begin{align*}
\E{}{\|\bm{\hat{\mu}}_{ij} - \bm{\mu}_{ij} \|_2} &\le \|A_2^{-1}\| \Bigg(\bigg(7 + \frac{1}{\sqrt{2}} \bigg) C_m \frac{m}{\sqrt{n}} + \bigg(\frac{9\sqrt{2}}{2} + 1 \bigg) \frac{m}{\sqrt{n}} + \frac{C_a}{a^5_{|min|}} \cdot \frac{m}{\sqrt{n}} \bigg(\frac{1}{\sqrt{2}} + \frac{1}{\sqrt{2r}} \bigg) \Bigg).
\end{align*}

Finally, we can compute $\|A_1^{-1}\|$ and $\|A_2^{-1}\|$ since both matrices are constants, so the total estimation error is
\begin{align*}
\E{}{\|\bm{\hat{\mu}} - \bm{\mu}\|_2} \le &3.19 \left(\left(\frac{3}{2} + \sqrt{2} \right)\sqrt{\frac{m}{n}} + \frac{C_a}{2a^5_{|min|}} \sqrt{\frac{m}{n}} \right) + \\
&6.35 \Bigg(\bigg(7 + \frac{1}{\sqrt{2}} \bigg) C_m \frac{m}{\sqrt{n}} + \bigg(\frac{9\sqrt{2}}{2} + 1\bigg) \frac{m}{\sqrt{n}} + \frac{C_a}{a^5_{|min|}} \cdot \frac{m}{\sqrt{n}} \bigg(\frac{1}{\sqrt{2}} + \frac{1}{\sqrt{2r}} \bigg) \Bigg).
\end{align*}

\subsection{Proof of Theorem 2 (Information Theoretical Lower Bound)}%\ref{thm:lower_bound} (Information Theoretical Lower Bound)} %\ref{thm:lower_bound}

For Theorem 2 %\ref{thm:lower_bound} 
and Theorem 3, %\ref{thm:gen_offline}, 
we will need the following lemma.

\begin{lemma}

Let $\theta_1$ and $\theta_2$ be two sets of canonical parameters for an exponential family model, and let $\mu_1$ and $\mu_2$ be the respective mean parameters. If we define $e_{min}$ to be the smallest eigenvalue of the covariance matrix $\Sigma$ for the random variables in the graphical model,
$$\| \theta_1 - \theta_2 \| \le \frac{1}{e_{min}} \| \mu_1 - \mu_2 \|$$.
\label{lemma:fenchel}
\end{lemma}
\begin{proof}
Let $A(\theta)$ be the log partition function. Now, recall that the Hessian $\nabla^2 A(\theta)$ is equal to $\Sigma$ above. Next, since $e_{min}$ is the smallest eigenvalue, $\nabla^2 A(\theta) - e_{min} I = \Sigma - e_{min} I$ is positive semi-definite, so $A(\theta)$ is strongly convex with parameter $e_{min}$.

Note that since $A(\cdot)$ is strongly convex with parameter $e_{min}$, then $A^*(\cdot)$, its Fenchel dual, has Lipchitz continuous gradients with parameter $\frac{1}{e_{min}}$ \citep{Zhou2018OnTF}. This means that 
\begin{align*}
\| \nabla A^*(\mu_1) - \nabla A^*(\mu_2) \| \le \frac{1}{e_{min}} \|\mu_1 - \mu_2\|.
\end{align*}
But $\nabla A^*(\mu)$ is the inverse mapping from mean parameters to canonical parameters, so this is just
$$\|\theta_1 - \theta_2 \| \le \frac{1}{e_{min}} \|\mu_1 - \mu_2\|$$.
\end{proof}

Now, we provide the proof for Theorem $2$. Consider the following family of distributions for a graphical model with one hidden variable $Y$, $m$ observed variables that are all conditionally independent given $Y$, and no sources abstaining:
\begin{align*}
\mathcal{P} = \big\{P = \frac{1}{z}\exp(\theta_Y Y + \sum_{j = 1}^m \theta_j \lambda_j Y): \theta \in \mathbb{R}^{m + 1} \big\}
\end{align*}

We define a set of canonical parameters $\theta_v = \delta v$, where $\delta > 0$, $v \in \{-1, 1\}^m$ ($\theta_Y$ is fixed since it maps to a known mean parameter), and $P_v$ is the corresponding distribution in $\mathcal{P}$. $\mathcal{P}$ induces a $\frac{\delta}{\sqrt{m}}$-Hamming separation for the L2 loss because
\begin{align*}
\|\theta - \theta_v\|_2 &= \Big(\sum_{j = 1}^m |\theta_j - [\theta_v]_j|^2\Big)^{1/2} \ge \frac{\sum_{j = 1}^m 1 \cdot | \theta_j - [\theta_v]|_j}{\big(\sum_{j = 1}^m 1^2 \big)^{1/2}} \\
&= \frac{1}{\sqrt{m}} \sum_{j = 1}^m |\theta_j - [\theta_v]_j| \ge \frac{\delta}{\sqrt{m}} \sum_{j = 1}^m \mathbf{1}\{\mathrm{sign}(\theta_j) \neq v_j \}.
\end{align*}

We use Cauchy-Schwarz inequality in the first line and the fact that if the sign of $\theta_j$ is different from $v_j$, then $\theta_j$ and $[\theta_v]_j$ must be at least $\delta$ apart. Then applying Assouad's Lemma \citep{Yu1997}, the minimax risk is bounded by
\begin{align*}
\mathcal{M}_n(\theta(\mathcal{P}), L2) = \inf_{\hat{\theta}} \sup_{P \in \mathcal{P}} \mathbb{E}_P[\|\hat{\theta}(X_1, \dots, X_n) - \theta(P)\|_2] \ge \frac{\delta}{2\sqrt{m}} \sum_{j = 1}^m 1 - \|P^n_{+j} - P^n_{-j}\|_{TV}.
\end{align*}

$\hat{\theta}(X_1, \dots, X_n)$ is an estimate of $\theta$ based on the $n$ observable data points, while $\theta(P)$ is the canonical parameters of a distribution $P$. $P_{\pm j}^n = \frac{1}{2^{m - 1}} \sum_{v} P^n_{v, \pm j}$, where $P^n_{v, \pm j}$ is the product of $n$ distributions parametrized by $\theta_v$ with $v_j = \pm 1$. We use the convexity of total variation distance, Pinsker's inequality, and decoupling of KL-divergence to get 
\begin{align*}
\|P^n_{+j} - P^n_{-j}\|^2_{TV} \le \underset{d_{ham}(v, v') \le 1}{\text{max}} \|P^n_v - P^n_{v'}\|^2_{TV} \le \frac{1}{2} \; \underset{d_{ham}(v, v') \le 1}{\text{max}}
 KL(P_v^n \| P_{v'}^n) = \frac{n}{2} \; \underset{d_{ham}(v, v') \le 1}{\text{max}}
 KL(P_v \| P_{v'}).
\end{align*}

$v$ and $v'$ above only differ in one term. Then our lower bound becomes
\begin{align}
\mathcal{M}_n(\theta(\mathcal{P}), L2) &\ge \frac{\delta}{2\sqrt{m}} \sum_{j = 1}^m 1 - \sqrt{\frac{n}{2} \; \underset{d_{ham}(v, v') \le 1}{\text{max}} KL(P_v \| P_{v'})} = \frac{\delta \sqrt{m}}{2} \left(1 - \sqrt{\frac{n}{2} \; \underset{d_{ham}(v, v') \le 1}{\text{max}} KL(P_v \| P_{v'})}\right).
\label{eq:KL_lowerbound}
\end{align}

We must bound the KL-divergence between $P_v$ and $P_v'$. Suppose WLOG that $v$ and $v'$ differ at the $i$th index with $v_i = 1, v'_i = -1$, and let $z_v$ and $z_{v'}$ be the respective terms used to normalize the distributions. Then the KL divergence is
\begin{align}
KL(P_v \| P_{v'}) = \mathbb{E}_v[\langle \theta_v - \theta_{v'}, \lambda Y \rangle] + \ln \frac{z_{v'}}{z_v} = 2\delta \mathbb{E}_v[\lambda_i Y] + \ln \frac{z_{v'}}{z_v}.
\label{eq:KL}
\end{align}

We can write an expression for $\mathbb{E}_v[\lambda_i Y]$:
\begin{align}
\mathbb{E}_v[\lambda_i Y] &= 2(P_v(\lambda_i = 1, Y = 1) + P_v(\lambda_i = -1, Y = -1)) - 1 \nonumber \\
&= \frac{2}{z_v} \Big( \sum_{\lambda_{\neg i}} \exp(\theta_Y + \delta + \sum_{j \neq i}^m (\delta v_j)\lambda_j) + \exp(-\theta_Y + \delta - \sum_{j \neq i}^m (\delta v_j)\lambda_j)\Big) - 1 \nonumber \\
&= \frac{2}{z_v} \exp(\delta) \sum_{\lambda_{\neg i}} 2 \cosh (\theta_Y + \sum_{j \neq i}^m (\delta v_j) \lambda_j) - 1. \label{eq:KL_expectation}
\end{align}

Similarly, $z_v$ and $z_{v'}$ can be written as 
\begin{align*}
z_v &= \exp(\delta) \sum_{\lambda_{\neg i}} 2 \cosh (\theta_Y + \sum_{j \neq i}^m (\delta v_j) \lambda_j) + \sum_{\lambda_{\neg i}} \exp(\theta_Y - \delta + \sum_{j \neq i} (\delta v_j)\lambda_j) + \sum_{\lambda_{\neg i}} \exp(-\theta_Y - \delta - \sum_{j \neq i}(\delta v_j)) \\
&= (\exp(\delta) + \exp(-\delta)) \sum_{\lambda_{\neg i}} 2 \cosh(\theta_Y + \sum_{j \neq i}(\delta v_j)\lambda_j) = 4 \cosh(\delta) \sum_{\lambda_{\neg i}}  \cosh(\theta_Y + \sum_{j \neq i}(\delta v_j)\lambda_j) \\
z_{v'} &= 4 \cosh(\delta) \sum_{\lambda_{\neg i}}  \cosh(\theta_Y + \sum_{j \neq i}(\delta v'_j)\lambda_j)
\end{align*}

Plugging $z_v$ back into \eqref{eq:KL_expectation}, we get:
\begin{align*}
\E{v}{\lf_i Y} &= 4 \cdot \frac{\exp(\delta) \sum_{\lf_{\neg i}} \cosh(\theta_Y + \sum_{j \neq i }^m (\delta v_j) \lf_j)}{4 \cosh(\delta) \sum_{\lf_{\neg i}} \cosh(\theta_Y + \sum_{j \neq i}^m (\delta v_j) \lf_j)} - 1 = \frac{\exp(\delta)}{\cosh(\delta)} - 1.
\end{align*}

Also note that $\frac{z_{v'}}{z_v} = 1$ since $v_j' = v_j$ for all $j \neq i$. The KL-divergence expression \eqref{eq:KL} now becomes
\begin{align*}
KL(P_v \| P_{v'}) = 2\delta\left(\frac{\exp(\delta)}{\cosh(\delta)} -1\right) + \ln(1) = 2\delta\left(\frac{\exp(\delta)}{\cosh(\delta)} -1\right).
\end{align*}

We finally show that this expression is less than $2\delta^2$. Note that for positive $\delta$, $f(\delta) = \frac{\exp(\delta)}{\cosh(\delta)} -1 < \delta$, because $f(\delta)$ is concave and $f'(0) = 1$. Then we clearly have that $KL(P_v \| P_{v'}) \le 2\delta^2$. Putting this back into our expression for the minimax risk, \eqref{eq:KL_lowerbound} becomes
\begin{align*}
\mathcal{M}_n(\theta(\mathcal{P}), L2) \ge \frac{\delta \sqrt{m}}{2} (1 - \sqrt{n \delta^2}).
\end{align*}

Then if we set $\delta = \frac{1}{2\sqrt{n}}$, we get that 
\begin{align*}
\mathcal{M}_n(\theta(\mathcal{P}), L2) \ge \frac{\sqrt{m}}{8\sqrt{n}}.
\end{align*}

Lastly, to convert to a bound over the mean parameters, we use Lemma \ref{lemma:fenchel} to conclude that
\begin{align*}
\inf_{\hat{\mu}} \sup_{P \in \mathcal{P}} \E{P}{\|\hat{\mu}(X_1, \dots, X_n) - \mu(P) \|_2} \ge \frac{e_{min}}{8} \sqrt{\frac{m}{n}}.
\end{align*}

From this, we can conclude that the estimation error on the label model parameters $\|\bm{\hat{\mu}} - \bm{\mu} \|_2$ is also at least $\frac{e_{min}}{8} \sqrt{\frac{m}{n}}$.

\subsection{Proof of Theorem 3 (Generalization Error)}%\ref{thm:gen_offline} (Generalization Error)} %\ref{thm:gen_offline}

We base our proof off of Theorem $1$ of \citet{Ratner19} with modifications to account for model misspecification. To learn the parametrization of our end model $f_w$, we want to minimize a loss function $L(w, \bm{X}, \bm{Y}) \in [0, 1]$. The expected loss we would normally minimize using some $w^* = \argmax{w}{} L(w)$ is
\begin{align*}
L(w) = \E{(\bm{X}, \bm{Y}) \sim \mathcal{D}}{L(w, \bm{X}, \bm{Y})}.
\end{align*}

However, since we do not have access to the true labels $\bm{Y}$, we instead minimize the expected noise-aware loss. Recall that $\bm{\mu}$ is the parametrization of the label model we would learn with population-level statistics, and $\bm{\hat{\mu}}$ is the parametrization we learn with the empirical estimates from our data. Denote $P_{\bm{\mu}}$ and $P_{\bm{\hat{\mu}}}$ as the respective distributions. If we were to have a population-level estimate of $\bm{\mu}$, the loss to minimize would be 
\begin{align*}
L_{\bm{\mu}}(w) = \E{(\bm{X}, \bm{Y}) \sim \mathcal{D}}{\E{\bm{\widetilde{Y}} \sim P_{\bm{\mu}}(\cdot | \bm{\lf}(\bm{X}))}{L(w, \bm{X}, \bm{\widetilde{Y}})}}.
\end{align*}

However, because we must estimate $\bm{\hat{\mu}}$ and further are minimizing loss over $n$ samples, we want to estimate a $\hat{w}$ that minimizes the empirical loss,
\begin{align*}
\hat{L}_{\bm{\hat{\mu}}}(w) = \frac{1}{n} \sum_{i = 1}^n \E{\bm{\widetilde{Y}} \sim P_{\bm{\hat{\mu}}}(\cdot | \bm{\lf}(\bm{X_i}))}{L(w, \bm{X_i}, \bm{\widetilde{Y}})}.
\end{align*}

We first write $L(w)$ in terms of $L_{\bm{\mu}}(w)$.
\begin{align*}
L(w) = \;& \E{(\bm{X}, \bm{Y}) \sim \mathcal{D}}{L(w, \bm{X}, \bm{Y})} = \E{(\bm{X'}, \bm{Y'}) \sim D}{\E{(\bm{X}, \bm{Y}) \sim D}{L(w, \bm{X'}, \bm{Y}) | \bm{X} = \bm{X'}}} \\
=\;& \mathbb{E}_{(\bm{X'}, \bm{Y'}) \sim D}\big[\E{(\bm{X}, \bm{\widetilde{Y}}) \sim P_{\bm{\mu}}}{L(w, \bm{X'}, \bm{Y}) | \bm{X} = \bm{X'}} + \E{(\bm{X}, \bm{Y}) \sim \mathcal{D}}{L(w, \bm{X'}, \bm{Y}) | \bm{X} = \bm{X'}} \\
-\;& \E{(\bm{X}, \bm{\widetilde{Y}}) \sim P_{\bm{\mu}}}{L(w, \bm{X'}, \bm{Y}) | \bm{X} = \bm{X'}}\big] \\
\le\;& \E{(\bm{X'}, \bm{Y'}) \sim \mathcal{D}}{\E{(\bm{\lf}, \bm{\widetilde{Y}}) \sim P_{\bm{\mu}}}{L(w, \bm{X'}, \bm{Y}) | \bm{\lf} = \bm{\lf'})}} \\
+\;& \mathbb{E}_{(\bm{X'}, \bm{Y'}) \sim \mathcal{D}}\bigg[\Big|\sum_{x, y} L(w, \bm{X'}, y) (\mathcal{D}(\bm{X} = x, \bm{Y} = y | \bm{X} = \bm{X'}) - P_{\bm{\mu}}(\bm{X} = x, \bm{Y} = y | \bm{X} = \bm{X'})) \Big|\bigg]\\
\le\;& L_{\bm{\mu}}(w) + \mathbb{E}_{(\bm{X'}, \bm{Y'}) \sim \mathcal{D}} \Big[\sum_{x, y} L(w, \bm{X'}, y) \cdot \big| \mathcal{D}(\bm{X} = x, \bm{Y} = y | \bm{X} = \bm{X'}) - P_{\bm{\mu}}(\bm{X} = x, \bm{Y} = y | \bm{X} = \bm{X'})\big| \Big] \\
\le\;& L_{\bm{\mu}}(w) + \mathbb{E}_{(\bm{X'}, \bm{Y'}) \sim \mathcal{D}} \Big[\sum_{x, y} \big| \mathcal{D}(\bm{X} = x, \bm{Y} = y | \bm{X} = \bm{X'}) - P_{\mu}(\bm{X} = x, \bm{Y} = y | \bm{X} = \bm{X'})\big|\Big]
\end{align*}

Here we have used the fact that $L(w, \bm{X'}, y) \le 1$. Note that $\mathcal{D}(\bm{X} = x, \bm{Y} = y | \bm{X} = \bm{X'}) = \mathcal{D}(\bm{Y} = y | \bm{X} = \bm{X'})$ only when $\bm{X'} = x$, and is $0$ otherwise.
The same holds for $P_{\bm{\mu}}$, so
\begin{align*}
L(w) &\le L_{\bm{\mu}}(w) + \E{(\bm{X'}, \bm{Y'}) \sim \mathcal{D}}{\sum_y \big| \mathcal{D}(\bm{Y} = y | \bm{X} = \bm{X'}) - P_{\bm{\mu}}(\bm{Y} = y | \bm{X} = \bm{X'})\big| }.
\end{align*}

Note that the expression $\sum_y \big| \mathcal{D}(\bm{Y} = y | \bm{X} = \bm{X'}) - P_{\bm{\mu}}(\bm{Y} = y | \bm{X} = \bm{X'})\big| $ is just half the total variation distance between $\mathcal{D}(\bm{Y} | \bm{X'})$ and $P_{\bm{\mu}}(\bm{Y} | \bm{X'})$. Then, using Pinsker's inequality, we bound $L(w)$ in terms of the conditional KL divergence between $\mathcal{D}$ and $P_{\mu}$:
\begin{align*}
L(w) &\le L_{\bm{\mu}}(w) + \E{\bm{X'} \sim \mathcal{D}}{2 \cdot TV(\mathcal{D}(\bm{Y} |  \bm{X'}), P_{\bm{\mu}}(\bm{Y} | \bm{X'})) } \\
&\le L_{\bm{\mu}}(w) + 2 \cdot \E{\bm{X} \sim \mathcal{D}}{\sqrt{(1/2) KL (\mathcal{D}(\bm{Y} | \bm{X}) \;\| \;P_{\bm{\mu}}(\bm{Y} | \bm{X}))}} \\
&\le  L_{\bm{\mu}}(w) +  \sqrt{2 \cdot  KL (\mathcal{D}(\bm{Y} | \bm{X}) \;\| \;P_{\bm{\mu}}(\bm{Y} | \bm{X}))}.
\end{align*}

There is a similar lower bound on $L(w)$ if we perform the same steps as above on the inequality $L(w) \ge L_{\bm{\mu}}(w) - \E{(\bm{X'}, \bm{Y'}) \sim \mathcal{D}}{\Big| \E{(\bm{X}, \bm{Y}) \sim \mathcal{D}}{L(w, \bm{X'}, \bm{Y}) | \bm{X} = \bm{X'}} - \E{(\bm{X}, \bm{\widetilde{Y}}) \sim P_{\bm{\mu}}}{L(w, \bm{X'}, \bm{Y}) | \bm{X} = \bm{X'}} \Big|}$. This yields
\begin{align*}
L(w) \ge L_{\bm{\mu}}(w) - \sqrt{2 \cdot KL (\mathcal{D}(\bm{Y} | \bm{X}) \;\| \;P_{\bm{\mu}}(\bm{Y} | \bm{X}))}.
\end{align*}

Therefore,
\begin{align*}
L(\hat{w}) - L(w^*) \le L_{\bm{\mu}}(\hat{w}) - L_{\bm{\mu}}(w^*) + 2\sqrt{2 \cdot  KL (\mathcal{D}(\bm{Y} | \bm{X}) \;\| \;P_{\bm{\mu}}(\bm{Y} | \bm{X}))}.
\end{align*}

We finish the proof of the generalization bound with the procedure from \citet{Ratner19} but also use the conversion from canonical parameters to mean parameters as stated in Lemma \ref{lemma:fenchel}, and note that the estimation error of the mean parameters is always less than the estimation error of the label model parameters. Then our final generalization result is 
\begin{align*}
L(\hat{w}) - L(w^*) \le \gamma(n) + \frac{8|\mathcal{Y}|}{e_{min}} \|\bm{\hat{\mu}} - \bm{\mu} \|_2 + \delta(\mathcal{D}, P_{\bm{\mu}}),
\end{align*}

where $\delta(\mathcal{D}, P_{\bm{\mu}}) = 2\sqrt{2 \cdot  KL (\mathcal{D}(\bm{Y} | \bm{X}) \;\| \;P_{\bm{\mu}}(\bm{Y} | \bm{X}))}$, $e_{min}$ is the minimum eigenvalue of $\Cov{}{\bm{\lf}, \bm{Y}}$ over the construction of the binary Ising model, and $\gamma(n)$ bounds the empirical risk minimization error.

\section{Extended Experimental Details}
\label{sec:extexp}

We describe additional details about the tasks, including details about data
sources, supervision sources, and end models.
We also report details about our ablation studies.
All timing measurements were taken on a machine with an Intel Xeon E5-2690 v4 CPU and
Tesla P100-PCIE-16GB GPU.
Details about the sizes of the train/dev/test splits and end models are shown
in Table~\ref{table:stats}.

\subsection{Dataset Details}

%!TEX root = ../main.tex

\begin{table}[ht!]
    \centering
    \begin{tabular}{@{}llccc@{}}
        \toprule
        \textbf{Dataset} & \textbf{End Model}  & \textbf{$N_{train}$} & \textbf{$N_{dev}$}   & \textbf{$N_{test}$}   \\ \midrule
        \spouse          & LSTM                & 22,254               & 2,811                & 2,701                 \\
        \spam            & Logistic Regression & 1,586                & 120                  & 250                   \\
        \weather         & Logistic Regression & 187                  & 50                   & 50                    \\
        \commercial      & ResNet-50           & 64,130               & 9,479                & 7,496                 \\
        \interview       & ResNet-50           & 6,835                & 3,026                & 3,563                 \\
        \tennis          & ResNet-50           & 6,959                & 746                  & 1,098                 \\
        \basketball      & ResNet-18           & 3,594                & 212                  & 244                   \\ 
        \bottomrule
        \end{tabular}
    \caption{
    We report the train/dev/test split of each dataset.
    The dev and test set have ground truth labels, and we assign labels to the
    training set using our method or one of the baseline methods.
    }
    \label{table:stats}
\end{table}

\paragraph{\spouse, \weather}
We use the datasets from~\citet{Ratner18} and the train/dev/test splits from
that work (\weather\ is called \textbf{Crowd} in that work).

\paragraph{\spam}
We use the dataset as provided by
Snorkel\footnote{https://www.snorkel.org/use-cases/01-spam-tutorial} and those
train/dev/test splits.

\paragraph{\interview, \basketball}
We use the datasets from~\citet{sala2019multiresws} and the train/dev/test
splits from that work.

\paragraph{\commercial}
We use the dataset from~\citet{fu2019rekall} and the train/dev/test splits from
that work.

\paragraph{\tennis}
We obtained broadcast footage from four professional tennis matches, and
annotated segments when the two players are in a rally.
We temporally downsampled the images at 1 FPS. We split into dev/test by taking
segments from each match (using contiguous segments for dev and test,
respectively) to ensure that dev and test come from the same distribution.

\subsection{Task-Specific End Models}

For the datasets we draw from previous work (each dataset except for \tennis),
we use the previously published end model architectures
(LSTM~\cite{Hochreiter1997-fj} for \spouse, logistic regression over bag of
n-grams for \spam\ and over Bert features for \weather~\cite{devlin2018bert},
ResNet pre-trained on ImageNet for the video tasks).
For \tennis, we use ResNet-50 pre-trained on ImageNet to classify individual
frames.
We do not claim that these end models achieve the best possible performance for
each task; our goal is the compare the relative imporovements that our weak
supervision models provide compare to other baselines through label quality,
which is orthogonal to achieving state-of-the-art performance for these
specific tasks.

For end models that come from previous works, we use the hyperparameters from
those works.
For the label model baselines, we use the hyperparameters from previous works
as well.
For our label model, we use class balance from the dev set, or tune the class
balance ourselves with a grid search.
We also tune which triplets we use for parameter recovery on the dev set.
For our end model parameters, we either use the hyperparameters from previous
works, or run a simple grid search over learning rate and momentum.

\subsection{Supervision Sources}

Supervision sources are expressed as short Python functions.
Each source relied on different information to assign noisy labels:

\paragraph{\spouse, \weather, \spam}
For these tasks, we used the same supervision sources as used in previous
work~\cite{Ratner18}.
These are all text classification tasks, so they rely on text-based heuristics
such as the presence or absence of certain words, or particular regex patterns.

\paragraph{\interview, \basketball}
Again, we use sources from previous work~\cite{sala2019multiresws}.
For \interview, these sources rely on the presence of certain faces in the
frame, as determined by an identity classifier, or certain text in the
transcript.
For \basketball, these sources rely on an off-the-shelf object detector to
detect balls or people, and use heuristics based on the average pixel of the
detected ball or distance between the ball and person to determine whether the
sport being played is basketball or not.

\paragraph{\commercial}
In this dataset, there is a strong signal for the presence or absence of
commercials in pixel histograms and the text; in particular, commercials are
book-ended on either side by sequences of black frames, and commercial segments
tend to have mixed-case or missing transcripts (whereas news segments are in
all caps).
We use these signals to build the weak supervision sources.

\paragraph{\tennis}
This dataset uses an off-the-shelf pose detector to provide primitives for the
weak supervision sources.
The supervision sources are heuristics based on the number of people on court
and their positions.
Additional supervision sources use color histograms of the frames (i.e., 
how green the frame is, or whether there are enough white pixels for the court
markings to be shown).

\subsection{Ablation Studies}
\label{sec:supp_ablation}

%!TEX root = ../main.tex

\begin{table*}[t]
    \centering
    \begin{tabular}{@{}rlccccccccccccc@{}}
        \toprule
                                        & \spouse & \spam & \weather \\
        \midrule
        \textbf{Random abstains}        & 20.9    & 64.1  & 69.1     \\
        \textbf{\sysx}                  & 49.6    & 92.3  & 88.9     \\
        \midrule
        \textbf{Single Triplet Worst}   & 4.5     & 67.0  & 0.0      \\
        \textbf{Single Triplet Best}    & 51.2    & 83.6  & 77.6     \\
        \textbf{Single Triplet Average} & 37.9    & 73.4  & 31.0     \\
        \cmidrule(l){2-4}
        \textbf{\sysx Label Model}      & 47.0    & 89.1  & 77.6     \\
        \bottomrule \\
        \end{tabular}
        \caption{
        End model performance in terms of F1 score with random votes
        replacing abstentions (first row), compared to \sysx, for the benchmark
        applications.
        }
    \label{table:random_abstains}
\end{table*}

We report the results of two ablation studies on the benchmark applications.
In the first study, we examine the effect of randomly replacing abstains with
votes, instead of augmenting $G_{dep}$.
In the second study, we examine the effect of using a single random selection
of triplets instead of taking the mean or median over all triplet assignments.

Table~\ref{table:random_abstains} (top) shows end model performance for the three
benchmark tasks when replacing abstains with random votes (top row), compared
to \sysx\ end model performance.
Replacing abstentions with random votes results in a major degradation in
performance.

Table~\ref{table:random_abstains} (bottom) shows label model performance when
using a single random assignment of triplets, compared to the \sysx\ label
model, which takes the median or mean of all possible triplets.
There is large variance when taking a single random assignment of triplets,
whereas using an aggregation is more stable.
In particular, while selecting a good seed can result in performance that
matches (\weather) or exceeds (\spouse) \sysx\ label model performance,
selecting a \textit{bad} seed result in much worse performance (including
catastrophically bad predictors).
As a result, \sysx\ outperforms random assignments on average.

As a final note, we comment on using means vs. medians for aggregating accuracy
scores.
For all tasks except for \weather, there is no difference in label model
performance.
For \weather, using medians is more accurate, since the supervision sources
have a large abstention rate.
As a result, many triplets result in accuracy scores of zero (hence the $0$ F1
score in Table~\ref{table:random_abstains}).
This throws off the median aggregation, since the median accuracy score becomes
zero for many sources.
However, mean aggregation is more robust to these zero's, since the positive
accuracy scores from the triplets can correct for the accuracy.

\fi

\end{document}